\documentclass[11pt,a4paper]{article}
\usepackage[margin=0.75in]{geometry}
\usepackage[utf8]{inputenc}
\usepackage{cmap}
\usepackage[T1]{fontenc}
\usepackage{lmodern}
\usepackage{booktabs, multirow, pifont}
\usepackage[square, numbers]{natbib}     
\usepackage[colorlinks=true]{hyperref}
\usepackage{amssymb, amsthm, amsmath, amsfonts}
\usepackage{authblk, dsfont, mathrsfs, color, bm, enumitem, mathtools}
\usepackage{subfigure, graphicx, transparent}
\usepackage{dirtytalk}
\graphicspath{{figures/}}
\usepackage[dvipsnames]{xcolor}
\usepackage{comment}
\hypersetup{colorlinks, linkcolor={red!80!black}, citecolor={blue!80!black}, urlcolor={ForestGreen}}
\usepackage{mdframed}

\theoremstyle{definition}
\newtheorem{theorem}{Theorem}[section]
\newtheorem{definition}[theorem]{Definition}
\newtheorem{assumption}[theorem]{Assumption}
\newtheorem{proposition}[theorem]{Proposition}
\newtheorem{corollary}[theorem]{Corollary}
\newtheorem{lemma}[theorem]{Lemma}
\newtheorem{remark}[theorem]{Remark}

\DeclareUnicodeCharacter{211D}{\ensuremath{\mathbb R}}

\DeclarePairedDelimiterX{\inner}[2]{\langle}{\rangle}{#1, #2}

\newcommand{\R}{\mathbb{R}}
\newcommand{\cX}{\mathcal{X}}
\newcommand{\cF}{\mathcal{F}}
\newcommand{\E}{\mathbb{E}}
\newcommand{\Prob}{\mathbb{P}}
\newcommand{\norm}[1]{\left\| #1 \right\|}
\newcommand{\diam}{\operatorname{diam}}
\newcommand{\proj}{\Pi}

\title{Beyond Optimal Rates in Stochastic Optimization: Trajectory-Adaptive Stopping Rules}

\author[1]{Liviu Aolaritei}
\author[3]{Lucas L{\'e}vy}
\author[3]{Francis Bach}
\author[1,2,3]{Michael I. Jordan}

\affil[1]{Department of Electrical Engineering and Computer Sciences, UC Berkeley, USA}
\affil[2]{Department of Statistics, UC Berkeley, USA}
\affil[3]{Inria, {\'E}cole Normale Sup{\'e}rieure, PSL Research University,
France}

\date{}

\begin{document}
\maketitle

\begin{abstract}
Stochastic gradient descent (SGD) is typically analyzed at a deterministic horizon chosen before the algorithm is run, even though practical stopping decisions are made adaptively by inspecting the evolving trajectory. This mismatch creates a fundamental certification problem: fixed-time guarantees do not generally remain valid at data-dependent stopping times, while deterministic horizons derived from worst-case bounds can be highly conservative. We address this problem for strongly convex stochastic optimization by constructing fully observable, trajectory-adaptive upper confidence sequences for the squared distance of the last iterate to the optimizer and the suboptimality of a weighted average. These bounds hold simultaneously over time, attain the optimal $1/t$ decay rate up to iterated-logarithmic factors in the worst case, and adapt to the realized stochastic gradients, allowing SGD to stop as soon as a prescribed accuracy is certified without sacrificing statistical validity. Our approach treats the evolving SGD trajectory as a sequential experiment whose observations provide evidence about the unknown optimization error. To formalize this perspective, we develop new recursive confidence-sequence techniques and a general time-uniform empirical Bernstein inequality for adapted processes with time-varying conditional means and predictable ranges that may grow without bound. We further extend these confidence-sequence constructions to minibatch SGD, with the empirical Bernstein bounds exploiting the realized second-moment structure within each minibatch. Numerical experiments show that the resulting stopping rules can require several orders of magnitude fewer iterations than natural deterministic horizons.
\end{abstract}


\section{Introduction}
\label{sec:intro}

Stochastic optimization is a foundational computational paradigm across machine learning, operations research, statistics, and scientific computing. When objectives are expectations under unknown distributions or finite sums over prohibitively large datasets, stochastic gradient methods make optimization possible using only samples or noisy first-order information \cite{robbins1951stochastic,nemirovski2009robust}. Their scalability has made them the principal computational workhorses of modern machine learning \cite{bottou2018optimization}. Yet a basic operational question remains unresolved: \emph{when should a stochastic optimization algorithm stop?}

Existing theoretical analyses of stochastic gradient descent (SGD) are predominantly \emph{ex ante}. Before the algorithm is run, one specifies a stepsize schedule and a deterministic horizon $T$, and then studies its performance at that horizon. This perspective has produced a deep understanding of stochastic first-order methods: it has identified optimal rates, clarified the roles of averaging and strong convexity, and provided principled guidance for selecting stepsizes \cite{polyak1992acceleration,moulines2011non,rakhlin2012making,lacoste2012simpler}. Its guarantees, however, are attached to horizons chosen independently of the realized run. In other words, they answer the forward-looking question of what can be guaranteed if SGD is run until a predetermined time $T$, but not the online question of whether the trajectory observed so far provides enough evidence that the algorithm should stop.

In practice, however, this online question is unavoidable. Practitioners alter stepsizes when progress slows and stop a run when the observed trajectory appears to have stabilized. Yet the quantities that would directly certify optimization progress are usually unknown: the objective is often a population loss defined by an expectation under an unknown distribution, and its optimizer and optimal value are unknown. Consequently, neither the true suboptimality nor the distance to the optimizer is directly observable during the run. Practitioners therefore rely on observable proxies, such as empirical losses computed from validation samples \cite{prechelt1998early}, and use these proxies together with the evolving stochastic trajectory to decide whether to tune, continue, or stop the algorithm. These proxies are themselves random and dependent on the particular samples used to construct them. They may be informative, but favorable behavior of a proxy is not itself a certificate for suboptimality or distance to the optimizer.

Existing fixed-horizon guarantees do not certify this adaptive mode of operation: a high-probability statement proved at a deterministic time $T$ does not, in general, remain valid at a data-dependent stopping time selected after repeatedly inspecting the iterates, stochastic gradients, or validation losses \cite{feller1940statistical,anscombe1954fixed,robbins1952some}. The stochastic-oracle assumptions may remain unchanged, but the stopping time is now coupled to the random fluctuations revealed along the observed run. Among the many iterations examined, an unusually favorable fluctuation may be precisely what triggers termination. Therefore, a guarantee calibrated for one predetermined time does not automatically account for this selection over time.

One can preserve validity by fixing the stopping horizon in advance. Given a fixed-time high-probability bound and a target accuracy $\varepsilon$, one may solve for a time $T_\varepsilon$ at which the bound falls below $\varepsilon$, and then run SGD for exactly that many iterations. This produces a valid certified stopping rule, but typically a very conservative one. Because $T_\varepsilon$ is obtained by inverting a worst-case guarantee before the trajectory is observed, it reflects the most adverse behavior permitted by the assumptions rather than the behavior of the realized run. In practice, SGD can perform substantially better than these guarantees predict, sometimes by several orders of magnitude, so the resulting deterministic horizon may require many more iterations than would be sufficient on the realized run. This limitation is not resolved by choosing a stepsize schedule that attains the optimal worst-case rate: such optimality describes performance over a broad problem class, whereas stopping asks whether the particular run has already reached a prescribed accuracy. Existing theory is therefore indispensable for ex ante algorithm design, but it is not itself an online stopping mechanism.

This leaves a clear gap between rigorous but potentially conservative fixed-horizon guarantees and adaptive early-stopping procedures that generally lack a certificate. Bridging this gap requires more than another fixed-time convergence bound. It requires a theory in which the evolving trajectory is itself allowed to determine, with statistical validity, when sufficient progress has been made. This suggests a different perspective on stochastic optimization: a running SGD trajectory can be viewed as a \emph{sequential experiment}. Its iterates and stochastic gradients do not merely update the decision variable; they also provide evidence about the unknown optimization error. The stopping problem then becomes one of determining whether the observed trajectory has produced enough evidence to certify that a prescribed accuracy has been reached. This leads to the central question of the paper:

\begin{center}
\emph{Can SGD certify, from its own realized trajectory, when it should stop?}
\end{center}

The statistical theory underlying this question has deep roots in anytime-valid sequential inference, which studies how to construct inferential guarantees that remain valid when observations are examined sequentially and the decision to continue or stop depends on what has already been observed \cite{ville1939,wald1945sequential,robbins1952some,ramdas2023game}. In the present setting, the relevant objects are confidence sequences: observable, trajectory-dependent upper bounds that, with high probability, bound a target quantity, such as suboptimality or distance to the optimizer, simultaneously at every iteration \cite{darling1967confidence,lai1976confidence,howard2021time}. SGD may therefore be monitored continuously and stopped at the first time the bound falls below a prescribed tolerance, while the certificate remains valid at the selected stopping time. Constructing sharp confidence sequences for SGD requires optimization-specific tools because the unknown errors evolve recursively with the iterates and are driven by the same stochastic gradients from which the certificates must be constructed.

A first step toward meeting this challenge was taken by \citet{aolaritei2025stopping}, who developed anytime-valid confidence sequences and certified stopping rules for SGD under general convexity and smooth nonconvexity. Specifically, that work constructed fully observable certificates for suboptimality in the convex setting and for stationarity in the smooth nonconvex setting. These certificates allow the algorithm to be monitored continuously and stopped at a data-dependent time while retaining a valid guarantee.

In the present paper, we substantially expand this line of work by developing a considerably sharper theory for strongly convex stochastic optimization that exploits the contractive structure of the dynamics. Retaining trajectory adaptivity and time-uniform validity in this setting requires new constructions, including a recursive confidence-sequence argument and an empirical Bernstein theory for adapted processes with time-varying conditional means and growing predictable ranges. The resulting fully observable certificates bound the squared distance of the last iterate to the optimizer and the suboptimality of a weighted-average iterate. They recover the canonical $1/t$ rate, up to the generally unavoidable $\log\log t$ factors associated with time-uniform validity \cite{darling1967iterated,howard2021time}, and adapt to the stochastic behavior observed along the realized trajectory. In our numerical experiments, the corresponding stopping rules terminate orders of magnitude earlier than natural deterministic horizons obtained by inverting conventional fixed-time guarantees.


\subsection{Problem formulation}
\label{subsec:problem:form}

We consider stochastic optimization problems of the form
\[
    \min_{x \in \cX} f(x),
\]
where $\cX \subseteq \R^d$ is a closed convex set and $f : \cX \to \R$ is differentiable.\footnote{The analysis extends directly to the nondifferentiable setting by replacing gradients and stochastic gradients with subgradients and stochastic subgradients, respectively.} Throughout, we assume that the objective satisfies the following condition.

\begin{assumption}[Strong convexity]
\label{assump:strong-convexity}
The function $f$ is $\mu$-strongly convex on $\cX$, meaning that
\[
    f(y)
    \ge
    f(x)
    +
    \inner{\nabla f(x)}{y - x}
    +
    \frac{\mu}{2} \norm{y - x}^2,
    \qquad
    \forall x, y \in \cX.
\]
\end{assumption}

Under Assumption~\ref{assump:strong-convexity}, the minimizer is unique; we denote it by $x^\star$. We study projected stochastic gradient descent (SGD), defined recursively by
\begin{equation}
\label{eq:proj-sgd-update}
    x_{t+1}
    =
    \proj_{\cX}\bigl(x_t - \eta_t g_t\bigr),
    \qquad t \ge 1,
\end{equation}
and initialized from $x_1 \in \cX$ almost surely. Here, $\proj_{\cX}$ denotes the Euclidean projection onto $\cX$, $\{g_t\}_{t \ge 1}$ are stochastic gradients, and $\{\eta_t\}_{t \ge 1}$ is a stepsize sequence. We work with the natural filtration
\[
    \cF_0 \coloneqq \sigma(x_1),
    \qquad
    \cF_t \coloneqq \sigma(x_1, g_1, \ldots, x_t, g_t),
    \qquad t \ge 1.
\]
Thus, $\cF_t$ contains the complete SGD trajectory observed through iteration $t$. We assume that each stepsize is selected using only information available before the corresponding stochastic gradient is observed.

\begin{assumption}[Predictable stepsizes]
\label{assump:stepsize}
For every $t \ge 1$, the stepsize $\eta_t$ is $\cF_{t-1}$-measurable.
\end{assumption}

Assumption~\ref{assump:stepsize} permits both deterministic and data-dependent stepsize rules, provided they depend only on the past trajectory. It includes constant and diminishing schedules, as well as adaptive choices based on previously observed iterates and stochastic gradients (e.g., AdaGrad-type schedules and related adaptive methods \cite{duchi2011adaptive,ward2020adagrad}). The stochastic gradients satisfy the following standard unbiasedness and tail conditions.

\begin{assumption}[Stochastic gradients]
\label{assump:stochastic-gradients}
For every $t \ge 1$, the stochastic gradient $g_t$ satisfies:
\begin{itemize}
    \item[(i)] \emph{Unbiasedness:}
    \[
        \E[g_t \mid \cF_{t-1}]
        =
        \nabla f(x_t).
    \]

    \item[(ii)] \emph{Conditional sub-Gaussian noise:}
    defining $\xi_t \coloneqq g_t - \nabla f(x_t)$, there exists a constant $\sigma^2 > 0$ such that, for every $\lambda \in \R$ and every $u \in \R^d$,
    \begin{equation}
    \label{eq:cond-subg-noise}
        \E\left[
            \exp \bigl( \lambda \inner{u}{\xi_t} \bigr)
            \,\middle|\,
            \cF_{t-1}
        \right]
        \le
        \exp\left(
            \frac{\lambda^2 \sigma^2 \norm{u}^2}{2}
        \right)
        \qquad\text{almost surely.}
    \end{equation}
\end{itemize}
\end{assumption}

Assumption~\ref{assump:stochastic-gradients}(i) is the standard conditional unbiasedness requirement in stochastic approximation \cite{nemirovski2009robust}. Assumption~\ref{assump:stochastic-gradients}(ii) imposes a conditional sub-Gaussian bound on the gradient noise while allowing its distribution to vary along the trajectory. This condition is satisfied by Gaussian and many light-tailed noise models. It also holds whenever $\norm{g_t} \le G$ almost surely, in which case one may take $\sigma^2 = G^2$ by the conditional Hoeffding lemma \cite{hoeffding1963probability}.

To ensure that the resulting contraction factors are positive, we impose an upper bound on the stepsizes after an initial period.

\begin{assumption}[Upper bound on stepsizes]
\label{assump:stepsize-bound}
There exists a deterministic integer $t_0 \ge 1$ such that
\[
    0 < \eta_t < \frac{1}{2 \mu}
    \qquad\text{almost surely for every }t \ge t_0.
\]
\end{assumption}

Assumption~\ref{assump:stepsize-bound} is satisfied by the standard diminishing stepsize schedules used for strongly convex stochastic optimization. It is imposed only from the deterministic time $t_0$ onward, allowing arbitrary predictable stepsizes during the initial phase of the algorithm. The index $t_0$ serves as the starting point of our time-uniform analysis.

Our goal is to provide \emph{anytime-valid, trajectory-adaptive guarantees} for the progress of SGD. Given the stochastic gradients, predictable stepsizes, and observed iterates, we seek upper bounds on suitable performance measures that remain valid simultaneously over the entire run. We formalize this objective through the following definition.

\begin{definition}[Anytime-valid upper confidence sequence]
\label{def:upper-confidence-sequence}
Let $\{Q_t\}_{t \ge t_0}$ be a real-valued stochastic process. For any $\alpha \in (0,1)$, a sequence of random variables $\{U_t(\alpha)\}_{t \ge t_0}$ is an \emph{anytime-valid upper confidence sequence} for $\{Q_t\}_{t \ge t_0}$ at level $1 - \alpha$ if
\[
    \Prob\left(
        \forall t \ge t_0 :\
        Q_t \le U_t(\alpha)
    \right)
    \ge
    1 - \alpha.
\]
\end{definition}

When $Q_t$ represents a performance measure of the current optimization trajectory, such as the squared distance to the optimizer or a notion of suboptimality, Definition~\ref{def:upper-confidence-sequence} provides a sequence of upper bounds that are valid simultaneously over time. Consequently, the bounds may be monitored continuously and evaluated at a time selected from the observed trajectory without invalidating the coverage guarantee. 

To evaluate our confidence sequences, we assume access to the problem constants that enter their construction. In particular, the strong-convexity parameter $\mu$ is assumed known, and $\sigma^2$ is a known valid upper bound on the conditional sub-Gaussian proxy in Assumption~\ref{assump:stochastic-gradients}(ii). In the bounded-gradient setting considered later, we assume that a known constant $G$ satisfies $\norm{g_t}\le G$ almost surely, in which case one may take $\sigma^2=G^2$ as explained earlier. The fully observable constructions additionally require an observable upper bound $R_0$ on $\norm{x_{t_0}-x^\star}^2$; under bounded gradients, the choice $R_0=G^2/\mu^2$ is always valid. By contrast, $x^\star$, $f(x^\star)$, and the performance measures considered below are not assumed known. Importantly, $G$ and $R_0$ need only be valid upper bounds rather than tightly calibrated values: the numerical experiments in Section~\ref{sec:experiments-misspecification} show that, for the empirical-Bernstein confidence sequences, the effect of conservative choices of these quantities is largely transient.

We now introduce the two performance measures considered in this paper and the stopping rules induced by their corresponding upper confidence sequences.

\bigskip
\noindent\textbf{Squared distance to the optimizer.}
Our first performance measure is
\begin{equation}
\label{eq:def-Zt}
    Z_t \coloneqq \norm{x_t - x^\star}^2.
\end{equation}
This quantity measures the accuracy of the current iterate and therefore defines a natural last-iterate stopping criterion: ideally, the algorithm would stop as soon as $Z_t \le \varepsilon$. Since $x^\star$ is unknown, however, this criterion is not directly observable. An anytime-valid upper confidence sequence provides a natural trajectory-adaptive surrogate.

At time $t$, after observing $g_t$, the algorithm computes $x_{t+1}$, so $Z_{t+1}$ is $\cF_t$-measurable. We therefore seek a computable, $\cF_t$-measurable upper bound for $Z_{t+1}$. Suppose that $\bigl\{U^{\mathrm{dist}}_{t+1}(\alpha)\bigr\}_{t \ge t_0}$ is an anytime-valid upper confidence sequence for $\{Z_{t+1}\}_{t \ge t_0}$ with this property. Then $U^{\mathrm{dist}}_{t+1}(\alpha)$ certifies the accuracy of the newly computed iterate $x_{t+1}$, leading to the stopping rule
\begin{equation}
\label{eq:distance-stopping-time}
    \tau^{\mathrm{dist}}_\varepsilon(\alpha)
    \coloneqq
    \inf\left\{
        t \ge t_0 :\
        U^{\mathrm{dist}}_{t+1}(\alpha)
        \le
        \varepsilon
    \right\},
\end{equation}
with the convention $\inf \varnothing \coloneqq \infty$. Consequently, $\tau^{\mathrm{dist}}_\varepsilon(\alpha)$ is a stopping time with respect to $\{\cF_t\}_{t \ge t_0}$. When finite, it returns $x_{\tau^{\mathrm{dist}}_\varepsilon(\alpha)+1}$, the first iterate certified to lie within squared distance $\varepsilon$ of $x^\star$.

\bigskip
\noindent\textbf{Weighted average suboptimality.}
Our second performance measure is the linearly weighted average suboptimality
\begin{equation}
\label{eq:weighted-suboptimality}
    \bar F_t
    \coloneqq
    \frac{1}{S_t}
    \sum_{s = t_0}^t
    s \bigl( f(x_s) - f(x^\star) \bigr),
    \qquad
    S_t
    \coloneqq
    \sum_{s = t_0}^t s,
    \qquad
    t \ge t_0.
\end{equation}
The corresponding weighted average iterate is
\begin{equation}
\label{eq:weighted-average-iterate}
    \bar x_t
    \coloneqq
    \frac{1}{S_t}
    \sum_{s = t_0}^t s\,x_s.
\end{equation}
By convexity of $f$,
\[
    f(\bar x_t) - f(x^\star)
    \le
    \bar F_t.
\]
Thus, an upper bound on $\bar F_t$ certifies the suboptimality of the observable iterate $\bar x_t$. The linearly increasing weights emphasize later iterates and were introduced in \cite{lacoste2012simpler}, together with the stepsize $\eta_t = 2/(\mu(t+1))$, to recover the optimal convergence rate for strongly convex stochastic optimization; we adopt the same stepsize for this performance metric.

Suppose that $\bigl\{U^{\mathrm{sub}}_t(\alpha)\bigr\}_{t \ge t_0}$ is an anytime-valid upper confidence sequence for $\{\bar F_t\}_{t \ge t_0}$ and that $U^{\mathrm{sub}}_t(\alpha)$ is $\cF_t$-measurable for every $t \ge t_0$. Using the same target accuracy $\varepsilon > 0$, define
\begin{equation}
\label{eq:suboptimality-stopping-time}
    \tau^{\mathrm{sub}}_\varepsilon(\alpha)
    \coloneqq
    \inf\left\{
        t \ge t_0 :\
        U^{\mathrm{sub}}_t(\alpha)
        \le
        \varepsilon
    \right\},
\end{equation}
with the convention $\inf \varnothing \coloneqq \infty$. When finite, this rule returns $\bar x_{\tau^{\mathrm{sub}}_\varepsilon(\alpha)}$, the first weighted average iterate certified to be $\varepsilon$-suboptimal.

The validity of the two stopping rules follows directly from the simultaneous-coverage property.

\begin{proposition}[Certified stopping rules]
\label{prop:certified-stopping-rules}
Fix $\alpha \in (0,1)$ and $\varepsilon > 0$.
\begin{itemize}
    \item[(i)] If $\{U^{\mathrm{dist}}_{t+1}(\alpha)\}_{t \ge t_0}$ is an anytime-valid upper confidence sequence for $\{Z_{t+1}\}_{t \ge t_0}$, then
    \[
        \Prob\left(
            \{\tau^{\mathrm{dist}}_\varepsilon(\alpha) < \infty\}
            \cap
            \left\{
                \norm{
                    x_{\tau^{\mathrm{dist}}_\varepsilon(\alpha)+1}
                    - x^\star
                }^2
                >
                \varepsilon
            \right\}
        \right)
        \le
        \alpha.
    \]

    \item[(ii)] If $\{U^{\mathrm{sub}}_t(\alpha)\}_{t \ge t_0}$ is an anytime-valid upper confidence sequence for $\{\bar F_t\}_{t \ge t_0}$, then
    \[
        \Prob\left(
            \{\tau^{\mathrm{sub}}_\varepsilon(\alpha) < \infty\}
            \cap
            \left\{
                f\bigl(
                    \bar x_{\tau^{\mathrm{sub}}_\varepsilon(\alpha)}
                \bigr)
                -
                f(x^\star)
                >
                \varepsilon
            \right\}
        \right)
        \le
        \alpha.
    \]
\end{itemize}
Thus, for each stopping rule separately, with probability at least $1 - \alpha$, whenever the stopping time is finite, the returned iterate satisfies the prescribed accuracy requirement.
\end{proposition}

The result follows immediately from Definition~\ref{def:upper-confidence-sequence}. We therefore omit the proof and refer to \cite[Theorem~3.1]{aolaritei2025stopping}, where the same stopping-time argument is carried out. The aim of the current paper is to construct the upper confidence sequences $U^{\mathrm{dist}}_{t+1}(\alpha)$ and $U^{\mathrm{sub}}_t(\alpha)$ explicitly from the observed iterates, stochastic gradients, and stepsizes $\bigl\{(x_s, g_s, \eta_s)\bigr\}_{s = t_0}^t$.


\subsection{Contributions}
\label{subsec:contributions}

Our main contributions can be summarized as follows:
\begin{itemize}
    \item[(i)] \textbf{Confidence sequences under sub-Gaussian noise.}
    Through a time-uniform Hoeffding argument, we construct fully observable, trajectory-adaptive confidence sequences for the two performance measures introduced above. These confidence sequences can be computed online from the observed stochastic gradients and stepsizes. In the worst case, they match the optimal $1/t$ rate for strongly convex stochastic optimization, up to the iterated-logarithmic cost of time-uniform validity, while remaining sensitive to the observed trajectory and therefore potentially smaller than the corresponding worst-case bounds.
    \begin{enumerate}[label=\arabic*., leftmargin=*]
        \item \textbf{Squared distance to the optimizer.}
        We construct an anytime-valid upper confidence sequence $\bigl\{\widehat U^{\mathrm{dist}}_{t+1}(\alpha)\bigr\}_{t \ge t_0}$ whose stochastic term depends explicitly on the conditional sub-Gaussian variance proxy $\sigma^2$. If $\norm{g_t} \le G$ almost surely and $\eta_t = 1/(\mu t)$, then its worst-case decay satisfies
        \begin{equation}
        \label{eq:contrib-rate-distance}
            \widehat U^{\mathrm{dist}}_{t+1}(\alpha)
            =
            O\left(
                \frac{G^2}{\mu^2}
                \frac{(\log(1/\alpha) + \log\log(e t))}{t}
            \right).
        \end{equation}
    
        \item \textbf{Weighted average suboptimality.\footnote{The same approach can be adapted to other averaging schemes, such as uniform averaging. We use linearly increasing weights because they yield the optimal $O(1/t)$ rate, whereas uniform averaging incurs an additional logarithmic factor and achieves an $O(\log t/t)$ rate \cite{lacoste2012simpler}.}}
        Building on the observable distance confidence sequence, we construct an anytime-valid upper confidence sequence $\bigl\{\widehat U^{\mathrm{sub}}_t(\alpha)\bigr\}_{t \ge t_0}$ for weighted average suboptimality. If $\norm{g_t} \le G$ almost surely and $\eta_t = 2/(\mu(t+1))$, then its worst-case decay satisfies
        \begin{equation}
        \label{eq:contrib-rate-suboptimality}
            \widehat U^{\mathrm{sub}}_t(\alpha)
            =
            O\left(
                \frac{G^2}{\mu}
                \frac{(\log(1/\alpha) + \log\log(e t))}{t}
            \right).
        \end{equation}
    \end{enumerate}
    
    \item[(ii)] \textbf{Refined confidence sequences under bounded stochastic gradients.}
    Through a novel time-uniform empirical Bernstein argument (see item~(iv)), we refine both confidence sequences under the assumption that $\norm{g_t}\le G$ almost surely. The resulting bounds replace the fixed worst-case variance proxy $\sigma^2 = G^2$ in the Hoeffding construction by the realized squared stochastic-gradient magnitudes $\norm{g_t}^2$, together with an additional correction that decays at the faster $t^{-3/2}$ rate, up to iterated-logarithmic factors. The refined confidence sequences are fully observable and preserve the worst-case rates in equations~\eqref{eq:contrib-rate-distance} and~\eqref{eq:contrib-rate-suboptimality}. At the same time, they adapts more sharply to the observed trajectory: in our numerical experiments, the resulting confidence sequences are generally one order of magnitude smaller than their Hoeffding-based counterparts. This refinement constitutes the central contribution of the paper.

    \begin{figure}[t]
    \centering
    \includegraphics[width=\textwidth]{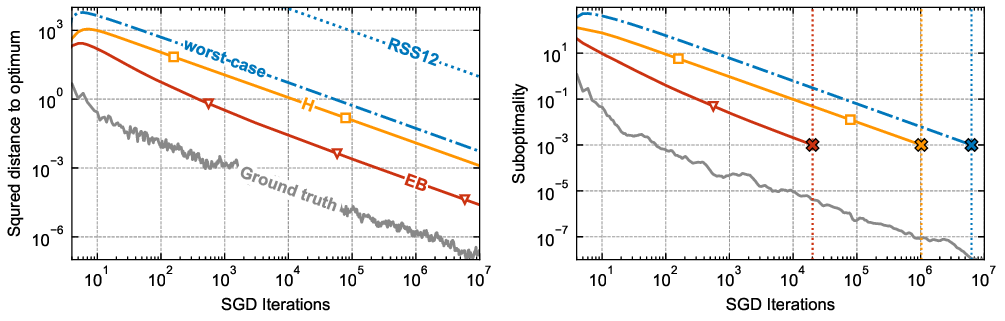}
    \caption{
    Trajectory-adaptive certification of SGD on an SVM problem with minibatch size $b=256$ and confidence level $99\%$.
    \textbf{H} and \textbf{EB} denote our Hoeffding- and empirical-Bernstein-based confidence sequences, respectively;
    \textbf{worst-case} is the trajectory-independent version of our Hoeffding construction obtained by replacing the realized squared stochastic-gradient norms $\norm{g_t}^2$ by their uniform upper bound $G^2$, while \textbf{RSS12} denotes the finite-horizon high-probability bound of \citet{rakhlin2012making}.\protect\footnotemark \ 
    Left: squared distance to the optimizer. Right: suboptimality of the weighted average.
    The gray curves show the corresponding ground truth.
    In the right panel, crosses and vertical lines mark the first iteration at which each bound falls below the target $\varepsilon = 10^{-3}$, and hence the time at which that bound first certifies the prescribed accuracy.
    Further details are provided in Section~\ref{sec:experiments}.
    }
    \label{fig:figure-introduction}
    \end{figure}
    \footnotetext{The bound of \citet{rakhlin2012making} is simultaneous over all $t \le T$, but requires the horizon $T$ to be fixed in advance. More recent fully time-uniform guarantees, such as \citet{pham2025time}, remove the need to prescribe $T$ but are numerically more conservative in this comparison.}

    \item[(iii)] \textbf{Minibatch extensions.}
    We extend both constructions to minibatch SGD with $b$ conditionally independent stochastic gradients per iteration. For the Hoeffding confidence sequences, minibatching reduces the conditional sub-Gaussian variance proxy from $\sigma^2$ to $\sigma^2/b$. For the empirical Bernstein confidence sequences, we develop a minibatch analogue that pools the observed quadratic variation of the individual stochastic gradients and adapts to the largest eigenvalue of the realized second-moment matrix of each minibatch. The minibatch empirical Bernstein boundary has a leading term that improves at the usual $1/\sqrt b$ rate, together with a lower-order correction that improves at the faster $1/b$ rate. In our numerical experiments, the improvement relative to the corresponding worst-case bounds grows with the minibatch size, reaching between two and three orders of magnitude for the larger minibatch sizes considered. The gap relative to existing ex ante guarantees with explicit constants can be substantially larger still; see Figure~\ref{fig:figure-introduction}.

    \item[(iv)] \textbf{A time-uniform empirical Bernstein inequality with growing predictable range.}
    To support the refinement in item~(ii), we establish a general time-uniform empirical Bernstein inequality for adapted processes with time-varying conditional means and predictable increment ranges that may grow without bound. This feature is essential for our SGD application: in both the squared-distance and weighted-average suboptimality analyses, the relevant predictable ranges can scale as $\sqrt t$ under the worst-case convergence rate. Existing empirical Bernstein confidence-sequence constructions, such as \citet[Theorem~4]{howard2021time}, allow the conditional means to vary but assume a fixed range scale, and therefore do not directly cover this setting. When the predictable range evolves over time, the admissible exponential parameter evolves with it. Starting from the pointwise exponential inequality of \citet[Lemma~4.1]{fan2015exponential}, we overcome this difficulty through predictable truncation and simultaneous stitching over dyadic scales of the observed quadratic process and the running predictable range. The resulting explicit boundary adapts jointly to both quantities and is nondecreasing in each of them.
\end{itemize}


\subsection{Related work}

We first discuss the statistical foundations of our approach and their historical connection to stochastic approximation, before turning to the relevant background in strongly convex stochastic optimization.

\medskip
\noindent\textbf{Anytime-valid inference and stochastic approximation.}
The statistical framework underlying our analysis is anytime-valid sequential inference, which develops measures of evidence and uncertainty that remain valid under continuous monitoring and data-dependent stopping. Its motivating pathology is classical: a procedure may be statistically valid at every predetermined sample size and yet lose validity when observations are repeatedly inspected and sampling stops after a favorable fluctuation. An early illustration appears in Feller's 1940 discussion of extra-sensory perception, where repeated sampling could turn chance fluctuations into apparently significant evidence \cite{feller1940statistical}. Anscombe later famously described this practice as ``sampling to a foregone conclusion'' \cite{anscombe1954fixed}, while Robbins placed the same phenomenon within the broader theory of sequential experimentation \cite{robbins1952some}. The martingale foundations go back to Ville \cite{ville1939}, while Wald developed the classical theory of sequential hypothesis testing \cite{wald1945sequential}. Subsequent work by Darling, Robbins, Siegmund, Lai, and others developed confidence sequences that are valid at any and all times, boundary-crossing methods, and tests of power one \cite{darling1967confidence,darling1968some,robbins1970statistical,robbins1970boundary,lai1976confidence}. After remaining a comparatively small niche for decades, these ideas reemerged rapidly around 2017 under the modern framework of \emph{safe anytime-valid inference} \cite{ramdas2023game}. During 2018--2020, several independent lines of work with closely aligned motivations and techniques developed the modern theory of \emph{e-values}, and the term itself was jointly adopted in its present measure-theoretic sense in 2020; see \cite{vovk2021evalues,shafer2021testing,grunwald2024safe,ramdas2025hypothesis} and the references therein. The terminology is recent, but the underlying objects trace back to likelihood ratios, nonnegative martingales, and Ville's betting interpretation.

A central technical ingredient in this framework is the construction of time-uniform boundaries that control a stochastic process uniformly over time. Such boundaries are commonly obtained from exponential supermartingales through \emph{mixtures} or \emph{stitching}. The method of mixtures integrates over the exponential tuning parameter; in the literature on self-normalized processes, this technique is also known as pseudo-maximization and provides a way to adapt to a random intrinsic time \cite{de2004self,de2007pseudo,de2009self}. Stitching instead allocates tuning parameters and error probabilities across successive intrinsic-time epochs and combines the resulting bounds through a union bound. Epoch-based arguments of this kind go back at least to \citet{darling1967iterated}, while the approach was developed systematically for modern confidence sequences by \citet{howard2021time}. Our empirical Bernstein analysis follows the stitching route, but must also accommodate a predictable range that evolves with the SGD trajectory; we therefore stitch simultaneously over the observed quadratic process and the running range scale.

This history also includes a close connection to stochastic approximation. Robbins and Monro introduced stochastic approximation \cite{robbins1951stochastic}, while the Robbins--Siegmund almost-supermartingale theorem later became a standard tool for establishing the convergence of stochastic approximation algorithms \cite{robbins1971convergence}. Lai likewise contributed to adaptive stochastic approximation jointly with Robbins \cite{lai1979adaptive}. The two traditions therefore share important intellectual origins, but their modern theories developed largely separately. \citet{aolaritei2025stopping} brought these perspectives together by constructing trajectory-adaptive confidence sequences and certified stopping rules for convex and smooth nonconvex SGD. The present paper continues this line of work in the strongly convex setting, where contractivity permits substantially sharper certificates but also requires new recursive and self-normalized sequential arguments.

\medskip
\noindent\textbf{Strongly convex stochastic optimization.}
Classical work by Polyak and Juditsky established the asymptotic optimality of trajectory averaging under suitable regularity conditions \cite{polyak1992acceleration}. Building on this asymptotic theory, \citet{moulines2011non} gave nonasymptotic analyses of standard and averaged SGD for smooth convex objectives, showing that inverse-time stepsizes attain the optimal rate under strong convexity, while slower decay combined with averaging is more robust to stepsize calibration and remains effective in the absence of strong convexity. Subsequent work obtained optimal $O(1/T)$ rates through specialized algorithms and averaging schemes. \citet{hazan2011beyond} proposed an epoch-based method for strongly convex stochastic optimization without smoothness assumptions, while \citet{ghadimi2012optimal,ghadimi2013optimal} developed accelerated and multistage schemes for strongly convex composite problems, together with optimal or nearly optimal complexity bounds and large-deviation guarantees. For standard SGD, \citet{rakhlin2012making} showed that smoothness yields the optimal $O(1/T)$ rate, while suffix averaging recovers the same rate without smoothness. They also obtained a finite-horizon high-probability squared-distance bound valid simultaneously for all $t\le T$, with an additional iterated-logarithmic factor. Further weighted-averaging schemes attaining the optimal rate were developed by \citet{lacoste2012simpler}, using linearly increasing weights for projected stochastic subgradient descent, and by \citet{nedic2014stochastic} for stochastic mirror descent. Extensions include adaptation to unknown local strong convexity in logistic regression \cite{bach2014adaptivity} and trajectory averaging on Riemannian manifolds for geodesically strongly convex objectives \cite{tripuraneni2018averaging}.

The gap between averaged and last-iterate performance motivated a complementary line of work. For nonsmooth strongly convex objectives, \citet{shamir2013stochastic} established an $O(\log T/T)$ last-iterate rate in expectation and proposed an online averaging scheme attaining the optimal $O(1/T)$ rate. \citet{harvey2019tight} proved the corresponding high-probability last-iterate bound, established its tightness through a matching deterministic example, and obtained an optimal $O(1/T)$ high-probability guarantee for suffix averaging.  For the linearly weighted average of \citet{lacoste2012simpler}, \citet{harvey2019simple} established the optimal $O(\log(1/\delta)/T)$ high-probability rate. \citet{jain2021making} instead designed horizon-dependent stepsize sequences that transfer optimal guarantees for averaged SGD to the last iterate. More recently, \citet{liu2024revisiting} developed a unified last-iterate analysis in expectation and with high probability across general domains, composite objectives, non-Euclidean geometries, and both convex and strongly convex settings.

Recent work has derived high-probability convergence rates for strongly convex SGD that hold uniformly over time. \citet{pham2025time} obtained such a rate for the last-iterate squared distance, matching the order in~\eqref{eq:contrib-rate-distance}, while \citet{chen2026revisiting} established the corresponding order in~\eqref{eq:contrib-rate-suboptimality} for the last-iterate objective gap under smoothness. Although time-uniform, these guarantees remain ex ante and worst-case: their envelopes depend only on the iteration count and global problem parameters, rather than adapting to the realized trajectory. For the accuracy-based stopping problem considered here, time uniformity alone therefore offers no advantage over a finite-horizon guarantee: thresholding either bound yields a deterministic worst-case iteration budget, while the time-uniform bound may be more conservative because of its additional iterated-logarithmic or constant-factor cost, as illustrated by the comparison between \citet{pham2025time} and \citet{rakhlin2012making}. The confidence sequences developed here instead adapt to the stochastic fluctuations observed along the realized trajectory, while providing fully observable certificates that can be evaluated throughout the run and used for online stopping, with the same decay orders in the worst case. This distinction is quantitatively substantial in our experiments: in Figure~\ref{fig:figure-introduction}, our trajectory-adaptive distance certificate is roughly five orders of magnitude smaller than the finite-horizon bound of \citet{rakhlin2012making}, with an even larger gap relative to the time-uniform bound of \citet{pham2025time}.


\subsection{Organization and notation}

\noindent\textbf{Organization.}
Section~\ref{sec:subgaussian-confidence-sequences} develops our confidence sequences under conditional sub-Gaussian noise using a time-uniform Hoeffding argument. Section~\ref{sec:refined-confidence-sequences} refines these constructions under bounded stochastic gradients through a time-uniform empirical Bernstein approach. Section~\ref{sec:minibatching} extends both approaches to minibatch SGD. Finally, Section~\ref{sec:experiments} numerically evaluates the resulting confidence sequences and their certified stopping rules. All proofs are deferred to Appendix~\ref{sec:proofs}.

\medskip
\noindent\textbf{Notation.}
For $x, y \in \R^d$, $\norm{x}$ denotes the Euclidean norm and $\inner{x}{y}$ the standard inner product. For a closed convex set $\cX \subseteq \R^d$, $\proj_{\cX}(z)$ denotes the Euclidean projection of $z$ onto $\cX$. For a symmetric matrix $M$, $\lambda_{\max}(M)$ denotes its largest eigenvalue. Given a filtration $\{\cF_t\}_{t \ge 0}$, a process $\{X_t\}$ is adapted if $X_t$ is $\cF_t$-measurable and predictable if $X_t$ is $\cF_{t-1}$-measurable. Conditional expectations are written $\E[\,\cdot \mid \cF_{t-1}]$, and we adopt the convention $\inf\varnothing = \infty$. Unless a base is indicated explicitly, all logarithms are natural. For two nonnegative sequences ${a_t}$ and ${b_t}$, we write $a_t = O(b_t)$ if there exists a constant $C < \infty$ such that $a_t \le C b_t$ throughout the stated range; we write $a_t\lesssim b_t$ synonymously with $a_t = O(b_t)$, and $a_t \asymp b_t$ if $a_t \lesssim b_t$ and $b_t \lesssim a_t$. Unless stated otherwise, implicit constants may depend on fixed problem parameters but are independent of $t$, $\alpha$, and the minibatch size $b$ whenever these quantities vary in the corresponding statement. All probabilities and expectations are taken with respect to the underlying algorithmic randomness.


\section{Confidence Sequences under Sub-Gaussian Noise}
\label{sec:subgaussian-confidence-sequences}

We construct anytime-valid upper confidence sequences for the two performance measures of interest: the last-iterate squared distance $Z_{t+1}$ and the suboptimality of the weighted-average iterate, $f(\bar x_t) - f(x^\star)$. Although the two guarantees arise from different SGD recursions, both reduce to controlling partial sums of directional gradient noise terms. We first isolate the common time-uniform concentration argument, then construct the observable distance confidence sequence, which will in turn be used to make the suboptimality confidence sequence fully observable.

The common probabilistic ingredient underlying both constructions is a conditional exponential moment bound of sub-Gaussian form. Consider an adapted sequence of real-valued random variables $\{X_t\}_{t \ge t_0}$ and a nonnegative predictable sequence $\{v_t\}_{t \ge t_0}$. Suppose that, for some $\gamma \in \R$,
\[
    \E\left[
        \exp\bigl(\lambda X_t \bigr)
        \,\middle|\,
        \cF_{t-1}
    \right]
    \le
    \exp \bigl( 2^\gamma \lambda^2 v_t \bigr)
    \qquad
    \text{almost surely}
\]
for every $\lambda > 0$ and every $t \ge t_0$. The predictable quantity $v_t$ serves as a conditional variance proxy for $X_t$, using only information available before $X_t$ is observed, while $\gamma$ records the constant arising in the particular application. For each fixed $\lambda > 0$, the conditional exponential bound implies that
\[
    \exp\left(
        \lambda\sum_{s = t_0}^t X_s
        -
        2^\gamma \lambda^2 \sum_{s = t_0}^t v_s
    \right)
\]
is a nonnegative supermartingale. The cumulative variance proxy $\sum_{s = t_0}^t v_s$ records the conditional fluctuation scale accumulated over the process. In sequential inference, such a cumulative process is often called an \emph{intrinsic time}: the scale of the concentration boundary for the partial sum adapts to the accumulated variance proxy, rather than being indexed solely by the iteration count.

This exponential supermartingale is the starting point for a time-uniform concentration bound. Ville's inequality controls its running maximum and, after rearranging, yields for each fixed $\lambda>0$ a linear upper boundary for the partial sums that holds simultaneously over time \cite{ville1939}. No single choice of $\lambda$, however, is well calibrated across all possible scales of the intrinsic time: different values of $\lambda$ yield sharper boundaries in different ranges of the cumulative variance proxy. Dyadic stitching combines these fixed-$\lambda$ boundaries over successive dyadic scales, while allocating the error probability across the scales. The result is a curved boundary that adapts to the realized intrinsic time and remains valid uniformly over the infinite horizon.

The following lemma formalizes this construction in the precise form used throughout the section. It is a dyadic specialization of the stitching framework developed by \citet{howard2021time}. 

\begin{lemma}[Anytime-valid stitched sub-Gaussian bound]
\label{prop:dyadic-stitching}
Let $\{\cF_t\}_{t \ge t_0-1}$ be a filtration. Let $\{X_t\}_{t \ge t_0}$ be adapted, and let $\{v_t\}_{t \ge t_0}$ be nonnegative and predictable. Define
\[
    \bar X_{t_0-1} \coloneqq 0,
    \qquad
    \bar X_t \coloneqq \sum_{s = t_0}^t X_s,
    \qquad
    V_{t_0-1} \coloneqq 0,
    \qquad
    V_t \coloneqq \sum_{s = t_0}^t v_s .
\]
Fix $\gamma \in \R$. Suppose that, for every $\lambda>0$ and every $t \ge t_0$,
\[
    \E\left[
        \exp\left(
            \lambda X_t - 2^\gamma \lambda^2 v_t
        \right)
        \,\middle|\,
        \cF_{t-1}
    \right]
    \le 1 .
\]
For $t \ge t_0-1$, let
\[
    V_{t,\mathrm{eff}}\coloneqq\max\{V_t, 1\},
    \qquad
    m_t \coloneqq \left\lceil \log_2 V_{t,\mathrm{eff}}\right\rceil .
\]
Then, for every $\alpha \in (0,1)$,
\[
    \Prob\left(
        \forall t \ge t_0:\
        \bar X_t
        \le
        2^{(\gamma+3)/2}
        \sqrt{
            V_{t,\mathrm{eff}}
            \left(
                \log\frac{1}{\alpha}
                +
                \log\bigl((m_t+1)(m_t+2)\bigr)
            \right)
        }\,
    \right)
    \ge
    1-\alpha .
\]
\end{lemma}

Lemma~\ref{prop:dyadic-stitching} provides the generic concentration step, but applying it to SGD requires a problem-specific analysis of the algorithmic dynamics. For each performance measure, we derive a recursion that isolates the contribution of the gradient noise as a partial sum of directional terms, and then establish the required conditional exponential moment bound using a predictable variance proxy.


\subsection{Last-iterate squared distance to the optimizer}
\label{subsec:last-iterate-squared-distance}

We begin by constructing an upper confidence sequence for the last-iterate squared distance to the optimizer. Recall from~\eqref{eq:def-Zt} that $Z_t=\norm{x_t-x^\star}^2$. To account for the time-varying contraction induced by strong convexity, define
\begin{equation}
\label{eq:def-at-At}
    a_t
    \coloneqq
    1 - 2 \mu \eta_t,
    \qquad
    A_{t_0 - 1}
    \coloneqq
    1,
    \qquad
    A_t
    \coloneqq
    \prod_{s = t_0}^t a_s^{-1},
    \quad
    t \ge t_0.
\end{equation}
By Assumption~\ref{assump:stepsize-bound}, $a_t \in (0,1)$ almost surely for every $t \ge t_0$, so $A_t$ is well defined. Following the notation of Lemma~\ref{prop:dyadic-stitching}, for $t \ge t_0$ define the adapted term
\[
    X_t
    \coloneqq
    2 A_t \eta_t
    \inner{x^\star - x_t}{\xi_t},
    \qquad
    \bar X_t
    \coloneqq
    \sum_{s = t_0}^t X_s.
\]
With these definitions, Lemma~\ref{lem:distance-decomposition} yields the following decomposed upper bound on the last-iterate distance: for every $t \ge t_0$,
\begin{equation}
\label{eq:distance-normalized-decomposition}
    Z_{t+1}
    \le
    \frac{1}{A_t}
    \left[
        Z_{t_0}
        +
        \sum_{s = t_0}^t
        A_s \eta_s^2 \norm{g_s}^2
        +
        \bar X_t
    \right].
\end{equation}
This decomposition follows from the standard squared-distance Lyapunov argument for projected stochastic approximation; see, e.g., \citet[Section~2.1]{nemirovski2009robust}. A closely related unrolling of the contractive SGD recursion for the classical stepsize $\eta_t = 1/(\mu t)$ appears in \citet[Appendix~B.7, Lemma~6]{rakhlin2012making}; the decomposition~\eqref{eq:distance-normalized-decomposition} records the same telescoping mechanism for general predictable stepsizes.

Thus, following~\eqref{eq:distance-normalized-decomposition}, the last-iterate distance to the optimizer is controlled by an initialization term, an observable accumulation of squared stochastic gradients, and a partial sum of directional gradient noise terms. A time-uniform bound on $\bar X_t$ therefore translates directly into an upper confidence sequence for $Z_{t+1}$. To control this partial sum, define
\[
    v_t^{\mathrm{dist}}
    \coloneqq
    \sigma^2 A_t^2 \eta_t^2 Z_t,
    \qquad
    t \ge t_0.
\]
The sequence $\{v_t^{\mathrm{dist}}\}$ is predictable because $x_t$, $\eta_t$, and $A_t$ are $\cF_{t-1}$-measurable. Applying the conditional sub-Gaussian bound in Assumption~\ref{assump:stochastic-gradients}(ii) with the predictable direction $2 A_t \eta_t (x^\star - x_t)$ gives
\[
    \E\left[
        e^{\lambda X_t}
        \,\middle|\,
        \cF_{t-1}
    \right]
    \le
    e^{2\lambda^2v_t^{\mathrm{dist}}}
    \qquad
    \text{almost surely},
\]
for every $\lambda \ge 0$ and every $t \ge t_0$. Thus, $v_t^{\mathrm{dist}}$ is the predictable intrinsic-time increment in the generic construction, and Lemma~\ref{prop:dyadic-stitching} applies with $\gamma = 1$.

To express the resulting confidence sequences compactly, for $v \ge 0$ define
\begin{equation}
\label{eq:v_eff}
    v_{\mathrm{eff}}
    \coloneqq
    \max\{v,1\},
    \qquad
    m(v)
    \coloneqq
    \left\lceil \log_2 v_{\mathrm{eff}}\right\rceil,
\end{equation}
and, for $\alpha \in (0,1)$, let
\begin{equation}
\label{eq:subgaussian-stitched-boundary}
    \mathfrak H_\alpha(v)
    \coloneqq
    \sqrt{
        v_{\mathrm{eff}}
        \left(
            \log\frac{1}{\alpha}
            +
            \log\bigl((m(v)+1)(m(v)+2)\bigr)
        \right)
    }.
\end{equation}
The function $\mathfrak H_\alpha$ collects the common dependence of the stitched boundaries on the intrinsic time and its dyadic scale, and is nondecreasing in $v$. For the distance process, define
\[
    V_t^{\mathrm{dist}}
    \coloneqq
    \sum_{s=t_0}^t v_s^{\mathrm{dist}},
    \qquad
    t \ge t_0.
\]
This is the cumulative intrinsic time associated with the directional gradient noise terms. Since the conditional exponential bound corresponds to $\gamma=1$, Lemma~\ref{prop:dyadic-stitching} guarantees that
\[
    \bar X_t
    \le
    4\, \mathfrak H_\alpha \left( V_t^{\mathrm{dist}} \right)
\]
simultaneously for every $t \ge t_0$ with probability at least $1-\alpha$. Combining this bound with~\eqref{eq:distance-normalized-decomposition} gives the following proposition.

\begin{proposition}[Confidence sequence for $\norm{x_{t+1}-x^\star}^2$]
\label{thm:adaptive-distance-CS}
Suppose Assumptions~\ref{assump:strong-convexity}--\ref{assump:stepsize-bound} hold. Then, for every $\alpha\in(0,1)$, the process
\begin{equation}
\label{eq:Udist}
    U^{\mathrm{dist}}_{t+1}(\alpha)
    \coloneqq
    \frac{1}{A_t}
    \left[
        Z_{t_0}
        +
        \sum_{s = t_0}^t
        A_s \eta_s^2 \norm{g_s}^2
        +
        4\, \mathfrak H_\alpha \left( V_t^{\mathrm{dist}} \right)
    \right],
    \qquad
    t \ge t_0,
\end{equation}
is an upper confidence sequence for the last-iterate squared distance to the optimizer, in the sense that
\[
    \Prob\left(
        \forall t \ge t_0:\
        Z_{t+1}
        \le
        U^{\mathrm{dist}}_{t+1}(\alpha)
    \right)
    \ge
    1-\alpha.
\]
\end{proposition}

Although the squared-gradient accumulation in~\eqref{eq:Udist} is observable from the SGD trajectory, the confidence sequence is not yet fully observable: the initial distance $Z_{t_0}$ is unknown, and the the boundary term $4\, \mathfrak H_\alpha \left( V_t^{\mathrm{dist}} \right)$ depends, through $V_t^{\mathrm{dist}}$, on the unknown distances $\{Z_s\}_{s = t_0}^t$. If $\cX$ is compact, an immediate remedy is to replace every $Z_s$ by the uniform bound $\diam(\cX)^2$. This produces a fully observable confidence sequence, but discards the contraction of the iterates toward the optimizer. As shown in Remark~\ref{remark:rate-diameter}, the resulting bound exhibits $t^{-1/2}$ polynomial decay, rather than the near-$1/t$ decay achieved by the recursive construction below.

To obtain a sharper observable bound, we exploit the simultaneous-coverage guarantee of Proposition~\ref{thm:adaptive-distance-CS}. Define the event
\[
    E_\alpha
    \coloneqq
    \left\{
        \forall t \ge t_0:\
        Z_{t+1}
        \le
        U^{\mathrm{dist}}_{t+1}(\alpha)
    \right\}.
\]
On $E_\alpha$, which has probability at least $1-\alpha$, every subsequent unknown distance is bounded by the corresponding upper confidence bound, that is, $Z_s\le U^{\mathrm{dist}}_s(\alpha)$ for all $s \ge t_0+1$. The bounds $U^{\mathrm{dist}}_s(\alpha)$ are themselves not observable, so they cannot be inserted directly into $V_t^{\mathrm{dist}}$. Instead, starting from an observable upper bound on $Z_{t_0}$, we construct observable envelopes $\widehat U^{\mathrm{dist}}_s(\alpha)$ recursively. Suppose that, up to the current time, these envelopes satisfy $U^{\mathrm{dist}}_s(\alpha) \le \widehat U^{\mathrm{dist}}_s(\alpha)$ on $E_\alpha$. Then
\[
    Z_s
    \le
    U^{\mathrm{dist}}_s(\alpha)
    \le
    \widehat U^{\mathrm{dist}}_s(\alpha),
\]
so the observable envelopes can be used to upper bound the unknown distances in the intrinsic-time process. The resulting observable intrinsic time dominates $V_t^{\mathrm{dist}}$ and, since $\mathfrak H_\alpha$ is nondecreasing, produces a boundary term that dominates $4 \, \mathfrak H_\alpha\left(V_t^{\mathrm{dist}}\right)$. Substituting this larger boundary term into~\eqref{eq:Udist} yields the next observable envelope and preserves the required domination, thereby closing the construction recursively.

We now formalize this construction. Fix $\alpha \in (0,1)$, and let $R_0$ be any $\cF_{t_0 - 1}$-measurable quantity satisfying $R_0 \ge Z_{t_0}$ almost surely. Define recursively
\begin{equation}
\label{eq:distance-radius-proxy}
    R_{t_0}^{\mathrm{dist}}
    \coloneqq
    R_0,
    \qquad
    R_s^{\mathrm{dist}}
    \coloneqq
    \widehat U_s^{\mathrm{dist}}(\alpha),
    \quad
    s \ge t_0+1.
\end{equation}
For every $t \ge t_0$, define the observable intrinsic time
\begin{equation}
\label{eq:observable-distance-intrinsic-time}
    \widehat V_t^{\mathrm{dist}}
    \coloneqq
    \sum_{s = t_0}^t
    \sigma^2 A_s^2 \eta_s^2 R_s^{\mathrm{dist}},
\end{equation}
and, using the notation in~\eqref{eq:v_eff} and the stitched boundary function in~\eqref{eq:subgaussian-stitched-boundary}, define the observable upper envelope
\begin{equation}
\label{eq:observable-distance-CS}
    \widehat U_{t+1}^{\mathrm{dist}}(\alpha)
    \coloneqq
    \frac{1}{A_t}
    \left[
        R_0
        +
        \sum_{s = t_0}^t
        A_s \eta_s^2 \norm{g_s}^2
        +
        4\,\mathfrak H_\alpha\left(\widehat V_t^{\mathrm{dist}}\right)
    \right],
    \qquad
    t\ge t_0.
\end{equation}

The recursion is well defined. When the recursion is evaluated at time $t$, the proxies $R_{t_0}^{\mathrm{dist}}, \ldots, R_t^{\mathrm{dist}}$ are already available: $R_{t_0}^{\mathrm{dist}} = R_0$, while, for $t \ge t_0+1$, $R_t^{\mathrm{dist}} = \widehat U_t^{\mathrm{dist}}(\alpha)$ was computed at time $t-1$. They determine the observable intrinsic time $\widehat V_t^{\mathrm{dist}}$, and hence the boundary term $4\,\mathfrak H_\alpha\left(\widehat V_t^{\mathrm{dist}}\right)$ and the next envelope $\widehat U_{t+1}^{\mathrm{dist}}(\alpha) = R_{t+1}^{\mathrm{dist}}$. The following theorem establishes that the resulting process is fully observable and retains the simultaneous-coverage guarantee.

\begin{theorem}[Observable confidence sequence for $\norm{x_{t+1} - x^\star}^2$]
\label{thm:observable-distance-cs}
Suppose Assumptions~\ref{assump:strong-convexity}--\ref{assump:stepsize-bound} hold, and let $R_0$ be an $\cF_{t_0-1}$-measurable quantity satisfying $R_0 \ge Z_{t_0}$ almost surely. Then, for every $\alpha \in (0,1)$, the recursively defined process $\{\widehat U_{t+1}^{\mathrm{dist}}(\alpha)\}_{t \ge t_0}$ is well defined and fully observable from $R_0$ and the SGD trajectory. Moreover,
\[
    \Prob\left(
        \forall t \ge t_0:\
        Z_{t+1}
        \le
        \widehat U_{t+1}^{\mathrm{dist}}(\alpha)
    \right)
    \ge
    1 - \alpha.
\]
\end{theorem}

A key feature of Theorem~\ref{thm:observable-distance-cs} is that it accommodates any predictable stepsize sequence satisfying the standing conditions, including stepsizes chosen adaptively from the observed trajectory. The confidence sequence is therefore not tied to a particular schedule or to an a priori convergence-rate calculation. To verify that the construction nevertheless recovers the expected behavior under a standard choice, we next specialize to $\eta_t = 1/(\mu t)$ and, for simplicity, assume that the stochastic gradients are uniformly bounded by $G$. This condition also provides a deterministic observable initialization: unbiasedness and conditional Jensen's inequality imply $\norm{\nabla f(x_{t_0})} \le G$, while strong convexity and first-order optimality yield $\mu\norm{x_{t_0} - x^\star} \le \norm{\nabla f(x_{t_0})}$. Consequently,
\begin{equation}
\label{eq:R_0}
    Z_{t_0}
    \le
    R_0
    \coloneqq
    \frac{G^2}{\mu^2}
    \qquad
    \text{almost surely}.
\end{equation}
The following proposition shows that, under this classical specialization, the observable envelope decays at the $1/t$ rate in the worst case, up to an iterated-logarithmic factor typical of anytime-valid analysis.

\begin{proposition}[Decay rate of the observable confidence sequence for $\norm{x_{t+1} - x^\star}^2$]
\label{thm:rate-distance-CS}
Fix $\alpha \in (0,1)$ and suppose Assumptions~\ref{assump:strong-convexity}--\ref{assump:stepsize-bound} hold. Let $t_0 = 3$ and $\eta_t = 1/(\mu t)$ for every $t \ge 3$, and suppose that $\norm{g_t} \le G$ almost surely for every $t \ge 3$. Initialize the recursive construction using $R_0$ from~\eqref{eq:R_0}. Define
\[
    B
    \coloneqq
    1
    +
    \frac{G^2}{\mu^2}
    +
    \frac{\sigma^2}{\mu^2},
    \qquad
    L_t(\alpha)
    \coloneqq
    1
    +
    \log\frac{1}{\alpha}
    +
    \log\log(et)
    +
    \log\log(eB).
\]
Then there exists a universal constant\footnote{In this proof, we track an explicit, though unoptimized, value of $C$ to illustrate that the constants can be made fully explicit. In subsequent rate proofs, $C$ denotes a universal constant whose value may change from line to line.} $C>0$ such that, almost surely,
\[
    \widehat U_{t+1}^{\mathrm{dist}}(\alpha)
    \le
    C B
    \frac{L_{t+1}(\alpha)}{t+1}
    \qquad
    \text{for every }t \ge 4.
\]
\end{proposition}

Proposition~\ref{thm:rate-distance-CS} reveals why the recursive construction preserves the strongly convex rate. The recursion does more than replace the latent distances by observable upper bounds: each newly computed envelope is fed back into the subsequent intrinsic-time increments. In this sense, the construction is \emph{self-localizing}, since a shrinking certified radius leads to a corresponding reduction in the variance scale governing future deviations. Under $\eta_t = 1/(\mu t)$, this feedback reduces the intrinsic-time growth from order $t^3$ under a fixed global radius to order $t^2$, up to slowly varying factors. Because the stitched boundary scales with the square root of the intrinsic time, it is therefore of order $t$, again up to slowly varying factors. After normalization by $A_t \asymp t^2$, this yields the $1/t$ decay established above. The following remark makes this improvement explicit by comparing the recursive construction with the simpler diameter-based substitution.

\begin{remark}[Decay rate under the diameter-based observable bound]
\label{remark:rate-diameter}
Suppose the conditions of Proposition~\ref{thm:rate-distance-CS} hold and that $\cX$ is compact with diameter $D$, so that $Z_t \le D^2$. Then, under $\eta_t = 1/(\mu t)$ with $t_0 = 3$ and $A_t = t(t-1)/2$,
\[
    V_t^{\mathrm{dist}}
    \le
    \sigma^2 D^2
    \sum_{s=3}^t A_s^2 \eta_s^2
    \lesssim
    t^3.
\]
Consequently, using the monotonicity of $\mathfrak H_\alpha$,
\[
    \frac{4\,\mathfrak H_\alpha\left(V_t^{\mathrm{dist}}\right)}{A_t}
    \lesssim
    t^{-1/2} \sqrt{L_t(\alpha)}.
\]
Thus, directly replacing the distances $Z_t$ by the diameter bound yields only a $t^{-1/2}$-type decay, up to stitched-logarithmic factors, whereas the recursive observable envelope preserves the near-$1/t$ rate. \hfill $\clubsuit$
\end{remark}

The preceding proposition was stated for the stepsize $\eta_t = 1/(\mu t)$. In the sequel, it is also useful to have the analogous rate statement for the common alternative schedule $\eta_t = 2/(\mu(t+1))$. We record this variant here.

\begin{remark}[Alternative polynomial stepsize]
\label{rem:alternative-polynomial-stepsize}
The same rate conclusion holds for the stepsize
\[
    \eta_t = \frac{2}{\mu(t+1)}.
\]
In this case, $1 - 2\mu\eta_t = (t-3)/(t+1)$, so the natural analysis start time is $t_0 = 4$. Following similar steps to those in the proof of Proposition~\ref{thm:rate-distance-CS}, there exists a universal constant $C < \infty$ such that, for every $t \ge 4$,
\[
    \widehat U^{\mathrm{dist}}_{t+1}(\alpha)
    \le
    C\,
    B\,
    \frac{L_{t+1}(\alpha)}{t+1}
    \qquad\text{almost surely}.
\]
The proof is similar to the proof of Proposition~\ref{thm:rate-distance-CS} and is omitted. \hfill $\clubsuit$
\end{remark}


\subsection{Suboptimality of the weighted-average iterate}
\label{subsec:weighted-average-suboptimality}

We next consider the suboptimality of the weighted-average iterate and construct an upper confidence sequence for $f(\bar x_t) - f(x^\star)$. Recalling the definitions in~\eqref{eq:weighted-suboptimality} and~\eqref{eq:weighted-average-iterate}, the analysis follows the same basic structure as for the last-iterate distance to the optimizer: the objective gap is bounded by an initialization term, an observable accumulation of squared stochastic gradients, and a partial sum of directional gradient noise terms that can again be controlled using Lemma~\ref{prop:dyadic-stitching}.

For this construction, we take $t_0\ge4$ and specialize to the stepsize $\eta_t = 2/(\mu(t+1))$ considered in Remark~\ref{rem:alternative-polynomial-stepsize}. This choice is naturally aligned with the linear weights defining $\bar x_t$: after the one-step suboptimality inequality is weighted by $t$, the resulting distance terms telescope across consecutive times. Define
\[
    E_t
    \coloneqq
    t \inner{x^\star - x_t}{\xi_t},
    \qquad
    \bar E_t
    \coloneqq
    \sum_{s = t_0}^t E_s,
    \qquad
    t \ge t_0.
\]
The resulting weighted decomposition is recorded in Lemma~\ref{lem:weighted-suboptimality-decomposition} and gives
\begin{equation}
\label{eq:subopt-normalized-decomposition}
    f(\bar x_t)-f(x^\star)
    \le
    \frac{1}{S_t}
    \left[
        \frac{1}{\mu}
        \sum_{s = t_0}^t
        \norm{g_s}^2
        +
        \frac{\mu t_0(t_0 - 1)}{4} Z_{t_0}
        +
        \bar E_t
    \right].
\end{equation}
This decomposition follows from the standard weighted telescoping argument for strongly convex SGD; see \citet[Section~3.2]{lacoste2012simpler}. We retain the directional gradient-noise term explicitly, rather than taking expectations, for the anytime-valid analysis.

The first term on the right-hand side of~\eqref{eq:subopt-normalized-decomposition} is observable from the SGD trajectory, the second is an initialization term, and the final term accumulates the gradient noise in the directions $x^\star - x_t$. Thus, as in the distance analysis, a time-uniform upper bound on $\bar E_t$ translates directly into an upper confidence sequence for the weighted-average suboptimality.

To control this final term, define the predictable intrinsic-time increment
\begin{equation}
\label{eq:def-sub-vt}
    v_t^{\mathrm{sub}}
    \coloneqq
    \sigma^2 t^2 Z_t,
    \qquad
    t \ge t_0.
\end{equation}
The sequence $\{E_t\}$ is adapted, while $\{v_t^{\mathrm{sub}}\}$ is predictable because $x_t$ is $\cF_{t-1}$-measurable. Applying the conditional sub-Gaussian bound in Assumption~\ref{assump:stochastic-gradients}(ii) with the predictable direction $t(x^\star-x_t)$ yields
\[
    \E\left[
        e^{\lambda E_t}
        \,\middle|\,
        \cF_{t-1}
    \right]
    \le
    \exp\left(
        \frac{\lambda^2 v_t^{\mathrm{sub}}}{2}
    \right)
    \qquad
    \text{almost surely}
\]
for every $\lambda \ge 0$ and every $t \ge t_0$. Since the coefficient in the conditional exponent is $1/2$, this is precisely the setting of Lemma~\ref{prop:dyadic-stitching} with $\gamma = -1$.

Using the notation in~\eqref{eq:v_eff} and the stitched-boundary function in~\eqref{eq:subgaussian-stitched-boundary}, define, for every $t \ge t_0$,
\[
    V_t^{\mathrm{sub}}
    \coloneqq
    \sum_{s = t_0}^t
    v_s^{\mathrm{sub}}.
\]
In this application, $V_t^{\mathrm{sub}}$ is the cumulative intrinsic time associated with the weighted directional-noise terms. Since the conditional exponential bound corresponds to $\gamma = -1$, Lemma~\ref{prop:dyadic-stitching} guarantees that
\[
    \bar E_t
    \le
    2\,\mathfrak H_\alpha\left(V_t^{\mathrm{sub}}\right)
\]
simultaneously for every $t \ge t_0$ with probability at least $1-\alpha$. Combining this bound with~\eqref{eq:subopt-normalized-decomposition} gives the following proposition.

\begin{proposition}[Confidence sequence for $f(\bar x_t) - f(x^\star)$]
\label{thm:adaptive-suboptimality-CS}
Suppose Assumptions~\ref{assump:strong-convexity}--\ref{assump:stepsize-bound} hold. Let $t_0 \ge 4$ and $\eta_t = 2/(\mu(t+1))$ for every $t \ge t_0$. Then, for every $\alpha \in (0,1)$, the process
\begin{equation}
\label{eq:Usub}
    U_t^{\mathrm{sub}}(\alpha)
    \coloneqq
    \frac{1}{S_t}
    \left[
        \frac{1}{\mu}
        \sum_{s=t_0}^t
        \norm{g_s}^2
        +
        \frac{\mu t_0 (t_0-1)}{4} Z_{t_0}
        +
        2\,\mathfrak H_\alpha\left(V_t^{\mathrm{sub}}\right)
    \right],
    \qquad
    t\ge t_0,
\end{equation}
is an upper confidence sequence for the weighted-average suboptimality, in the sense that
\[
    \Prob\left(
        \forall t\ge t_0:\
        f(\bar x_t)-f(x^\star)
        \le
        U_t^{\mathrm{sub}}(\alpha)
    \right)
    \ge
    1-\alpha.
\]
\end{proposition}

Proposition~\ref{thm:adaptive-suboptimality-CS} is not yet fully observable, because both the initialization term and the intrinsic time depend on the unknown distances to the optimizer. The key difference from the preceding subsection is that no second recursive construction is needed. Theorem~\ref{thm:observable-distance-cs} already provides observable upper bounds on the entire distance sequence, which can be substituted directly into the suboptimality bound. We allocate error probability $\alpha/2$ to the distance confidence sequence and $\alpha/2$ to the stitched bound for the weighted directional-noise process. On the intersection of the two corresponding events, every unknown distance is dominated by its observable envelope, and the resulting suboptimality bound holds simultaneously over time. A union bound then gives overall coverage at least $1-\alpha$. This coupling has only a mild cost: replacing $\alpha$ by $\alpha/2$ adds $\log 2$ to the confidence logarithms appearing inside the square-root boundaries.

Recall the observable distance proxies defined in~\eqref{eq:distance-radius-proxy}. To make their dependence on the confidence level explicit, for $\beta \in (0,1)$ write
\begin{equation}
\label{eq:sub-distance-proxy}
    R_{t_0}^{\mathrm{dist}}(\beta)
    \coloneqq
    R_0,
    \qquad
    R_s^{\mathrm{dist}}(\beta)
    \coloneqq
    \widehat U_s^{\mathrm{dist}}(\beta),
    \quad
    s \ge t_0+1.
\end{equation}
On the simultaneous-coverage event of Theorem~\ref{thm:observable-distance-cs} at level $\beta$, these proxies satisfy
\[
    Z_s
    \le
    R_s^{\mathrm{dist}}(\beta)
    \qquad
    \text{for every }s \ge t_0.
\]
For $t \ge t_0$, define the observable intrinsic time
\begin{equation}
\label{eq:observable-sub-intrinsic-time}
    \widehat V_t^{\mathrm{sub}}(\beta)
    \coloneqq
    \sum_{s = t_0}^t
    \sigma^2 s^2
    R_s^{\mathrm{dist}}(\beta).
\end{equation}
On the distance-coverage event at level $\beta$, the observable intrinsic time dominates its latent counterpart for every $t \ge t_0$. Since $\mathfrak H_\beta$ is nondecreasing,
\[
    2\,\mathfrak H_\beta\left(
        V_t^{\mathrm{sub}}
    \right)
    \le
    2\,\mathfrak H_\beta\left(
        \widehat V_t^{\mathrm{sub}}(\beta)
    \right)
    \qquad
    \text{for every }t \ge t_0.
\]
The same distance proxies also control the initialization term, since $Z_{t_0} \le R_{t_0}^{\mathrm{dist}}(\beta) = R_0$. Thus, the distance confidence sequence supplies both ingredients needed to make the suboptimality certificate observable. Given $\alpha \in (0,1)$, define the observable suboptimality envelope
\begin{equation}
\label{eq:observable-sub-CS}
    \widehat U_t^{\mathrm{sub}}(\alpha)
    \coloneqq
    \frac{1}{S_t}
    \left[
        \frac{1}{\mu}
        \sum_{s = t_0}^t
        \norm{g_s}^2
        +
        \frac{\mu t_0 (t_0 - 1)}{4} R_0
        +
        2\,\mathfrak H_{\alpha/2}\left(
            \widehat V_t^{\mathrm{sub}}(\alpha/2)
        \right)
    \right],
    \qquad
    t \ge t_0.
\end{equation}
The construction is observable: for $s \ge t_0+1$, the proxy $R_s^{\mathrm{dist}}(\beta) = \widehat U_s^{\mathrm{dist}}(\beta)$ is $\cF_{s-1}$-measurable and is therefore available before $g_s$ is observed. At the end of iteration $t$, all distance proxies in~\eqref{eq:sub-distance-proxy} with index at most $t$ and all gradients $g_s$ with $s \le t$ are available, so~\eqref{eq:observable-sub-intrinsic-time} and~\eqref{eq:observable-sub-CS} define an $\cF_t$-measurable quantity.

\begin{theorem}[Observable confidence sequence for $f(\bar x_t) - f(x^\star)$]
\label{thm:observable-suboptimality-CS}
Suppose Assumptions~\ref{assump:strong-convexity}--\ref{assump:stepsize-bound} hold. Let $t_0 \ge 4$ and $\eta_t = 2/(\mu(t+1))$ for every $t \ge t_0$, and let $R_0$ be an $\cF_{t_0-1}$-measurable quantity satisfying $R_0 \ge Z_{t_0}$ almost surely. Then, for every $\alpha \in (0,1)$, the process $\{\widehat U_t^{\mathrm{sub}}(\alpha)\}_{t \ge t_0}$ is fully observable and satisfies
\[
    \Prob\left(
        \forall t\ge t_0:\
        f(\bar x_t) - f(x^\star)
        \le
        \widehat U_t^{\mathrm{sub}}(\alpha)
    \right)
    \ge
    1-\alpha.
\]
\end{theorem}

We now turn to the decay of the observable suboptimality confidence sequence under the same bounded-gradient specialization used in Proposition~\ref{thm:rate-distance-CS}. In particular, the deterministic initialization $R_0 = G^2/\mu^2$ from~\eqref{eq:R_0} remains available, although any sharper observable upper bound on $Z_{t_0}$ may be used instead. Since the present construction uses the stepsize $\eta_t = 2/(\mu(t+1))$, the rate analysis combines Proposition~\ref{thm:rate-distance-CS} with its extension in Remark~\ref{rem:alternative-polynomial-stepsize}. The following proposition shows that the resulting observable suboptimality envelope recovers the optimal $1/t$ rate in the worst case, up to iterated-logarithmic factors.

\begin{proposition}[Decay rate of the observable confidence sequence for $f(\bar x_t) - f(x^\star)$]
\label{thm:rate-suboptimality-CS}
Fix $\alpha \in (0,1)$ and suppose Assumptions~\ref{assump:strong-convexity}--\ref{assump:stepsize-bound} hold. Let $R_0 = G^2/\mu^2$, $t_0 = 4$, and $\eta_t = 2/(\mu(t+1))$ for every $t \ge 4$, and suppose that $\norm{g_t} \le G$ almost surely for every $t \ge 4$. Define
\[
    B
    \coloneqq
    1
    +
    \frac{G^2}{\mu^2}
    +
    \frac{\sigma^2}{\mu^2},
    \qquad
    B_{\mathrm{sub}}
    \coloneqq
    1
    +
    \mu B
    =
    1
    +
    \mu
    +
    \frac{G^2}{\mu}
    +
    \frac{\sigma^2}{\mu},
\]
and, for $t \ge 4$, let
\[
    L_t^{\mathrm{sub}}(\alpha)
    \coloneqq
    1
    +
    \log\frac{1}{\alpha}
    +
    \log\log(et)
    +
    \log\log(eB)
    +
    \log\log(eB_{\mathrm{sub}}).
\]
Then there exists a universal constant $C > 0$ such that, almost surely,
\[
    \widehat U_t^{\mathrm{sub}}(\alpha)
    \le
    C B_{\mathrm{sub}}
    \frac{L_t^{\mathrm{sub}}(\alpha)}{t}
    \qquad
    \text{for every }t \ge 4.
\]
\end{proposition}

Proposition~\ref{thm:rate-suboptimality-CS} shows that passing from observable distance bounds to an observable suboptimality certificate does not degrade the optimal worst-case rate. Indeed, the near-$1/s$ distance envelopes make the suboptimality intrinsic time grow as $t^2$, up to slowly varying factors; the corresponding stitched boundary therefore grows as $t$, and division by $S_t\asymp t^2$ yields the near-$1/t$ decay. Directly bounding objective suboptimality also improves the dependence on the strong-convexity parameter by one power: the terms $G^2/\mu^2$ and $\sigma^2/\mu^2$ in the distance certificate become $G^2/\mu$ and $\sigma^2/\mu$. Consequently, when $\mu$ is small, the suboptimality certificate can become informative at a substantially earlier stage of the optimization trajectory.


\section{Refined Confidence Sequences under Bounded Gradients}
\label{sec:refined-confidence-sequences}

The preceding section constructed anytime-valid upper confidence sequences for the last-iterate squared distance $Z_{t+1}$ and the suboptimality of the weighted-average iterate, $f(\bar x_t) - f(x^\star)$, under a conditional sub-Gaussian assumption on the gradient noise. Under bounded stochastic gradients, conditional Hoeffding's lemma verifies this assumption with $\sigma^2 = G^2$, so those confidence sequences remain valid. With this choice, however, the intrinsic times entering the observable distance and suboptimality confidence sequences are driven by the worst-case proxy $G^2$ and therefore do not adapt to the realized magnitudes of the stochastic gradients. Because $G$ is a uniform bound, it can be substantially larger than the gradient magnitudes observed along a typical trajectory, making the stochastic contribution to the resulting certificates unnecessarily large.

The natural refinement is to let the stochastic boundaries respond to the gradients actually observed along the trajectory. Under the bounded-gradient assumption, an empirical Bernstein construction makes this possible by replacing the fixed proxy $G^2$ with quadratic processes involving the realized values of $\norm{g_t}^2$. The relevant directional terms, however, have ranges that evolve with the SGD trajectory, so a fixed-range empirical Bernstein bound is not sufficient for our purposes. We therefore develop a time-uniform empirical Bernstein inequality for a growing predictable range, which will provide the common concentration argument for both the last-iterate distance and weighted-average suboptimality. This development builds on the empirical Bernstein supermartingale construction of \citet{howard2021time} and the exponential inequality of \citet{fan2015exponential} underlying it. As emphasized in \cite{howard2021time}, this approach is technically closer to the literature on self-normalized bounds \cite{de2004self} than to classical empirical Bernstein arguments such as \citet{maurer2009empirical}. The latter combine a concentration bound for the sample mean involving the true variance with a separate concentration bound for the sample variance. By contrast, \citet{howard2021time} construct an exponential supermartingale that directly relates the centered process to an online empirical variance.

We impose the following assumption throughout this section.

\begin{assumption}[Bounded stochastic gradients]
\label{assump:bounded-stochastic-gradients}
There exists $G < \infty$ such that
\[
    \norm{g_t} \le G
    \qquad
    \text{almost surely for all } t \ge t_0.
\]
\end{assumption}

As shown in~\eqref{eq:R_0}, Assumption~\ref{assump:bounded-stochastic-gradients} always yields the observable initialization $R_0 = G^2/\mu^2$. This universal choice will be used below, although sharper observable bounds on $Z_{t_0}$ may be available in specific applications and can be substituted directly into the construction.

The concentration mechanism remains the same as in the preceding section: a one-step conditional exponential bound is lifted to a time-uniform boundary through a supermartingale argument and stitching. What changes is the compensating term. Rather than using a predictable variance proxy, the empirical Bernstein inequality below uses the realized square of the increment itself. This is the key mechanism that will allow the resulting confidence sequences to adapt to the stochastic gradients observed along the trajectory. 

The following lemma isolates this one-step ingredient in the precise form needed below. Starting from a pointwise exponential inequality, it yields after conditional centering a bound in which the quadratic compensation depends on the realized square $X^2$, while the admissible tuning parameter is restricted by the increment range $b$. These two features will drive the time-uniform construction that follows. To state the lemma, for $\theta \in [0,1)$ define
\begin{equation}
\label{eq:eb-psi}
    \psi_{\mathrm E}(\theta)
    \coloneqq
    - \log(1-\theta) - \theta.
\end{equation}
We are now ready to state the lemma.

\begin{lemma}[Conditional exponential moment bound]
\label{lem:eb-exponential}
For every $\theta \in [0,1)$ and every $x \ge -1$,
\begin{equation}
\label{eq:eb-pointwise}
    \exp\left(
        \theta x - \psi_{\mathrm E}(\theta) x^2
    \right)
    \le
    1+\theta x.
\end{equation}
Moreover,
\begin{equation}
\label{eq:eb-psi-upper}
    \psi_{\mathrm E}(\theta)
    \le
    \frac{\theta^2}{2(1-\theta)}.
\end{equation}
Consequently, let $\mathcal G$ be a sigma-field, let $X$ be an integrable real-valued random variable satisfying $|X| \le b$ almost surely for some deterministic $b > 0$, and let $\lambda \in [0,1/b)$. Then
\begin{equation}
\label{eq:eb-conditional}
    \E\left[
        \exp\left(
            \lambda \bigl(X - \E[X \mid \mathcal G]\bigr)
            -
            \frac{\lambda^2 X^2}{2(1-\lambda b)}
        \right)
        \,\middle|\,
        \mathcal G
    \right]
    \le 1
    \qquad
    \text{almost surely}.
\end{equation}
\end{lemma}

The pointwise inequality~\eqref{eq:eb-pointwise} was first established in the proof of \citet[Lemma~4.1]{fan2015exponential}. In \citet[Theorem~4]{howard2021time}, this inequality is applied conditionally to deviations from a predictable reference sequence, and the corresponding squared deviations are accumulated to construct a time-uniform empirical Bernstein confidence sequence. Equation~\eqref{eq:eb-conditional} records the one-step conditional form needed here, with the predictable reference specialized to zero and the quadratic coefficient simplified using~\eqref{eq:eb-psi-upper}. A self-contained proof is provided in the appendix.

We now lift the one-step inequality in Lemma~\ref{lem:eb-exponential} to a time-uniform bound. Two evolving scales enter the construction. First, the quadratic compensation in~\eqref{eq:eb-conditional} accumulates the realized squared increments. Second, the admissible tuning parameter depends on the range of the increment. In our applications this range is predictable and may grow with the trajectory, without admitting any finite uniform bound over the infinite horizon. A time-uniform boundary must therefore adapt simultaneously to the accumulated quadratic process and to the largest predictable range encountered so far. We achieve this by stitching over dyadic scales of both quantities.

For $v, b \ge 0$, define the effective quadratic-process and range scales
\begin{equation}
\label{eq:eb-effective-scales}
    v_{\mathrm{eff}}
    \coloneqq
    \max\{v, 1\},
    \qquad
    b_{\mathrm{eff}}
    \coloneqq
    \max\{b, 1\},
\end{equation}
and the corresponding dyadic indices
\begin{equation}
\label{eq:eb-scale-indices}
    k(v)
    \coloneqq
    \left\lceil \log_2 v_{\mathrm{eff}} \right\rceil,
    \qquad
    \ell(b)
    \coloneqq
    \left\lceil \log_2 b_{\mathrm{eff}} \right\rceil.
\end{equation}
Both indices take values in $\{0, 1, 2, \ldots\}$. As in Lemma~\ref{prop:dyadic-stitching}, the unit floor prevents negative dyadic indices and keeps the boundary well defined when either argument is zero. For $\alpha \in (0,1)$, define
\begin{equation}
\label{eq:eb-log-factor}
    \Gamma_\alpha(v, b)
    \coloneqq
    \log \frac{1}{\alpha}
    +
    \log \bigl((k(v)+1)(k(v)+2)\bigr)
    +
    \log \bigl((\ell(b)+1)(\ell(b)+2)\bigr),
\end{equation}
and let
\begin{equation}
\label{eq:eb-boundary}
    \mathfrak B_\alpha(v, b)
    \coloneqq
    2 \sqrt{
        v_{\mathrm{eff}} \Gamma_\alpha(v, b)
    }
    +
    2 b_{\mathrm{eff}} \Gamma_\alpha(v, b).
\end{equation}
The two logarithmic scale terms in $\Gamma_\alpha$ account for the simultaneous stitching over the quadratic process and the predictable range. The resulting boundary is the empirical Bernstein counterpart of the stitched sub-Gaussian boundary from the preceding section, with the additional range scale allowing the predictable increment bound to evolve without any uniform upper bound over time. To the best of our knowledge, this is the first time-uniform empirical Bernstein bound that allows the predictable range to vary with the history and grow without any finite uniform bound over the infinite horizon.

\begin{theorem}[Empirical Bernstein bound with a growing predictable range]
\label{thm:variable-range-eb}
Fix a deterministic integer $t_0 \ge 1$. Let $\{\cF_t\}_{t \ge t_0-1}$ be a filtration. Let $\{H_t\}_{t \ge t_0}$ be a real-valued integrable adapted process, and let $\{c_t\}_{t \ge t_0}$ be an almost surely finite,\footnote{No uniform upper bound on the process $\{c_t\}_{t \ge t_0}$ is assumed.} nonnegative predictable process such that
\[
    |H_t|
    \le
    c_t
    \qquad
    \text{almost surely for every } t \ge t_0.
\]
Define
\[
    \overline H_{t_0-1}
    \coloneqq
    0,
    \qquad
    \overline H_t
    \coloneqq
    \sum_{s = t_0}^t
    \left(
        H_s - \E[H_s \mid \cF_{s-1}]
    \right),
\]
and
\[
    W_{t_0-1}
    \coloneqq
    0,
    \qquad
    W_t
    \coloneqq
    \sum_{s = t_0}^t H_s^2,
    \qquad
    C_t
    \coloneqq
    \max_{t_0 \le s \le t} c_s,
    \qquad
    t \ge t_0.
\]
Then, for every $\alpha \in (0,1)$,
\begin{equation}
\label{eq:variable-range-eb-conclusion}
    \Prob\left(
        \forall t \ge t_0:\
        \overline H_t
        \le
        \mathfrak B_\alpha(W_t, C_t)
    \right)
    \ge
    1-\alpha.
\end{equation}
Moreover, the map $(v, b) \mapsto \mathfrak B_\alpha(v, b)$ is nondecreasing in each argument.
\end{theorem}

Theorem~\ref{thm:variable-range-eb} will now be used to refine the distance and weighted-suboptimality confidence sequences developed in the preceding section. We treat the two cases separately, beginning with the last-iterate distance to the optimizer.


\subsection{Last-iterate squared distance to the optimizer}
\label{subsec:empirical Bernstein-distance}

We begin with the last-iterate distance to the optimizer. Recall from Lemma~\ref{lem:distance-decomposition} that the distance analysis reduces to controlling the partial sum of the directional gradient noise terms $X_t = 2 A_t \eta_t \inner{x^\star - x_t}{\xi_t}$. To bring this term within the scope of Theorem~\ref{thm:variable-range-eb}, define the corresponding uncentered directional stochastic-gradient term
\begin{equation}
\label{eq:eb-distance-H}
    H_t^{\mathrm{dist}}
    \coloneqq
    2 A_t \eta_t
    \inner{x^\star - x_t}{g_t},
    \qquad
    t \ge t_0.
\end{equation}
By conditional unbiasedness, $H_t^{\mathrm{dist}} - \E\left[ H_t^{\mathrm{dist}} \, \middle| \, \cF_{t-1}\right] = X_t$. Thus, the centered process appearing in Theorem~\ref{thm:variable-range-eb} is exactly the stochastic term already isolated by the distance decomposition. Under Assumption~\ref{assump:bounded-stochastic-gradients},
\[
    \left|H_t^{\mathrm{dist}}\right|
    \le
    2 G A_t \eta_t \sqrt{Z_t}.
\]
For the canonical stepsize $\eta_t \asymp 1/t$, we have $A_t \asymp t^2$, while the distance confidence sequence of the preceding section gives $Z_t \lesssim L_t/t$. Hence the predictable range can grow as $\sqrt{t L_t}$, and in particular need not remain bounded over the infinite horizon. This is precisely the setting covered by Theorem~\ref{thm:variable-range-eb}. Accordingly, define the latent quadratic and range processes
\begin{align}
\label{eq:eb-latent-variance}
    V_t^{\mathrm{dist,EB}}
    &\coloneqq
    \sum_{s = t_0}^t
    \left(H_s^{\mathrm{dist}}\right)^2
    =
    4 \sum_{s = t_0}^t
    A_s^2 \eta_s^2
    \inner{g_s}{x^\star - x_s}^2,
    \\
\label{eq:eb-latent-range}
    B_t^{\mathrm{dist,EB}}
    &\coloneqq
    2 G
    \max_{t_0 \le s \le t}
    A_s \eta_s \sqrt{Z_s}.
\end{align}
These are the distance-specific counterparts of the quadratic process $W_t$ and the running predictable range $C_t$ in Theorem~\ref{thm:variable-range-eb}. The generic theorem also requires $H_t^{\mathrm{dist}}$ to be integrable. Since $Z_t \le G^2/\mu^2$ almost surely, it is enough to assume
\begin{equation}
\label{eq:eb-integrability-condition}
    \E\left[
        A_t \eta_t
    \right]
    <
    \infty
    \qquad
    \text{for every deterministic } t \ge t_0.
\end{equation}
This mild technical condition is automatic for the deterministic stepsize schedules considered below.

Combining Lemma~\ref{lem:distance-decomposition} with Theorem~\ref{thm:variable-range-eb} gives the latent empirical Bernstein confidence sequence
\begin{equation}
\label{eq:adaptive-distance-eb}
    U_{t+1}^{\mathrm{dist,EB}}(\alpha)
    \coloneqq
    \frac{1}{A_t}
    \left[
        Z_{t_0}
        +
        \sum_{s = t_0}^t
        A_s \eta_s^2 \norm{g_s}^2
        +
        \mathfrak B_\alpha\left(
            V_t^{\mathrm{dist,EB}},
            B_t^{\mathrm{dist,EB}}
        \right)
    \right],
    \qquad
    t \ge t_0,
\end{equation}
which satisfies
\[
    \Prob\left(
        \forall t \ge t_0:\
        Z_{t+1}
        \le
        U_{t+1}^{\mathrm{dist,EB}}(\alpha)
    \right)
    \ge
    1-\alpha.
\]
This is the empirical Bernstein counterpart of Proposition~\ref{thm:adaptive-distance-CS} from the preceding section. The argument follows the same overall logic, with Theorem~\ref{thm:variable-range-eb} providing the corresponding concentration step. We defer the formal proof of this confidence-sequence guarantee to Proposition~\ref{thm:adaptive-distance-eb} in the appendix.

\begin{remark}[Comparison with the sub-Gaussian construction]
\label{remark:comparison:bounded:subgaussian}
This empirical Bernstein refinement admits a direct comparison with the sub-Gaussian confidence sequence from Proposition~\ref{thm:adaptive-distance-CS}. The leading square-root term in $\mathfrak B_\alpha \left(V_t^{\mathrm{dist,EB}}, B_t^{\mathrm{dist,EB}}\right)$ is the analogue of $4\,\mathfrak H_\alpha\left(V_t^{\mathrm{dist}}\right)$. Under the bounded-gradient choice $\sigma^2 = G^2$, the latter uses the fixed worst-case scale $G^2$, whereas the former is driven by the realized directional quantities $\inner{g_s}{x^\star - x_s}^2 \le \norm{g_s}^2 Z_s$. In the worst case, when $\norm{g_s}=G$ and the directional bound is attained, $V_t^{\mathrm{dist,EB}} = 4 V_t^{\mathrm{dist}}$, so the leading factor $2$ in $\mathfrak B_\alpha \left(V_t^{\mathrm{dist,EB}}, B_t^{\mathrm{dist,EB}}\right)$ exactly recovers the outer factor $4$ in $4\,\mathfrak H_\alpha\left(V_t^{\mathrm{dist}}\right)$, up to the additional iterated-logarithmic factor required to stitch over the predictable range. Away from this worst case, the leading term in $\mathfrak B_\alpha \left(V_t^{\mathrm{dist,EB}}, B_t^{\mathrm{dist,EB}}\right)$ adapts to the realized stochastic-gradient magnitudes and directions, as desired. Beyond this additional logarithmic factor, the only new additive contribution is the linear range term in $\mathfrak B_\alpha\left(V_t^{\mathrm{dist,EB}}, B_t^{\mathrm{dist,EB}}\right)$. For the observable construction below, Proposition~\ref{thm:observable-distance-eb-rate} shows that the corresponding range contribution is lower order, scaling as $(L_t/t)^{3/2}$ rather than $L_t/t$. \hfill $\clubsuit$
\end{remark}

The bound is not yet observable, however, because $Z_{t_0}$, the directions $x^\star - x_s$, and the distances entering $B_t^{\mathrm{dist,EB}}$ all depend on the unknown optimizer. We remove this dependence using the same recursive domination principle as in the preceding section, now applied simultaneously to the two arguments of $\mathfrak B_\alpha$. We formalize this construction. Fix $\alpha \in (0,1)$, and let $R_0$ be any $\cF_{t_0 - 1}$-measurable quantity satisfying $R_0 \ge Z_{t_0}$ almost surely. Define recursively the observable distance envelopes
\begin{equation}
\label{eq:eb-radius-proxy}
    R_{t_0}^{\mathrm{EB}}
    \coloneqq
    R_0,
    \qquad
    R_s^{\mathrm{EB}}
    \coloneqq
    \widehat U_s^{\mathrm{dist,EB}}(\alpha),
    \quad
    s \ge t_0 + 1.
\end{equation}
For $t \ge t_0$, set
\begin{align}
\label{eq:eb-observable-variance}
    \widehat V_t^{\mathrm{dist,EB}}
    &\coloneqq
    4 \sum_{s = t_0}^t
    A_s^2 \eta_s^2
    \norm{g_s}^2
    R_s^{\mathrm{EB}},
    \\
\label{eq:eb-observable-range}
    \widehat B_t^{\mathrm{dist,EB}}
    &\coloneqq
    2 G
    \max_{t_0 \le s \le t}
    A_s \eta_s
    \sqrt{R_s^{\mathrm{EB}}},
\end{align}
and define
\begin{equation}
\label{eq:eb-observable-distance}
    \widehat U_{t+1}^{\mathrm{dist,EB}}(\alpha)
    \coloneqq
    \frac{1}{A_t}
    \left[
        R_0
        +
        \sum_{s = t_0}^t
        A_s \eta_s^2 \norm{g_s}^2
        +
        \mathfrak B_\alpha\left(
            \widehat V_t^{\mathrm{dist,EB}},
            \widehat B_t^{\mathrm{dist,EB}}
        \right)
    \right].
\end{equation}
The recursion is sequentially well defined: when the envelope at time $t+1$ is computed, all $R_s^{\mathrm{EB}}$ with $s \le t$ are already available. As in the sub-Gaussian construction, simultaneous domination of the unknown distances $Z_s$ by these observable distance envelopes yields observable upper bounds on both the latent quadratic process and the predictable range; monotonicity of $\mathfrak B_\alpha$ then propagates the domination through the recursion.

\begin{theorem}[Refined observable confidence sequence for $\norm{x_{t+1} - x^\star}^2$]
\label{thm:observable-distance-eb}
Suppose Assumptions~\ref{assump:strong-convexity}--\ref{assump:stepsize-bound} and~\ref{assump:bounded-stochastic-gradients} hold, and suppose~\eqref{eq:eb-integrability-condition} is satisfied. Let $R_0$ be an $\cF_{t_0 - 1}$-measurable quantity satisfying $R_0 \ge Z_{t_0}$ almost surely. Then, for every $\alpha \in (0,1)$, the recursively defined process $\{\widehat U_{t+1}^{\mathrm{dist,EB}}(\alpha)\}_{t \ge t_0}$ is well defined and fully observable from $R_0$ and the SGD trajectory. Moreover,
\[
    \Prob\left(
        \forall t \ge t_0:\
        Z_{t+1}
        \le
        \widehat U_{t+1}^{\mathrm{dist,EB}}(\alpha)
    \right)
    \ge
    1 - \alpha.
\]
\end{theorem}

The observable construction preserves the trajectory adaptation highlighted in Remark~\ref{remark:comparison:bounded:subgaussian}: the quadratic process~\eqref{eq:eb-observable-variance} depends on the realized values $\norm{g_s}^2$, while the recursive observable distance envelopes $R_s^{\mathrm{EB}}$ remove the remaining dependence on the unknown optimizer. The global bound $G$ remains only in the predictable-range process~\eqref{eq:eb-observable-range} and, when used, in the canonical initialization $R_0 = G^2/\mu^2$. It remains to understand the cost of accommodating the growing predictable range. The empirical Bernstein boundary introduces the additional linear range term in~\eqref{eq:eb-boundary}, which has no counterpart in the sub-Gaussian construction. The next proposition shows that this term is of lower order and that the full observable confidence sequence retains the near-$1/t$ decay established in the preceding section. Define
\begin{equation}
\label{eq:eb-rate-scale}
    B_{\mathrm{EB}}
    \coloneqq
    1
    +
    \frac{G^2}{\mu^2},
\end{equation}
and, for $t \ge t_0$,
\begin{equation}
\label{eq:eb-rate-log}
    L_t^{\mathrm{EB}}(\alpha)
    \coloneqq
    1
    +
    \log \frac{1}{\alpha}
    +
    \log \log(et)
    +
    \log \log(e B_{\mathrm{EB}}).
\end{equation}
To isolate the contribution of the linear range term in~\eqref{eq:eb-boundary}, define
\begin{equation}
\label{eq:eb-range-contribution}
    \mathcal R_t^{\mathrm{EB}}(\alpha)
    \coloneqq
    \frac{
        2 \max\left\{
            \widehat B_t^{\mathrm{dist,EB}},
            1
        \right\}
    }{A_t}
    \Gamma_\alpha\left(
        \widehat V_t^{\mathrm{dist,EB}},
        \widehat B_t^{\mathrm{dist,EB}}
    \right).
\end{equation}

With these quantities in place, we are ready to state the proposition.

\begin{proposition}[Decay rate of the refined observable confidence sequence for $\norm{x_{t+1} - x^\star}^2$]
\label{thm:observable-distance-eb-rate}
Fix $\alpha \in (0,1)$ and suppose the conditions of Theorem~\ref{thm:observable-distance-eb} hold. Let $R_0 = G^2/\mu^2$, $\eta_t = 1/(\mu t)$, and $t_0 = 3$. Then there exist universal constants $C, C_{\mathrm r} < \infty$ such that, almost surely,
\begin{equation}
\label{eq:observable-distance-eb-rate}
    \widehat U_{t+1}^{\mathrm{dist,EB}}(\alpha)
    \le
    C B_{\mathrm{EB}}
    \left[
        \frac{L_{t+1}^{\mathrm{EB}}(\alpha)}{t+1}
        +
        \left(
            \frac{L_{t+1}^{\mathrm{EB}}(\alpha)}{t+1}
        \right)^2
    \right],
    \qquad
    t \ge 3,
\end{equation}
and, whenever $t \ge 4$ and $L_t^{\mathrm{EB}}(\alpha) \le t$,
\begin{equation}
\label{eq:eb-range-lower-order}
    \mathcal R_t^{\mathrm{EB}}(\alpha)
    \le
    C_{\mathrm r} B_{\mathrm{EB}}
    \left(
        \frac{L_t^{\mathrm{EB}}(\alpha)}{t}
    \right)^{3/2}.
\end{equation}
\end{proposition}

Proposition~\ref{thm:observable-distance-eb-rate} yields two useful conclusions. First, once $L_{t+1}^{\mathrm{EB}}(\alpha) \le t+1$, the observable confidence sequence satisfies
\[
    \widehat U_{t+1}^{\mathrm{dist,EB}}(\alpha)
    \le
    2 C B_{\mathrm{EB}}
    \frac{L_{t+1}^{\mathrm{EB}}(\alpha)}{t+1},
\]
thereby recovering the same near-$1/t$ worst-case decay rate as the sub-Gaussian construction. For fixed $\alpha$ and fixed problem parameters, this condition holds for all sufficiently large $t$. Second,~\eqref{eq:eb-range-lower-order} shows that the additional contribution induced by the growing predictable range is of lower order than the leading $L_t^{\mathrm{EB}}(\alpha)/t$ term.

\begin{remark}[Alternative polynomial stepsize]
\label{rem:eb-alternative-polynomial-stepsize}
The same conclusion holds for the stepsize
\[
    \eta_t
    =
    \frac{2}{\mu(t+1)}.
\]
In this case, $1 - 2 \mu \eta_t = (t-3)/(t+1)$, so the natural analysis start time is $t_0 = 4$. Following similar steps to those in the proof of Proposition~\ref{thm:observable-distance-eb-rate}, there exists a universal constant $C < \infty$ such that, almost surely, for every $t \ge 4$,
\[
    \widehat U_{t+1}^{\mathrm{dist,EB}}(\alpha)
    \le
    C B_{\mathrm{EB}}
    \left[
        \frac{L_{t+1}^{\mathrm{EB}}(\alpha)}{t+1}
        +
        \left(
            \frac{L_{t+1}^{\mathrm{EB}}(\alpha)}{t+1}
        \right)^2
    \right].
\]
The analogue of \eqref{eq:eb-range-lower-order} holds as well, with a different universal constant. The proof is similar and is omitted. \hfill $\clubsuit$
\end{remark}


\subsection{Suboptimality of the weighted-average iterate}
\label{subsec:empirical Bernstein-suboptimality}

We now turn to the weighted-average suboptimality. Recall from Lemma~\ref{lem:weighted-suboptimality-decomposition} that the suboptimality analysis reduces to controlling the partial sum of the directional gradient noise terms $E_t = t\,\inner{x^\star - x_t}{\xi_t}$. To bring this term within the scope of Theorem~\ref{thm:variable-range-eb}, define the corresponding uncentered directional stochastic-gradient term
\begin{equation}
\label{eq:eb-sub-H}
    H_t^{\mathrm{sub}}
    \coloneqq
    t\,\inner{x^\star - x_t}{g_t},
    \qquad
    t \ge t_0.
\end{equation}
By conditional unbiasedness, $H_t^{\mathrm{sub}} - \E\left[ H_t^{\mathrm{sub}} \, \middle| \, \cF_{t-1}\right] = E_t$. Thus, the centered process appearing in Theorem~\ref{thm:variable-range-eb} is exactly the stochastic term already isolated by the suboptimality decomposition. Under Assumption~\ref{assump:bounded-stochastic-gradients},
\[
    \left|H_t^{\mathrm{sub}}\right|
    \le
    G t \sqrt{Z_t}.
\]
As in the distance construction, the bound $Z_t \lesssim L_t/t$ implies that the predictable range can grow as $\sqrt{t L_t}$ and need not remain bounded over the infinite horizon. Accordingly, define the latent quadratic and range processes
\begin{align}
\label{eq:eb-sub-latent-variance}
    V_t^{\mathrm{sub,EB}}
    &\coloneqq
    \sum_{s = t_0}^t
    \left(H_s^{\mathrm{sub}}\right)^2
    =
    \sum_{s = t_0}^t
    s^2
    \inner{g_s}{x^\star - x_s}^2,
    \\
\label{eq:eb-sub-latent-range}
    B_t^{\mathrm{sub,EB}}
    &\coloneqq
    G
    \max_{t_0 \le s \le t}
    s \sqrt{Z_s}.
\end{align}
These are the corresponding quadratic and predictable-range processes for the suboptimality analysis. Moreover, since $Z_t \le G^2/\mu^2$ almost surely, $H_t^{\mathrm{sub}}$ is integrable for every deterministic $t \ge t_0$.

Combining Lemma~\ref{lem:weighted-suboptimality-decomposition} with Theorem~\ref{thm:variable-range-eb} gives the latent empirical Bernstein confidence sequence
\begin{equation}
\label{eq:adaptive-suboptimality-eb}
    U_t^{\mathrm{sub,EB}}(\alpha)
    \coloneqq
    \frac{1}{S_t}
    \left[
        \frac{1}{\mu}
        \sum_{s = t_0}^t
        \norm{g_s}^2
        +
        \frac{\mu t_0 (t_0 - 1)}{4}
        Z_{t_0}
        +
        \mathfrak B_\alpha\left(
            V_t^{\mathrm{sub,EB}},
            B_t^{\mathrm{sub,EB}}
        \right)
    \right],
    \qquad
    t \ge t_0,
\end{equation}
which satisfies
\[
    \Prob\left(
        \forall t \ge t_0:\
        f(\bar x_t) - f(x^\star)
        \le
        U_t^{\mathrm{sub,EB}}(\alpha)
    \right)
    \ge
    1-\alpha.
\]
This is the empirical Bernstein counterpart of Proposition~\ref{thm:adaptive-suboptimality-CS} from the preceding section. The argument follows the same overall logic, with Theorem~\ref{thm:variable-range-eb} providing the corresponding concentration step. We defer the formal proof of this confidence-sequence guarantee to Proposition~\ref{thm:adaptive-suboptimality-eb} in the appendix.

As in the distance construction, the latent bound is not observable because $Z_{t_0}$, the directions $x^\star - x_s$, and the distances entering $B_t^{\mathrm{sub,EB}}$ depend on the unknown optimizer. Here this dependence can be removed directly using the observable distance confidence sequence constructed above. Let $R_0$ be any $\cF_{t_0 - 1}$-measurable quantity satisfying $R_0 \ge Z_{t_0}$ almost surely. For $\beta \in (0,1)$, define the observable distance envelopes
\begin{equation}
\label{eq:eb-sub-radius-proxy}
    R_{t_0}^{\mathrm{dist,EB}}(\beta)
    \coloneqq
    R_0,
    \qquad
    R_s^{\mathrm{dist,EB}}(\beta)
    \coloneqq
    \widehat U_s^{\mathrm{dist,EB}}(\beta),
    \quad
    s \ge t_0 + 1.
\end{equation}
For $t \ge t_0$, set
\begin{align}
\label{eq:eb-sub-observable-variance}
    \widehat V_t^{\mathrm{sub,EB}}(\beta)
    &\coloneqq
    \sum_{s = t_0}^t
    s^2 \norm{g_s}^2
    R_s^{\mathrm{dist,EB}}(\beta),
    \\
\label{eq:eb-sub-observable-range}
    \widehat B_t^{\mathrm{sub,EB}}(\beta)
    &\coloneqq
    G
    \max_{t_0 \le s \le t}
    s \sqrt{R_s^{\mathrm{dist,EB}}(\beta)}.
\end{align}
Given $\alpha \in (0,1)$, define
\begin{equation}
\label{eq:eb-sub-observable-cs}
    \widehat U_t^{\mathrm{sub,EB}}(\alpha)
    \coloneqq
    \frac{1}{S_t}
    \left[
        \frac{1}{\mu}
        \sum_{s = t_0}^t
        \norm{g_s}^2
        +
        \frac{\mu t_0 (t_0 - 1)}{4}
        R_0
        +
        \mathfrak B_{\alpha/2}\left(
            \widehat V_t^{\mathrm{sub,EB}}(\alpha/2),
            \widehat B_t^{\mathrm{sub,EB}}(\alpha/2)
        \right)
    \right],
    \qquad
    t \ge t_0.
\end{equation}
As in the distance construction, the observable distance envelopes dominate the unknown distances simultaneously and therefore yield observable upper bounds on both latent arguments of $\mathfrak B_{\alpha/2}$. The use of confidence level $\alpha/2$ accounts for the simultaneous validity of the distance and suboptimality confidence sequences.

\begin{theorem}[Refined observable confidence sequence for $f(\bar x_t) - f(x^\star)$]
\label{thm:observable-suboptimality-eb}
Suppose Assumptions~\ref{assump:strong-convexity}--\ref{assump:stepsize-bound} and~\ref{assump:bounded-stochastic-gradients} hold. Let $t_0 \ge 4$ and $\eta_t = 2/(\mu(t+1))$ for every $t \ge t_0$, and let $R_0$ be an $\cF_{t_0 - 1}$-measurable quantity satisfying $R_0 \ge Z_{t_0}$ almost surely. Then, for every $\alpha \in (0,1)$, the process $\{\widehat U_t^{\mathrm{sub,EB}}(\alpha)\}_{t \ge t_0}$ is well defined and fully observable from $R_0$ and the SGD trajectory. Moreover,
\[
    \Prob\left(
        \forall t \ge t_0:\
        f(\bar x_t) - f(x^\star)
        \le
        \widehat U_t^{\mathrm{sub,EB}}(\alpha)
    \right)
    \ge
    1-\alpha.
\]
\end{theorem}

The observable construction preserves the same trajectory adaptation: the quadratic process~\eqref{eq:eb-sub-observable-variance} depends on the realized values $\norm{g_s}^2$, while the observable distance envelopes remove the remaining dependence on the unknown optimizer. As in the distance analysis, it remains to quantify the cost of the growing predictable range. Recall $B_{\mathrm{EB}}$ from~\eqref{eq:eb-rate-scale}, and define
\begin{equation}
\label{eq:eb-sub-rate-scale}
    B_{\mathrm{sub}}^{\mathrm{EB}}
    \coloneqq
    1
    +
    \mu B_{\mathrm{EB}}
    =
    1
    +
    \mu
    +
    \frac{G^2}{\mu},
\end{equation}
and, for $t \ge t_0$,
\begin{equation}
\label{eq:eb-sub-rate-log}
    L_t^{\mathrm{sub,EB}}(\alpha)
    \coloneqq
    1
    +
    \log \frac{2}{\alpha}
    +
    \log \log(et)
    +
    \log \log(e B_{\mathrm{EB}})
    +
    \log \log\left(e B_{\mathrm{sub}}^{\mathrm{EB}}\right).
\end{equation}
To isolate the contribution of the linear range term in~\eqref{eq:eb-boundary}, define
\begin{equation}
\label{eq:eb-sub-range-contribution}
    \mathcal R_t^{\mathrm{sub,EB}}(\alpha)
    \coloneqq
    \frac{
        2 \max\left\{
            \widehat B_t^{\mathrm{sub,EB}}(\alpha/2),
            1
        \right\}
    }{S_t}
    \Gamma_{\alpha/2}\left(
        \widehat V_t^{\mathrm{sub,EB}}(\alpha/2),
        \widehat B_t^{\mathrm{sub,EB}}(\alpha/2)
    \right).
\end{equation}

With these quantities in place, we are ready to state the proposition.

\begin{proposition}[Decay rate of the refined observable confidence sequence for $f(\bar x_t) - f(x^\star)$]
\label{thm:observable-suboptimality-eb-rate}
Fix $\alpha \in (0,1)$ and suppose Assumptions~\ref{assump:strong-convexity}--\ref{assump:stepsize-bound} and~\ref{assump:bounded-stochastic-gradients} hold. Let $R_0 = G^2/\mu^2$, $t_0 = 4$, and $\eta_t = 2/(\mu(t+1))$ for every $t \ge 4$. Then there exist universal constants $C, C_{\mathrm r} < \infty$ such that, almost surely,
\begin{equation}
\label{eq:observable-suboptimality-eb-rate}
    \widehat U_t^{\mathrm{sub,EB}}(\alpha)
    \le
    C B_{\mathrm{sub}}^{\mathrm{EB}}
    \left[
        \frac{L_t^{\mathrm{sub,EB}}(\alpha)}{t}
        +
        \left(
            \frac{L_t^{\mathrm{sub,EB}}(\alpha)}{t}
        \right)^2
    \right],
    \qquad
    t \ge 4,
\end{equation}
and, whenever $L_t^{\mathrm{sub,EB}}(\alpha) \le t$,
\begin{equation}
\label{eq:eb-sub-range-lower-order}
    \mathcal R_t^{\mathrm{sub,EB}}(\alpha)
    \le
    C_{\mathrm r} B_{\mathrm{sub}}^{\mathrm{EB}}
    \left(
        \frac{L_t^{\mathrm{sub,EB}}(\alpha)}{t}
    \right)^{3/2}.
\end{equation}
\end{proposition}

Proposition~\ref{thm:observable-suboptimality-eb-rate} yields the same two conclusions as in the distance analysis. First, once $L_t^{\mathrm{sub,EB}}(\alpha) \le t$, the observable confidence sequence satisfies
\[
    \widehat U_t^{\mathrm{sub,EB}}(\alpha)
    \le
    2 C B_{\mathrm{sub}}^{\mathrm{EB}}
    \frac{L_t^{\mathrm{sub,EB}}(\alpha)}{t},
\]
thereby recovering the same near-$1/t$ worst-case decay rate as the sub-Gaussian construction. For fixed $\alpha$ and fixed problem parameters, this condition holds for all sufficiently large $t$. Second,~\eqref{eq:eb-sub-range-lower-order} shows that the additional contribution induced by the growing predictable range is of lower order than the leading $L_t^{\mathrm{sub,EB}}(\alpha)/t$ term.


\section{Minibatch Extensions}
\label{sec:minibatching}

The confidence sequences developed in the preceding two sections admit natural extensions to minibatch SGD. Fix an integer $b \ge 1$. At every iteration $t \ge 1$, conditionally on $\cF_{t-1}$, suppose that we observe independent and identically distributed stochastic gradients $g_t^{(1)}, \ldots, g_t^{(b)}$ satisfying
\begin{equation}
\label{eq:minibatch-gradient-assumptions}
    \E\left[
        g_t^{(i)}
        \,\middle|\,
        \cF_{t-1}
    \right]
    =
    \nabla f(x_t),
    \qquad
    i = 1, \ldots, b.
\end{equation}
Let $\cF_t \coloneqq \sigma\left(\cF_{t-1}, g_t^{(1)}, \ldots, g_t^{(b)}\right)$ denote the filtration generated by the history through iteration $t$. Define the minibatch gradient and its noise by
\begin{equation}
\label{eq:minibatch-gradient}
    \bar g_t
    \coloneqq
    \frac{1}{b}
    \sum_{i=1}^b
    g_t^{(i)},
    \qquad
    \bar\xi_t
    \coloneqq
    \bar g_t - \nabla f(x_t),
\end{equation}
and let the SGD update use $\bar g_t$ in place of $g_t$.

For the sub-Gaussian confidence sequences of Section~\ref{sec:subgaussian-confidence-sequences}, the effect of minibatching is immediate. Suppose, in addition, that the individual noises $\xi_t^{(i)} \coloneqq g_t^{(i)} - \nabla f(x_t)$ satisfy the conditional sub-Gaussian condition in Assumption~\ref{assump:stochastic-gradients} with the same variance proxy $\sigma^2$. Conditional independence then implies that their average $\bar\xi_t$ satisfies the same conditional sub-Gaussian bound with variance proxy $\sigma^2/b$. Consequently, the confidence sequences of Section~\ref{sec:subgaussian-confidence-sequences} extend directly to minibatch SGD by replacing $\sigma^2$ with $\sigma^2/b$ in the corresponding variance processes and replacing $g_t$ with $\bar g_t$ in the SGD recursions.

The empirical Bernstein construction of Section~\ref{sec:refined-confidence-sequences} allows the information contained in the minibatch to be used more finely. The sub-Gaussian construction summarizes the stochastic fluctuations through the fixed conditional variance proxy $\sigma^2/b$. By contrast, the empirical Bernstein argument can be applied to the $b$ individual observations before they are averaged. Conditional independence allows the corresponding one-step exponential bounds to be combined while retaining the sum of the individual realized quadratic contributions. Thus, the leading square-root term enjoys the usual $1/\sqrt b$ minibatch reduction while continuing to adapt to the realized directional second moments within each minibatch, whereas the linear range term carries an explicit factor $1/b$.

This observation yields a general minibatch extension of Theorem~\ref{thm:variable-range-eb}, in which the centered process is formed from the minibatch averages while the realized quadratic process retains the individual observations within each minibatch. The resulting form will apply directly to both the distance and suboptimality confidence sequences.

\begin{proposition}[Minibatch empirical Bernstein bound with a growing predictable range]
\label{cor:minibatch-variable-range-eb}
Fix deterministic integers $t_0 \ge 1$ and $b \ge 1$, and let $\{\cF_t\}_{t \ge t_0-1}$ be a filtration. For every $t \ge t_0$, let $H_t^{(1)}, \ldots, H_t^{(b)}$ be integrable, $\cF_t$-measurable random variables that are conditionally independent given $\cF_{t-1}$ and have a common conditional mean $m_t \coloneqq \E\left[H_t^{(i)} \mid \cF_{t-1}\right]$, for $i = 1, \ldots, b$. Let $\{c_t\}_{t \ge t_0}$ be an almost surely finite,\footnote{As in Theorem~\ref{thm:variable-range-eb}, no uniform upper bound on the process $\{c_t\}_{t \ge t_0}$ is assumed.} nonnegative predictable process such that $\left|H_t^{(i)}\right| \le c_t$ almost surely for every $t \ge t_0$ and $i = 1, \ldots, b$. Define
\[
    \overline H_{t_0-1}^{\mathrm{MB}}
    \coloneqq
    0,
    \qquad
    \overline H_t^{\mathrm{MB}}
    \coloneqq
    \sum_{s=t_0}^t
    \left(
        \frac{1}{b}
        \sum_{i=1}^b
        H_s^{(i)}
        -
        m_s
    \right),
\]
and
\[
    W_{t_0-1}^{\mathrm{MB}}
    \coloneqq
    0,
    \qquad
    W_t^{\mathrm{MB}}
    \coloneqq
    \sum_{s=t_0}^t
    \sum_{i=1}^b
    \left(H_s^{(i)}\right)^2,
    \qquad
    C_t
    \coloneqq
    \max_{t_0 \le s \le t}
    c_s.
\]
Then, for every $\alpha \in (0,1)$,
\begin{equation}
\label{eq:minibatch-eb-bound}
    \Prob\left(
        \forall t \ge t_0:\
        \overline H_t^{\mathrm{MB}}
        \le
        \frac{1}{b}
        \mathfrak B_\alpha\left(
            W_t^{\mathrm{MB}},
            C_t
        \right)
    \right)
    \ge
    1-\alpha.
\end{equation}
\end{proposition}

For $b=1$, Proposition~\ref{cor:minibatch-variable-range-eb} reduces exactly to Theorem~\ref{thm:variable-range-eb}. We next specialize Proposition~\ref{cor:minibatch-variable-range-eb} to the SGD setting. For this purpose, we impose the minibatch analogue of Assumption~\ref{assump:bounded-stochastic-gradients}.

\begin{assumption}[Bounded minibatch stochastic gradients]
\label{assump:bounded-minibatch-stochastic-gradients}
There exists $G < \infty$ such that
\[
    \norm{g_t^{(i)}} \le G
    \qquad
    \text{almost surely for all } t \ge t_0
    \text{ and } i = 1, \ldots, b.
\]
\end{assumption}

We now specialize Proposition~\ref{cor:minibatch-variable-range-eb} to the SGD confidence sequences, beginning with the last-iterate distance. For every $t \ge t_0$, define the observable within-minibatch second-moment matrix
\begin{equation}
\label{eq:minibatch-second-moment}
    \widehat\Sigma_t
    \coloneqq
    \frac{1}{b}
    \sum_{i = 1}^b
    g_t^{(i)}
    \bigl(g_t^{(i)}\bigr)^\top.
\end{equation}
As in the single-gradient construction of Section~\ref{sec:refined-confidence-sequences}, set
\[
    H_s^{(i),\mathrm{dist}}
    \coloneqq
    2 A_s \eta_s
    \inner{x^\star - x_s}{g_s^{(i)}},
    \qquad
    i = 1, \ldots, b.
\]
Their averaged centered increment is exactly the stochastic term arising from the distance decomposition with $\bar g_s$ in place of $g_s$. Moreover,
\[
    \sum_{i = 1}^b
    \left(
        H_s^{(i),\mathrm{dist}}
    \right)^2
    =
    4 b A_s^2 \eta_s^2
    (x_s - x^\star)^\top
    \widehat\Sigma_s
    (x_s - x^\star),
\]
while Assumption~\ref{assump:bounded-minibatch-stochastic-gradients} gives
\[
    \left|
        H_s^{(i),\mathrm{dist}}
    \right|
    \le
    2 G A_s \eta_s \sqrt{Z_s}.
\]
Thus, the minibatch quadratic process retains the realized directional second moment encoded by $\widehat\Sigma_s$, rather than reducing the stochastic fluctuations to a fixed variance proxy.

To make the resulting bound fully observable, we use the same recursive domination argument as in the distance construction of Section~\ref{sec:refined-confidence-sequences}. Since
\begin{equation}
\label{eq:spectral:relaxation}
    (x_s - x^\star)^\top
    \widehat\Sigma_s
    (x_s - x^\star)
    \le
    \lambda_{\max}(\widehat\Sigma_s) Z_s,
\end{equation}
the unknown distance can again be replaced by a previously constructed distance envelope. Fix $\alpha \in (0,1)$, and let $R_0$ be any $\cF_{t_0 - 1}$-measurable quantity satisfying $R_0 \ge Z_{t_0}$ almost surely. Define recursively
\begin{equation}
\label{eq:mb-eb-distance-envelope-proxy}
    R_{t_0}^{\mathrm{EB,MB}}
    \coloneqq
    R_0,
    \qquad
    R_s^{\mathrm{EB,MB}}
    \coloneqq
    \widehat U_s^{\mathrm{dist,EB,MB}}(\alpha),
    \quad
    s \ge t_0 + 1.
\end{equation}
For every $t \ge t_0$, set
\begin{align}
\label{eq:mb-eb-observable-distance-variance}
    \widehat V_t^{\mathrm{dist,EB,MB}}
    &\coloneqq
    4 b
    \sum_{s = t_0}^t
    A_s^2 \eta_s^2
    \lambda_{\max}(\widehat\Sigma_s)
    R_s^{\mathrm{EB,MB}},
    \\
\label{eq:mb-eb-observable-distance-range}
    \widehat B_t^{\mathrm{dist,EB,MB}}
    &\coloneqq
    2 G
    \max_{t_0 \le s \le t}
    A_s \eta_s
    \sqrt{R_s^{\mathrm{EB,MB}}},
\end{align}
and define
\begin{equation}
\label{eq:mb-eb-observable-distance}
    \widehat U_{t+1}^{\mathrm{dist,EB,MB}}(\alpha)
    \coloneqq
    \frac{1}{A_t}
    \left[
        R_0
        +
        \sum_{s = t_0}^t
        A_s \eta_s^2
        \norm{\bar g_s}^2
        +
        \frac{1}{b}
        \mathfrak B_\alpha\left(
            \widehat V_t^{\mathrm{dist,EB,MB}},
            \widehat B_t^{\mathrm{dist,EB,MB}}
        \right)
    \right].
\end{equation}
As in the single-gradient case, the construction is recursive: $R_s^{\mathrm{EB,MB}}$ is available before the minibatch at iteration $s$ is observed, and the quantities above determine the distance envelope for the next iterate. The following corollary is the minibatch analogue of Theorem~\ref{thm:observable-distance-eb}.

\begin{corollary}[Refined observable minibatch confidence sequence for $\norm{x_{t+1} - x^\star}^2$]
\label{cor:mb-observable-distance-eb}
Suppose Assumptions~\ref{assump:strong-convexity}--\ref{assump:stepsize-bound} and~\ref{assump:bounded-minibatch-stochastic-gradients} hold, and suppose~\eqref{eq:eb-integrability-condition} is satisfied. Let $R_0$ be an $\cF_{t_0 - 1}$-measurable quantity satisfying $R_0 \ge Z_{t_0}$ almost surely. Then, for every $\alpha \in (0,1)$, the recursively defined process $\{\widehat U_{t+1}^{\mathrm{dist,EB,MB}}(\alpha)\}_{t \ge t_0}$ is well defined and fully observable from $R_0$ and the minibatch SGD trajectory. Moreover,
\[
    \Prob\left(
        \forall t \ge t_0:\
        Z_{t+1}
        \le
        \widehat U_{t+1}^{\mathrm{dist,EB,MB}}(\alpha)
    \right)
    \ge
    1 - \alpha.
\]
\end{corollary}

The spectral relaxation~\eqref{eq:spectral:relaxation} makes the directional quadratic process observable while retaining information about the geometry of the realized minibatch. Indeed, the latent quantity depends on the unknown direction $x_s - x^\star$, whereas $\lambda_{\max}(\widehat\Sigma_s)$ provides a direction-free upper bound determined entirely by the observed stochastic gradients. Moreover,
\[
    \lambda_{\max}(\widehat\Sigma_s)
    \le
    \operatorname{tr}(\widehat\Sigma_s)
    =
    \frac{1}{b}
    \sum_{i = 1}^b
    \norm{g_s^{(i)}}^2
    \le
    G^2,
\]
so the resulting bound can be substantially sharper than replacing the quadratic process by the uniform worst-case factor $G^2$, particularly when the realized gradients are small or their energy is spread across several directions. This refinement also remains computationally tractable: since $\widehat\Sigma_s$ has rank at most $b$, its largest eigenvalue can be computed from a $b \times b$ Gram matrix, without forming or diagonalizing a full $d \times d$ matrix.

We next turn to weighted-average suboptimality. As in the single-gradient construction of Section~\ref{sec:refined-confidence-sequences}, the stochastic term in the weighted suboptimality decomposition is controlled through the directional observations
\[
    H_s^{(i),\mathrm{sub}}
    \coloneqq
    s
    \inner{x^\star - x_s}{g_s^{(i)}},
    \qquad
    i = 1, \ldots, b.
\]
Their averaged centered increment is $s\inner{x^\star - x_s}{\bar\xi_s}$, and
\[
    \sum_{i = 1}^b
    \left(
        H_s^{(i),\mathrm{sub}}
    \right)^2
    =
    b s^2
    (x_s - x^\star)^\top
    \widehat\Sigma_s
    (x_s - x^\star),
\]
while Assumption~\ref{assump:bounded-minibatch-stochastic-gradients} gives
\[
    \left|
        H_s^{(i),\mathrm{sub}}
    \right|
    \le
    G s \sqrt{Z_s}.
\]
Thus, as in the distance construction, the latent quadratic and range processes can be made observable using the same spectral relaxation together with the previously constructed minibatch distance envelopes. For a confidence level $\beta \in (0,1)$, define
\begin{equation}
\label{eq:mb-eb-sub-distance-envelope}
    R_{t_0}^{\mathrm{dist,EB,MB}}(\beta)
    \coloneqq
    R_0,
    \qquad
    R_s^{\mathrm{dist,EB,MB}}(\beta)
    \coloneqq
    \widehat U_s^{\mathrm{dist,EB,MB}}(\beta),
    \quad
    s \ge t_0 + 1.
\end{equation}
For every $t \ge t_0$, set
\begin{align}
\label{eq:mb-eb-sub-observable-variance}
    \widehat V_t^{\mathrm{sub,EB,MB}}(\beta)
    &\coloneqq
    b
    \sum_{s = t_0}^t
    s^2
    \lambda_{\max}(\widehat\Sigma_s)
    R_s^{\mathrm{dist,EB,MB}}(\beta),
    \\
\label{eq:mb-eb-sub-observable-range}
    \widehat B_t^{\mathrm{sub,EB,MB}}(\beta)
    &\coloneqq
    G
    \max_{t_0 \le s \le t}
    s
    \sqrt{
        R_s^{\mathrm{dist,EB,MB}}(\beta)
    }.
\end{align}
Given $\alpha \in (0,1)$, define
\begin{equation}
\label{eq:mb-eb-sub-observable-cs}
\begin{aligned}
    \widehat U_t^{\mathrm{sub,EB,MB}}(\alpha)
    \coloneqq
    \frac{1}{S_t}
    \Bigg[
        &\frac{1}{\mu}
        \sum_{s = t_0}^t
        \norm{\bar g_s}^2
        +
        \frac{\mu t_0(t_0 - 1)}{4}
        R_0
        \\
        &+
        \frac{1}{b}
        \mathfrak B_{\alpha/2}\left(
            \widehat V_t^{\mathrm{sub,EB,MB}}(\alpha/2),
            \widehat B_t^{\mathrm{sub,EB,MB}}(\alpha/2)
        \right)
    \Bigg],
    \qquad
    t \ge t_0.
\end{aligned}
\end{equation}

The following corollary is the minibatch analogue of Theorem~\ref{thm:observable-suboptimality-eb}.

\begin{corollary}[Refined observable minibatch confidence sequence for $f(\bar x_t) - f(x^\star)$]
\label{cor:mb-observable-suboptimality-eb}
Suppose Assumptions~\ref{assump:strong-convexity}--\ref{assump:stepsize-bound} and~\ref{assump:bounded-minibatch-stochastic-gradients} hold. Let $t_0 \ge 4$ and $\eta_t = 2/(\mu(t+1))$ for every $t \ge t_0$, and let $R_0$ be an $\cF_{t_0 - 1}$-measurable quantity satisfying $R_0 \ge Z_{t_0}$ almost surely. Then, for every $\alpha \in (0,1)$, the process $\{\widehat U_t^{\mathrm{sub,EB,MB}}(\alpha)\}_{t \ge t_0}$ is well defined and fully observable from $R_0$ and the minibatch SGD trajectory. Moreover,
\[
    \Prob\left(
        \forall t \ge t_0:\
        f(\bar x_t) - f(x^\star)
        \le
        \widehat U_t^{\mathrm{sub,EB,MB}}(\alpha)
    \right)
    \ge
    1-\alpha.
\]
\end{corollary}

When $b = 1$, we have $\widehat\Sigma_t = g_t g_t^\top$ and $\lambda_{\max}(\widehat\Sigma_t) = \norm{g_t}^2$, so the minibatch distance and suboptimality constructions reduce exactly to their single-gradient empirical Bernstein counterparts in Section~\ref{sec:refined-confidence-sequences}. For general $b$, the factor $b$ in the observable quadratic processes combines with the prefactor $1/b$ outside the empirical Bernstein boundary to yield the usual $1/\sqrt b$ reduction in the leading square-root term, up to the iterated-logarithmic factors from stitching, while the linear range term carries the explicit factor $1/b$.


\section{Experiments}
\label{sec:experiments}

To study the empirical performance of our confidence sequences, we run a series of experiments on a primal support vector machine (SVM) problem, using both real-world and synthetic data.
Section~\ref{sec:experiments-setup} details the optimization problem, the experimental protocol, and the bounds we compare against.
Section~\ref{sec:experiments-1} then discusses the introductory figure (Figure~\ref{fig:figure-introduction}), which illustrates the improvement in certified stopping on a single instance.
Section~\ref{sec:experiments-minibatch} (Figure~\ref{fig:comparing_b}) varies the minibatch size $b$, exhibiting its effect on the relative behavior of the bounds.
Section~\ref{sec:experiments-adaptivity} (Figure~\ref{fig:compare-separability}) compares three synthetic datasets of increasing difficulty while keeping all constants entering the bounds fixed, thereby isolating trajectory adaptivity.
We conclude with two verifications. Section~\ref{sec:experiments-misspecification} (Figures~\ref{fig:compare-G} and~\ref{fig:compare-R_0}) checks the sensitivity of our bounds to conservatively chosen constants, and Section~\ref{sec:experiments_true_single_pass} (Figure~\ref{fig:true-single-pass}) examines whether the same behavior persists in a single-pass regime beyond the finite-dataset formulation used in the preceding experiments.


\subsection{Experimental setup}
\label{sec:experiments-setup}

We consider the SVM primal optimization problem
\begin{equation}
    \label{eq:svm-objective}
    f(x)\coloneqq \frac{\mu}{2}\|x\|^2 + \E_z\big[\max\{0,1-z^\top x\}\big]\,,
\end{equation}
where $z$ follows a probability distribution $\mathbb Q$ on $\R^d$ and satisfies
$\|z\|\le M$ almost surely. This problem is $\mu$-strongly convex and nonsmooth.
For a minibatch size $b\ge 1$, at each time $t\ge 1$ we sample
$z_t^{(1)},\ldots,z_t^{(b)}$ independently from $\mathbb Q$, conditionally on
$\mathcal F_{t-1}$, and define
\begin{equation*}
    g_t^{(i)}
    =
    \mu x_t-z_t^{(i)}\mathbf 1\{1-(z_t^{(i)})^\top x_t\ge 0\},
    \qquad
    \bar g_t=\frac{1}{b}\sum_{i=1}^b g_t^{(i)},
    \qquad
    i=1,\ldots,b,
\end{equation*}
where $\bar g_t$ is used in the SGD update. For $b=1$, this reduces to the single-gradient setting; in that case, we simply write $g_t$. Since \eqref{eq:svm-objective} is not differentiable in general, each $g_t^{(i)}$ is an unbiased stochastic subgradient rather than a stochastic gradient. Although we assume differentiability throughout the paper for notational convenience, the same arguments apply with unbiased stochastic subgradients.

For the explicit constants, first notice that
\begin{equation*}
    \frac{\mu}{2}\|x^\star\|^2\le f(x^\star)\le f(0)=1.
\end{equation*}
We therefore take
$\mathcal X=\{x\in\R^d:\|x\|\le\sqrt{2/\mu}\}$,
which contains $x^\star$ and hence has the same minimizer as the unconstrained problem
\eqref{eq:svm-objective}. On $\mathcal X$,
\begin{equation*}
    \|g_t^{(i)}\|\le \sqrt{2\mu}+M\eqqcolon G
    \qquad\text{almost surely},\qquad i=1,\ldots,b.
\end{equation*}
Thus the bounded-gradient conditions required by the empirical-Bernstein-based confidence sequences hold.
In particular, the conditional Hoeffding lemma gives the valid minibatch
sub-Gaussian proxy $\sigma^2=G^2/b$ for $\bar g_t$.
Finally, we use $R_0=8/\mu$, the squared diameter of $\mathcal X$, as a deterministic
bound on $Z_{t_0}$.

\bigskip
\noindent\textbf{Data and protocol.}
We run experiments on the \emph{covertype.binary} dataset, obtained from the LIBSVM data websitem\footnote{\url{https://www.csie.ntu.edu.tw/~cjlin/libsvmtools/datasets/}} and on synthetic datasets generated with scikit-learn's \texttt{make\_classification}.\footnote{\url{https://scikit-learn.org/stable/modules/generated/sklearn.datasets.make_classification.html}} A bias term is appended to every dataset, and dense features are standardized to zero mean and unit variance.

To account for the fact that we work with finite datasets, we take $\mathbb Q$ to be the uniform distribution over the dataset, rather than the unknown data-generating distribution. Sampling independently from $\mathbb Q$, i.e., with replacement, makes the empirical risk problem a special case of our general setting. To confirm that the performance of our confidence sequences is not an artifact of this special case, we also conduct a \emph{true single-pass} experiment in which SGD uses each data point exactly once, directly targeting the expected risk under the data-generating distribution. The full details and results of this experiment are discussed in Section~\ref{sec:experiments_true_single_pass}.

In all experiments we set $\mu=0.1$, choose the confidence level $1-\alpha=99\%$, use $t_0=4$, and take the stepsize schedule $\eta_t=2/(\mu(t+1))$. For the finite-dataset experiments, the ground-truth quantities $x^\star$ and $f(x^\star)$ are computed beforehand with LIBLINEAR \citep{fan2008liblinear}. The resulting numerical error is several orders of magnitude below the smallest values appearing in our plots, so these quantities can be treated as exact.

Each figure reports a single trajectory of SGD. This is deliberate: our goal is to illustrate how the confidence sequences and the resulting stopping certificates evolve along an individual realized trajectory, rather than to average away their trajectory dependence across runs. Repeating the experiments with several seeds yields essentially identical qualitative behavior, and we therefore report a single run. Similarly, the horizon $T_{\max}$ at which our plots stop is arbitrary: all the confidence sequences we display remain valid beyond it and do not depend on it.

\begin{table}[t]
\centering
\begin{tabular}{lllccc}
\toprule
& Label & Quantity & Anytime-valid & Trajectory-adaptive & Reference \\
\midrule
\multirow{5}{*}{$Z_t$}
& \textbf{Ground truth} & $\|x_t-x^\star\|^2$ & --- & --- & --- \\
& \textbf{H} & $\widehat U^{\mathrm{dist}}_t(\alpha)$ & yes & yes & Theorem~\ref{thm:observable-distance-cs} \footnotemark \\
& \textbf{EB} & $\widehat U^{\mathrm{dist,EB,MB}}_t(\alpha)$ & yes & yes & Corollary~\ref{cor:mb-observable-distance-eb} \\
& \textbf{worst-case} & $\check U^{\mathrm{dist}}_t(\alpha)$ & yes & no & \eqref{eq:worst-case-distance-CS} \\
& \textbf{RSS12} & \eqref{eq:rakhlin_high_prob_bound} & no & no & \citep{rakhlin2012making} \\
\midrule
\multirow{4}{*}{$\bar F_t$}
& \textbf{Ground truth} & $f(\bar x_t)-f(x^\star)$ & --- & --- & --- \\
& \textbf{H} & $\widehat U^{\mathrm{sub}}_t(\alpha)$ & yes & yes & Theorem~\ref{thm:observable-suboptimality-CS} \\
& \textbf{EB} & $\widehat U^{\mathrm{sub,EB,MB}}_t(\alpha)$ & yes & yes & Corollary~\ref{cor:mb-observable-suboptimality-eb} \\
& \textbf{worst-case} & $\check U^{\mathrm{sub}}_t(\alpha)$ & yes & no & \eqref{eq:worst-case-suboptimality-CS} \\
\bottomrule
\end{tabular}
\caption{Quantities compared in the experiments. All bounds are plotted at confidence level $1-\alpha=0.99$ and use the same trajectory.}
\label{tab:compared-bounds}
\end{table}
\footnotetext{In the minibatch experiments, we use the direct minibatch extension described at the beginning of Section~\ref{sec:minibatching}, replacing $\sigma^2$ with $\sigma^2/b$ in the corresponding variance processes and $g_t$ with $\bar g_t$ in the SGD recursions.}

\bigskip
\noindent\textbf{Compared bounds.}
For both the squared distance to the optimum $Z_{t+1}=\|x_{t+1}-x^\star\|^2$ and the suboptimality $f(\bar x_t)-f(x^\star)$ of the weighted average, we compare the quantities listed in Table~\ref{tab:compared-bounds}. The two \textbf{worst-case} baselines are obtained from our Hoeffding-based confidence sequences by replacing the realized squared-gradient terms by their a priori upper bound $G^2$ and propagating this replacement through the recursive distance envelopes. With $\sigma^2=G^2$ denoting the individual-gradient sub-Gaussian proxy, so that $\sigma^2/b$ is the corresponding minibatch proxy, this yields
\begin{equation}
\label{eq:worst-case-distance-CS}
    \check U^\text{dist}_{t_0}(\alpha)
    \coloneqq
    R_0,
    \qquad
    \check U^\text{dist}_{t+1}(\alpha)
    \coloneqq
    \frac{1}{A_t}
    \Bigg(
        R_0
        +
        G^2\sum_{s=t_0}^t A_s\eta_s^2
        +
        4\,\mathfrak{H}_\alpha
        \Bigg(
            \frac{\sigma^2}{b}
            \sum_{s=t_0}^t
            A_s^2\eta_s^2
            \check U_s^\text{dist}(\alpha)
        \Bigg)
    \Bigg)
\end{equation}
for squared distance, and
\begin{equation}
\label{eq:worst-case-suboptimality-CS}
    \check U^\text{sub}_{t}(\alpha)
    \coloneqq
    \frac{1}{S_t}
    \Bigg(
        \frac{t-t_0+1}{\mu}G^2
        +
        \frac{\mu\,t_0(t_0-1)}{4}R_0
        +
        2\,\mathfrak{H}_{\alpha/2}
        \Bigg(
            \frac{\sigma^2}{b}
            \sum_{s=t_0}^t
            s^2
            \check U^\text{dist}_s\Bigl(\frac{\alpha}{2}\Bigr)
        \Bigg)
    \Bigg)
\end{equation}
for suboptimality. Under the deterministic stepsize schedule used in our experiments, these sequences are deterministic and trajectory-independent. They are also well defined recursively: the distance sequence is initialized at $\check U^\text{dist}_{t_0}(\alpha)=R_0$, and the suboptimality sequence uses the already defined distance sequence at level $\alpha/2$. Moreover, since $\|\bar g_t\|^2\le G^2$ almost surely and $\mathfrak H_\alpha$ is nondecreasing, the same recursive domination argument used for the Hoeffding-based confidence sequences shows that the worst-case baselines remain valid confidence sequences. Comparing \textbf{H} with the corresponding worst-case baseline therefore isolates the effect of adapting the Hoeffding-based construction to the realized trajectory.

For the distance to the optimum, we additionally compare against the high-probability bound of \citet[Proposition~1]{rakhlin2012making}, which states that, for any fixed horizon $T\ge 4$, with probability at least $1-\alpha$, simultaneously for all $t\le T$,
\begin{equation}
\label{eq:rakhlin_high_prob_bound}
    Z_t
    \le
    \frac{\bigl(624\log(\log(T)/\alpha)+1\bigr)G^2}{\mu^2 t}.
\end{equation}
This bound stems from the same standard one-step recursion underlying our distance analysis, but differs from $\check U^\text{dist}$ in several respects. First, \eqref{eq:rakhlin_high_prob_bound} is stated in closed form, at the price of a multiplicative constant that is not necessarily optimized, whereas $\check U^\text{dist}$ is defined recursively. Second, \eqref{eq:rakhlin_high_prob_bound} does not account for the variance reduction due to minibatching, whereas $\check U^\text{dist}$ uses the minibatch variance proxy $\sigma^2/b$. The effect of this difference will be illustrated in Section~\ref{sec:experiments-minibatch} by varying $b$. Third, \eqref{eq:rakhlin_high_prob_bound} is a finite-horizon simultaneous guarantee rather than a confidence sequence in our sense, since the horizon $T$ must be fixed in advance. We take $T=T_{\max}$, the most favorable valid choice for covering the entire plotted horizon. More recently, \citet{pham2025time} obtained an infinite-horizon time-uniform analogue that removes the need to specify $T$, but with a larger leading constant ($1008$ instead of $624$). Since our goal is to compare against a favorable explicit literature benchmark, we therefore use \eqref{eq:rakhlin_high_prob_bound} in the experiments. 

The bound \eqref{eq:rakhlin_high_prob_bound} is stated for the stepsize schedule $\eta_t=1/(\mu t)$, while our experiments use $\eta_t=2/(\mu(t+1))$. Both belong to the standard $1/t$ stepsize regime for strongly convex stochastic optimization, and we use \eqref{eq:rakhlin_high_prob_bound} as the corresponding explicit literature benchmark. Overall, the two baselines serve different purposes: \textbf{RSS12} provides an explicit finite-horizon literature benchmark, while $\check U^\text{dist}$ is the natural non-adaptive counterpart of our Hoeffding-based confidence sequence under the same stepsize schedule and minibatch variance proxy.

For the suboptimality of the weighted average, we are not aware of a high-probability bound in the literature stated with explicit constants. \citet{harvey2019simple} establish such a bound for the same linearly weighted averaging scheme and stepsize schedule, but only up to unspecified absolute constants, which makes a direct quantitative comparison impossible without further analysis. We therefore compare only against the worst-case baseline $\check U^{\text{sub}}_t$ above.


\subsection{Introductory plot}
\label{sec:experiments-1}

We begin with an experiment illustrating our main findings. We run SGD on \emph{covertype.binary} with minibatch size $b=256$. Figure~\ref{fig:figure-introduction} reports, on log-log axes, the squared distance to the optimum (left) and the suboptimality of the weighted average (right), with the bounds of Table~\ref{tab:compared-bounds} evaluated along the same trajectory. On the right panel, we mark for each bound the first time it drops below the threshold $\varepsilon=10^{-3}$. Stopping SGD at this time certifies, with probability $1-\alpha$, that the returned weighted average satisfies $f(\bar x_t)-f(x^\star)\le \varepsilon$.

Note first that $\check U^\text{dist}$ lies several orders of magnitude below \textbf{RSS12}. As explained in Section~\ref{sec:experiments-setup}, this reflects the larger explicit constants in \textbf{RSS12} together with the fact that it does not account for the variance reduction due to minibatching. The role of minibatching is examined separately in Section~\ref{sec:experiments-minibatch}.

Turning to the confidence sequences, the Hoeffding-based bound $\widehat U^\text{sub}$ (\textbf{H}) certifies the threshold $6.1$ times earlier than the worst-case baseline, while the empirical Bernstein minibatch bound $\widehat U^{\text{sub,EB,MB}}$ (\textbf{EB}) does so $312$ times earlier. In terms of gradient lookups, this corresponds to $5.1\times 10^6$ lookups for \textbf{EB}, $2.6\times 10^8$ for \textbf{H}, and $1.6\times 10^9$ for the worst-case baseline, all for the same confidence level and target accuracy.

These improvements illustrate the adaptivity of our confidence sequences, which make use of observed quantities along the trajectory rather than their a priori upper bounds. The further improvement of the empirical Bernstein sequence over the Hoeffding-based one reflects its ability to exploit the realized second-moment structure within each minibatch, rather than relying only on the fixed sub-Gaussian variance proxy. Overall, the experiment shows that trajectory adaptation can substantially tighten the resulting certificates while preserving their anytime-valid guarantees.


\subsection{Effect of the minibatch size}
\label{sec:experiments-minibatch}

\begin{figure}[t]
\centering
\includegraphics[width=\textwidth]{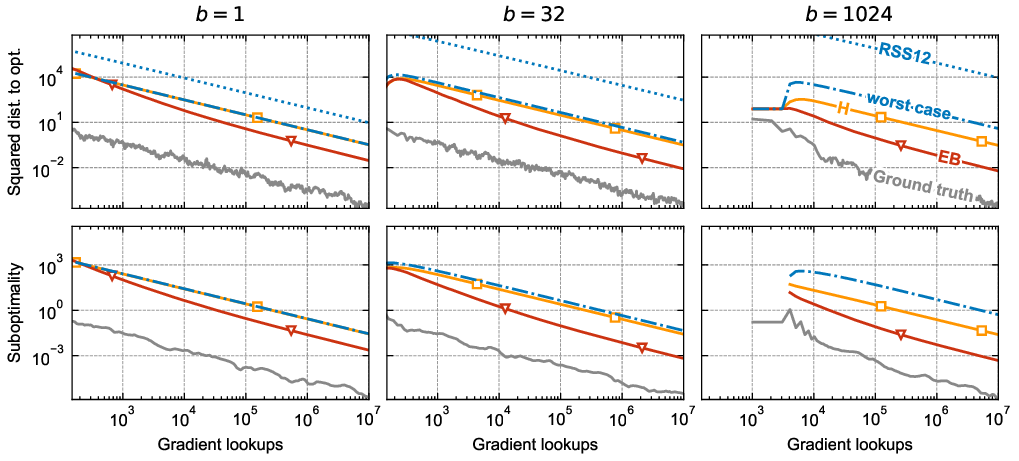}
\caption{SGD on the \emph{covertype.binary} dataset for minibatch sizes $b=1$ (left), $32$ (middle), and $1024$ (right). Top row: squared distance to the optimum. Bottom row: suboptimality of the weighted average. All panels use the conventions of Figure~\ref{fig:figure-introduction}: ground truth (solid gray line), \textbf{worst-case} bound (dash-dotted blue line), \textbf{H} (solid orange line with squares), and \textbf{EB} (solid red line with triangles); \textbf{RSS12} (dotted blue line) is shown in the distance panels only.}
\label{fig:comparing_b}
\end{figure}

We now repeat the previous experiment on \emph{covertype.binary} with $b\in\{1,32,1024\}$, all other parameters being unchanged. All three runs are given the same budget of $10^7$ gradient lookups. Although the performance of SGD may differ, the goal of this experiment is not to isolate how fast the algorithm converges, but how tightly its progress can be certified. The question of choosing the minibatch size $b$ is beyond the scope of this paper; here we only study its effect on the confidence sequences.

Figure~\ref{fig:comparing_b} reports the results. The ordering of the curves remains the same across all three experiments, and the overall picture of Section~\ref{sec:experiments-1} is unchanged. We make the following additional observations:
\begin{itemize}
\item The bound of \citet{rakhlin2012making} is about one order of magnitude above our \textbf{worst-case} baseline at $b=1$. As $b$ grows, the gap widens to three orders of magnitude at $b=1024$. At $b=1$, the difference is primarily due to the explicit constants, while for larger values of $b$ the gap widens because $\check U^\text{dist}$ captures the variance reduction due to minibatching through the proxy $\sigma^2/b$, whereas \textbf{RSS12} does not.

\item Hoeffding-based \textbf{H} is indistinguishable from \textbf{worst-case} at $b=1$. The adaptive contribution in \textbf{H} comes from retaining the realized squared minibatch-gradient norms $\|\bar g_s\|^2$ in the deterministic accumulation, whereas the recursive deviation term uses the fixed variance proxy $\sigma^2/b$. At $b=1$, the latter dominates, so replacing $\|g_s\|^2$ by $G^2$ changes little. As $b$ increases, the deviation term shrinks and the trajectory-dependent accumulation becomes relatively more important, making the effect of adaptivity increasingly visible.

\item Empirical-Bernstein-based \textbf{EB} improves on \textbf{H} by about one order of magnitude already at $b=1$, and this gap widens with $b$. Unlike \textbf{H}, whose deviation term is controlled by the fixed variance proxy $\sigma^2/b$, \textbf{EB} exploits the realized second-moment structure within each minibatch through $\lambda_{\max}(\widehat\Sigma_t)$. The contribution of each minibatch to the leading square-root term therefore carries the usual $1/\sqrt b$ scaling while adapting to $\lambda_{\max}(\widehat\Sigma_t)\le\operatorname{tr}(\widehat\Sigma_t)\le G^2$. In addition, the lower-order range term carries an explicit factor $1/b$ and hence decreases faster with the minibatch size.

\item The gap between the ground truth and the \textbf{worst-case} baseline also widens with $b$. The \textbf{EB} confidence sequence nevertheless recovers a growing share of this gap. Define
\[
    \rho_t\coloneqq
    \frac{\log\bigl(\check U_t^{\mathrm{sub}}(\alpha)/\widehat U_t^{\mathrm{sub,EB,MB}}(\alpha)\bigr)}
    {\log\bigl(\check U_t^{\mathrm{sub}}(\alpha)/(f(\bar x_t)-f(x^\star))\bigr)},
\]
the fraction of the gap between the \textbf{worst-case} bound and the ground truth recovered by \textbf{EB} on a logarithmic scale. At the end of the budget, we obtain $\rho=24.8\%$, $43.6\%$, and $53.8\%$ for $b=1,32,1024$, respectively.
\end{itemize}


\subsection{Adaptivity across problem instances}
\label{sec:experiments-adaptivity}

\begin{figure}[t]
\centering
\includegraphics[width=\textwidth]{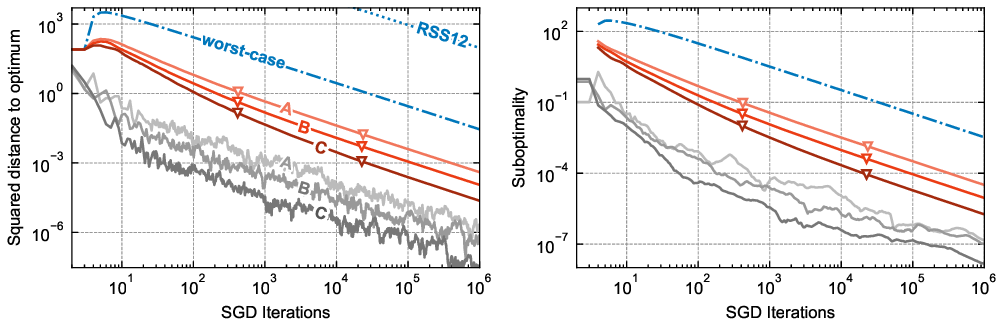}
\caption{SGD on three datasets generated with \texttt{make\_classification}, with increasing class separability from \textbf{A} to \textbf{C}. Left: squared distance to the optimum. Right: suboptimality of the weighted average. Each panel shows the ground truth (solid gray lines) and \textbf{EB} (solid red lines) for each dataset, with increasing shade darkness from \textbf{A} to \textbf{C} and each curve labelled accordingly, together with the \textbf{worst-case} bound (dash-dotted blue line); \textbf{RSS12} (dotted blue line) is shown in the distance panel only.}
\label{fig:compare-separability}
\end{figure}

In this experiment, we aim to isolate the trajectory adaptivity of the empirical Bernstein confidence sequence. We generate three datasets with \texttt{make\_classification}, of increasing class separability, and report the evolution of the ground-truth quantities $Z_t$ and $f(\bar x_t)-f(x^\star)$, together with \textbf{EB} (Figure~\ref{fig:compare-separability}).

\begin{table}[t]
\centering
\begin{tabular}{lcc}
\toprule
& \texttt{class\_sep} & \texttt{flip\_y}\\
\midrule
\textbf{A} (low separability) & 0.2 & 0.1 \\
\textbf{B} (medium separability) & 1.0 & 0.02 \\
\textbf{C} (high separability) & 5.0 & 0.0 \\
\bottomrule
\end{tabular}
\caption{Parameters of each generated dataset.}
\label{tab:separability-datasets}
\end{table}

The dataset parameters are given in Table~\ref{tab:separability-datasets}. For all three, we set $\texttt{n\_samples}=10^5$, $\texttt{n\_features}=20$, and $\texttt{n\_clusters\_per\_class}=1$. Each dataset is then rescaled so that $G$ has the same value across all three. Since $\mu$, $\alpha$, $b=256$, and $R_0$ are also fixed, the \textbf{worst-case} bound is identical across the three instances.

On these instances, SGD makes faster progress on the more highly separated datasets, and the confidence sequences follow the same ordering. The relationship between the geometry of the data and the performance of SGD on the SVM primal problem is outside the scope of this paper; what matters here is that this variation in performance is invisible to the \textbf{worst-case} analysis but reflected by the trajectory-adaptive confidence sequences. The same deterministic bound applies to all three runs, whereas \textbf{EB} separates them and, after $10^6$ SGD iterations, is roughly $17$ times tighter for \textbf{C} than for \textbf{A} in suboptimality: $\widehat U^{\mathrm{sub,EB,MB}}=1.9\times10^{-6}$ for \textbf{C}, compared with $3.2\times10^{-5}$ for \textbf{A}.

This behavior is consistent with the empirical Bernstein construction, which adapts to the realized second-moment structure within each minibatch through $\lambda_{\max}(\widehat\Sigma_t)$. Better separated datasets have fewer samples with an active hinge loss, and in our experiments this is accompanied by substantially smaller realized values of $\lambda_{\max}(\widehat\Sigma_t)$. The average value of $\lambda_{\max}(\widehat\Sigma_t)$ over the last $1\%$ of the run decreases by approximately $73\%$ from \textbf{A} to \textbf{B}, and by approximately $98\%$ from \textbf{A} to \textbf{C}. A deterministic bound depending on $G$ alone cannot distinguish these trajectories.


\subsection{Robustness to parameter misspecification}
\label{sec:experiments-misspecification}

The bounds considered above depend on deterministic problem constants such as $G$, $\sigma^2$, and $R_0$, whose roles differ across the different constructions. These quantities are rarely available exactly in practice. The situation is not symmetric: choosing a value that is too small can invalidate the corresponding guarantee, whereas a conservative upper bound preserves validity at the price of tightness. The practical question is therefore what is lost when these constants are chosen too conservatively. We do not vary $\mu$, since it enters the stepsize schedule and changing it would therefore also change the SGD trajectory itself.

\begin{figure}[t]
\centering
\includegraphics[width=\textwidth]{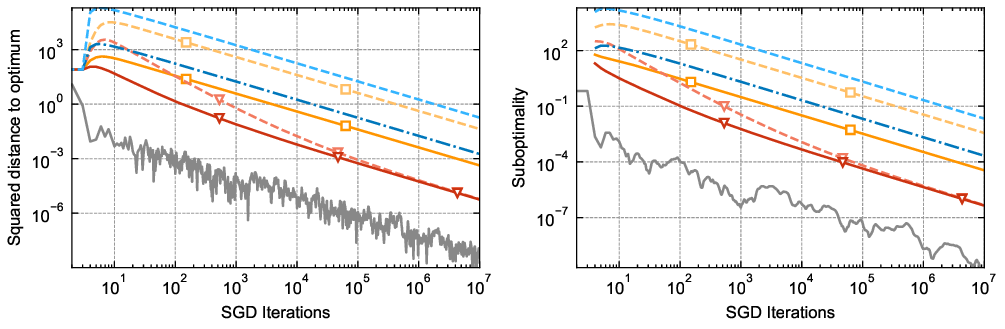}
\caption{SGD on dataset \textbf{B} from Table~\ref{tab:separability-datasets}, with $b=256$, comparing the bounds obtained with the correct gradient bound $G$ and with $G$ overestimated by a factor of $10$. Solid and dash-dotted lines correspond to the correct value, while dashed lines in lighter shades correspond to the overestimated one. Each panel shows the ground truth (solid gray line), the \textbf{worst-case} bound (blue line), \textbf{H} (orange line with squares), and \textbf{EB} (red line with triangles).}
\label{fig:compare-G}
\end{figure}

\bigskip
\noindent\textbf{Misspecification of the gradient bound $G$.}
We first compare the bounds computed with the correct value of $G$ to the same bounds computed with $G$ overestimated by a factor of $10$. For \textbf{H} and the \textbf{worst-case} baseline, we use the valid individual-gradient proxy $\sigma^2=G^2$, so overestimating $G$ by a factor of $10$ also overestimates $\sigma^2$ by a factor of $100$. The experiment is run on dataset \textbf{B} from the previous subsection, with the SGD trajectory itself unchanged. Figure~\ref{fig:compare-G} reports both families of curves.

The \textbf{worst-case} bounds degrade as expected for both squared distance and suboptimality. Their leading asymptotic terms scale with $G^2$, and the misspecified curves lie approximately a factor of $100$ above their well-specified counterparts. A similarly strong effect is observed for Hoeffding-based \textbf{H}, reflecting the dependence of its deviation term on the fixed variance proxy $\sigma^2/b$ rather than on the realized second-moment structure of the stochastic gradients.

Empirical-Bernstein-based \textbf{EB} behaves differently. Its leading deviation term adapts to the realized second-moment structure within each minibatch through $\lambda_{\max}(\widehat\Sigma_t)$, while $G$ enters only through the lower-order range term of the empirical Bernstein boundary. The misspecified and well-specified \textbf{EB} curves are consequently asymptotically equivalent. They differ noticeably only at first and agree within a factor of $1.06$ after $10^7$ iterations, for both $\widehat U^{\mathrm{dist,EB,MB}}$ and $\widehat U^{\mathrm{sub,EB,MB}}$.

\begin{figure}[t]
    \centering
    \includegraphics[width=\textwidth]{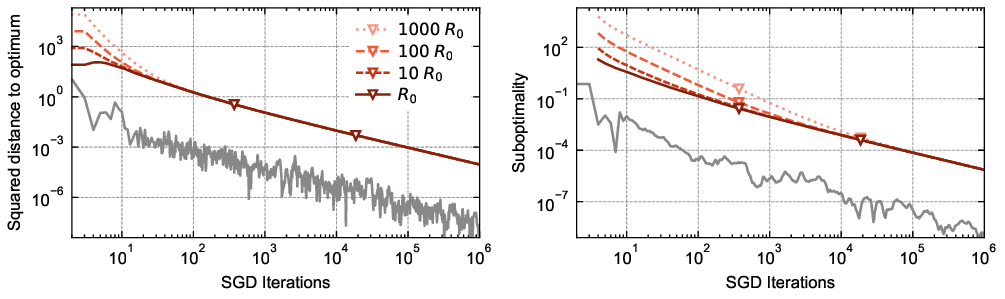}
    \caption{SGD on \texttt{make\_classification} dataset, $b=256$, comparing \textbf{EB} computed with the correct value of $R_0$ and with $R_0$ over-estimated by factors $10$, $100$ and $1000$. 
    Each panel shows the ground truth (solid gray line) and the four confidence sequences (red lines, in increasingly light shades and increasingly sparse dotting as the over-estimation factor grows, as indicated in the legend).
    }
    \label{fig:compare-R_0}
\end{figure}

\bigskip
\noindent\textbf{Misspecification of the initial distance bound $R_0$.}
We repeat the experiment with $R_0$ overestimated by factors $10$, $100$, and $1000$, while keeping all other quantities fixed, and report \textbf{EB} in each case in Figure~\ref{fig:compare-R_0}. Unlike $G$, $R_0$ enters through the initialization of the recursive distance envelope, and its influence therefore diminishes as the trajectory-dependent terms accumulate. For the squared-distance confidence sequence, the effect is pronounced initially but vanishes rapidly: with $R_0$ overestimated by a factor of $1000$, the confidence sequence initially differs from the correctly specified one by the same factor, but their relative difference is only $4\%$ after $100$ iterations and $0.01\%$ after $10^3$ iterations.

Taken together, these two experiments show that, along the trajectories considered here, the effect of conservative choices of $G$ and $R_0$ on \textbf{EB} is largely transient. This reflects the trajectory adaptivity of the empirical Bernstein construction: its leading deviation term is driven by realized second-moment quantities, while $G$ enters through the lower-order range term and the influence of $R_0$ is rapidly attenuated by the recursive construction. In practice, this suggests that safe upper bounds on $G$ and $R_0$ need not be specified very precisely for the resulting \textbf{EB} certificates to become largely insensitive to them later in the run.


\subsection{True single-pass}
\label{sec:experiments_true_single_pass}

\begin{figure}[t]
    \centering
    \includegraphics[width=\textwidth]{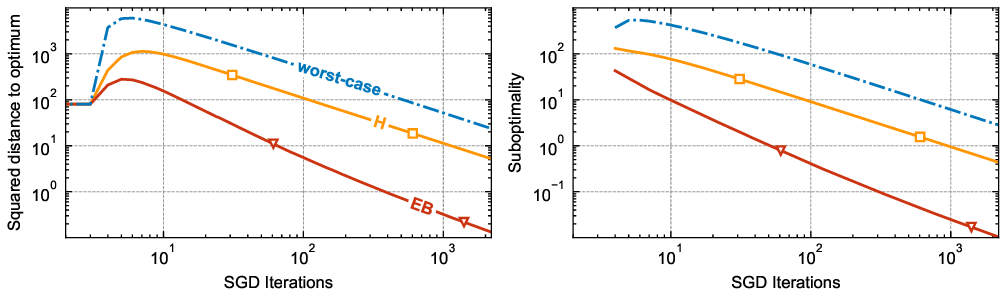}
    \caption{SGD on the \emph{covertype.binary} dataset in a true single-pass regime, with $b=256$. Each panel shows the \textbf{worst-case} bound (dash-dotted blue line), \textbf{H} (solid orange line with squares), and \textbf{EB} (solid red line with triangles).}
    \label{fig:true-single-pass}
\end{figure}

All experiments above sample with replacement from the empirical distribution of a finite dataset, which makes \eqref{eq:svm-objective} the empirical risk and allows $x^\star$ to be computed beforehand. We now turn to the population setting, where the observations are modeled as i.i.d.\ samples from an unknown data-generating distribution $\mathbb P$. Under this standard model, processing the dataset once in an order chosen independently of the observations has the same distribution as sequential i.i.d.\ sampling from $\mathbb P$. Two consequences follow: the minimizer of the expected risk cannot be computed beforehand, so no ground truth is available, and the run is necessarily limited by the available sample size.

We run this experiment on \emph{covertype.binary}, which contains $N=581012$ observations, with minibatch size $b=256$, and report the evolution of \textbf{worst-case}, \textbf{H}, and \textbf{EB} in Figure~\ref{fig:true-single-pass}. Using complete minibatches, the run stops after $\lfloor N/b\rfloor=2269$ SGD iterations, with no observation reused. The picture is essentially the same as the one observed throughout this section. For suboptimality, $\widehat U_t^{\mathrm{sub}}$ is $6.4$ times tighter than $\check U_t^{\mathrm{sub}}$, while $\widehat U_t^{\mathrm{sub,EB,MB}}$ is $2.7\times10^2$ times tighter. For comparison, in the experiment of Section~\ref{sec:experiments-1}, the corresponding factors are also $6.4$ and $2.7\times10^2$ after the same number of iterations. The improvements reported above are therefore not an artifact of repeatedly sampling from a finite empirical distribution and persist in this single-pass expected-risk setting.


\subsection{Summary of the empirical findings}

As shown in these experiments, on real-world data our confidence sequences can certify a target accuracy several orders of magnitude earlier than the trajectory-independent bounds against which they are compared. This improvement reflects their ability to adapt to the realized SGD trajectory. When all problem constants are fixed and only the problem instance varies, the worst-case bounds remain unchanged while our confidence sequences adapt to the instance. Minibatching further amplifies these differences, and our bounds recover a growing share of the gap between the worst-case bounds and the ground truth as $b$ increases.

The two constructions nevertheless behave differently. The empirical-Bernstein-based sequences are consistently tighter than their Hoeffding-based counterparts in our experiments and appear particularly attractive when the stochastic gradients can be uniformly bounded. Their leading deviation terms adapt to realized second-moment quantities, while the effects of conservative choices of $G$ and $R_0$ become negligible later in the runs considered above. By contrast, the Hoeffding-based sequences retain an asymptotic dependence on the fixed sub-Gaussian variance proxy $\sigma^2$.

Even with these improvements, a considerable gap remains between the confidence sequences and the ground truth in the experiments where the latter is available. Whether this residual gap can be reduced through a sharper analysis of the SGD dynamics is an interesting question that we leave open.


\section*{Acknowledgements}

Liviu Aolaritei acknowledges support from the Swiss National Science Foundation through the Postdoc.Mobility Fellowship (grant agreement P500PT\_222215). Funded in part by the European Union (ERC-2022-SYG-OCEAN-101071601). Views and opinions expressed are however those of the author(s) only and do not necessarily reflect those of the European Union or the European Research Council Executive Agency. Neither the European Union nor the granting authority can be held responsible for them.

\bibliographystyle{abbrvnat} 
\bibliography{bibfile.bib}


\appendix

\section{Proofs}
\label{sec:proofs}


\subsection{Proofs of Section~\ref{sec:subgaussian-confidence-sequences}}

\subsubsection{Proof of Lemma~\ref{prop:dyadic-stitching}}

For each $\lambda > 0$, define
\[
    M_t(\lambda)
    \coloneqq
    \exp\left(
        \lambda \bar X_t - 2^\gamma \lambda^2 V_t
    \right),
    \qquad t \ge t_0-1,
\]
with $M_{t_0-1}(\lambda) = 1$. By the conditional exponential bound,
\[
    \E[M_t(\lambda) \mid \mathcal F_{t-1}]
    =
    M_{t-1}(\lambda)
    \E\left[
        \exp\left(
            \lambda X_t - 2^\gamma \lambda^2 v_t
        \right)
        \,\middle|\,
        \mathcal F_{t-1}
    \right]
    \le
    M_{t-1}(\lambda).
\]
Thus $\{M_t(\lambda)\}_{t \ge t_0 - 1}$ is a nonnegative supermartingale. By Ville's inequality, for every $c > 0$,
\[
    \Prob\left(
        \exists\, t \ge t_0-1:
        \lambda \bar X_t - 2^\gamma \lambda^2 V_t \ge c
    \right)
    \le e^{-c}.
\]
Equivalently,
\begin{equation}
\label{eq:fixed-lambda}
    \Prob\left(
        \exists\, t \ge t_0-1:
        \bar X_t \ge \frac{c}{\lambda} + 2^\gamma \lambda V_t
    \right)
    \le e^{-c}.
\end{equation}

We now stitch these fixed-$\lambda$ bounds over dyadic ranges of the variance process. Define
\[
    I_0 \coloneqq [0, 1],
    \qquad
    I_m \coloneqq (2^{m-1}, 2^m],
    \qquad m \ge 1.
\]
For $m \ge 0$, set
\[
    \alpha_m \coloneqq \frac{\alpha}{(m + 1)(m + 2)},
    \qquad
    c_m \coloneqq \log\frac{1}{\alpha_m}
    =
    \log \frac{1}{\alpha}
    +
    \log \bigl((m+1)(m+2)\bigr).
\]
Then $\sum_{m=0}^\infty \alpha_m = \alpha$. For each $m \ge 0$, choose
\[
    \lambda_m
    \coloneqq
    \sqrt{
        \frac{c_m}{2^\gamma\, 2^m}
    } .
\]
Applying \eqref{eq:fixed-lambda} with $(\lambda, c) = (\lambda_m, c_m)$ gives
\[
    \Prob\left(
        \exists\, t \ge t_0-1:
        \bar X_t
        \ge
        \frac{c_m}{\lambda_m}
        +
        2^\gamma \lambda_m V_t
    \right)
    \le \alpha_m .
\]
Suppose now that $V_t \in I_m$. Then $V_t \le 2^m$, so on the complement of the above event,
\[
    \bar X_t
    <
    \frac{c_m}{\lambda_m}
    +
    2^\gamma \lambda_m 2^m .
\]
By the choice of $\lambda_m$,
\[
    \frac{c_m}{\lambda_m}
    =
    \sqrt{2^\gamma c_m 2^m},
    \qquad
    2^\gamma \lambda_m 2^m
    =
    \sqrt{2^\gamma c_m 2^m}.
\]
Hence
\[
    \bar X_t
    <
    2 \sqrt{2^\gamma c_m 2^m}.
\]
If $m \ge 1$, then $V_t > 2^{m-1}$, and so $2^m < 2 V_t$. Therefore
\[
    \bar X_t
    <
    2 \sqrt{2^\gamma c_m 2 V_t}
    =
    2^{(\gamma+3)/2}\sqrt{V_t c_m}.
\]
If $m=0$, then $V_{t,\mathrm{eff}}=1$, and
\[
    \bar X_t
    <
    2 \sqrt{2^\gamma c_0}
    \le
    2^{(\gamma+3)/2} \sqrt{V_{t,\mathrm{eff}} c_0}.
\]
Thus, for every $m\ge 0$, outside an event of probability at most
$\alpha_m$,
\[
    V_t\in I_m
    \quad\Longrightarrow\quad
    \bar X_t
    <
    2^{(\gamma + 3)/2} \sqrt{V_{t,\mathrm{eff}} c_m}.
\]

Taking a union bound over $m \ge 0$, with probability at least $1-\sum_{m=0}^\infty \alpha_m = 1 - \alpha$, the preceding implication holds simultaneously for all $m \ge 0$ and all $t \ge t_0 - 1$. On this event, fix any $t \ge t_0$ and let
\[
    m_t \coloneqq \left\lceil\log_2 V_{t,\mathrm{eff}}\right\rceil .
\]
Then $V_t \in I_{m_t}$, and therefore
\[
    \bar X_t
    \le
    2^{(\gamma+3)/2}
    \sqrt{
        V_{t,\mathrm{eff}}
        \left(
            \log\frac{1}{\alpha}
            +
            \log \bigl((m_t + 1)(m_t + 2)\bigr)
        \right)
    } .
\]
Since $t \ge t_0$ was arbitrary, the result follows.


\subsubsection{Proof of Proposition~\ref{thm:adaptive-distance-CS}}

By Lemma~\ref{lem:distance-decomposition}, for every $t \ge t_0$,
\[
    Z_{t+1}
    \le
    \frac{1}{A_t}
    \left[
        Z_{t_0}
        +
        \sum_{s = t_0}^t
        A_s \eta_s^2 \norm{g_s}^2
        +
        \bar X_t
    \right].
\]
It remains to control $\bar X_t$ uniformly over time. Recall that
\[
    X_t
    =
    2 A_t \eta_t
    \inner{x^\star - x_t}{\xi_t}.
\]
Since $x_t$, $\eta_t$, and $A_t$ are $\cF_{t-1}$-measurable, applying Assumption~\ref{assump:stochastic-gradients}(ii) with the predictable direction $2 A_t \eta_t(x^\star - x_t)$ gives
\[
    \E\left[
        e^{\lambda X_t}
        \,\middle|\,
        \cF_{t-1}
    \right]
    \le
    \exp\left(
        2 \lambda^2 \sigma^2 A_t^2 \eta_t^2 Z_t
    \right)
\]
for every $\lambda\ge0$. Thus, the sequences $\{X_t\}$ and $\{v_t^{\mathrm{dist}}\}$ satisfy the hypotheses of Lemma~\ref{prop:dyadic-stitching} with
\[
    v_t^{\mathrm{dist}}
    =
    \sigma^2 A_t^2 \eta_t^2 Z_t,
    \qquad
    \gamma=1.
\]
Applying Lemma~\ref{prop:dyadic-stitching} therefore yields
\[
    \Prob\left(
        \forall t\ge t_0:\
        \bar X_t
        \le
        4\,\mathfrak H_\alpha\left(
            V_t^{\mathrm{dist}}
        \right)
    \right)
    \ge
    1-\alpha.
\]
Substituting this bound into the preceding decomposition and using the definition of $U_{t+1}^{\mathrm{dist}}(\alpha)$ in~\eqref{eq:Udist} proves the result.


\subsubsection{Proof of Theorem~\ref{thm:observable-distance-cs}}

We first check that the recursion is well-defined. The quantity $\widehat U^{\mathrm{dist}}_{t_0 + 1}(\alpha)$ is defined explicitly from $R_0$ and the data at time $t_0$. Inductively, if
\[
    \widehat U^{\mathrm{dist}}_{t_0 + 1}(\alpha), \dots, \widehat U^{\mathrm{dist}}_{t}(\alpha)
\]
have been defined, then $\widehat V_t^{\mathrm{dist}}$ is also defined, which in turn defines $\widehat U^{\mathrm{dist}}_{t+1}(\alpha)$. Thus the sequence is recursively well-defined. Moreover, each $\widehat U^{\mathrm{dist}}_{t+1}(\alpha)$ depends only on $R_0$ and the observed trajectory up to time $t$, so the construction is fully observable. Now let
\[
    U^{\mathrm{dist}}_{t+1}(\alpha)
    \coloneqq
    \frac{1}{A_t}
    \left[
        Z_{t_0}
        +
        \sum_{s = t_0}^t A_s \eta_s^2 \norm{g_s}^2
        +
        4\,\mathfrak H_\alpha\left(V_t^{\mathrm{dist}}\right)
    \right]
\]
be the adaptive confidence sequence from Proposition~\ref{thm:adaptive-distance-CS}, where
\[
    V_t^{\mathrm{dist}}
    =
    \sum_{s = t_0}^t \sigma^2 A_s^2 \eta_s^2 Z_s.
\]
Set
\[
    E_\alpha
    \coloneqq
    \left\{
        \forall t \ge t_0:\ 
        Z_{t+1} \le U^{\mathrm{dist}}_{t+1}(\alpha)
    \right\}.
\]
By Proposition~\ref{thm:adaptive-distance-CS}, $\Prob(E_\alpha) \ge 1-\alpha$. We will show that on $E_\alpha$,
\[
    U^{\mathrm{dist}}_{t+1}(\alpha) \le \widehat U^{\mathrm{dist}}_{t+1}(\alpha)
    \qquad
    \text{for all } t \ge t_0.
\]
This will imply
\[
    Z_{t+1}
    \le
    U^{\mathrm{dist}}_{t+1}(\alpha)
    \le
    \widehat U^{\mathrm{dist}}_{t+1}(\alpha)
    \qquad
    \text{for all } t \ge t_0
\]
on $E_\alpha$, and therefore establish the claim. By~\eqref{eq:v_eff} and~\eqref{eq:subgaussian-stitched-boundary}, the map $v \mapsto \mathfrak H_\alpha(v)$ is nondecreasing on $[0,\infty)$. We now prove by induction on $t \ge t_0$ that on $E_\alpha$,
\[
    V_t^{\mathrm{dist}} \le \widehat V_t^{\mathrm{dist}}
    \qquad\text{and}\qquad
    U^{\mathrm{dist}}_{t+1}(\alpha) \le \widehat U^{\mathrm{dist}}_{t+1}(\alpha).
\]
For the base case $t=t_0$, since $R_0\ge Z_{t_0}$, we have
\[
    V_{t_0}^{\mathrm{dist}}
    =
    \sigma^2 A_{t_0}^2 \eta_{t_0}^2 Z_{t_0}
    \le
    \sigma^2 A_{t_0}^2 \eta_{t_0}^2 R_0
    =
    \widehat V_{t_0}^{\mathrm{dist}}.
\]
By monotonicity,
\[
    \mathfrak H_\alpha \left(V_{t_0}^{\mathrm{dist}}\right) 
    \le 
    \mathfrak H_\alpha\left(\widehat V_{t_0}^{\mathrm{dist}}\right).
\]
Using also $Z_{t_0} \le R_0$, we obtain
\begin{align*}
    U^{\mathrm{dist}}_{t_0 + 1}(\alpha)
    =
    \frac{1}{A_{t_0}}
    \left[
        Z_{t_0}
        +
        A_{t_0} \eta_{t_0}^2 \norm{g_{t_0}}^2
        +
        4\,\mathfrak H_\alpha \left(V_{t_0}^{\mathrm{dist}}\right)
    \right]
    &\le
    \frac{1}{A_{t_0}}
    \left[
        R_0
        +
        A_{t_0} \eta_{t_0}^2 \norm{g_{t_0}}^2
        +
        4\,\mathfrak H_\alpha\left(\widehat V_{t_0}^{\mathrm{dist}}\right)
    \right]
    \\&=
    \widehat U^{\mathrm{dist}}_{t_0 + 1}(\alpha).
\end{align*}
Now fix $t \ge t_0+1$, and suppose that for every $s = t_0, \dots, t-1$,
\[
    V_s^{\mathrm{dist}} \le \widehat V_s^{\mathrm{dist}}
    \qquad \text{and}\qquad
    U^{\mathrm{dist}}_{s+1}(\alpha)\le \widehat U^{\mathrm{dist}}_{s+1}(\alpha).
\]
Since we are on $E_\alpha$, for each $s=t_0+1,\dots,t$,
\[
    Z_s \le U^{\mathrm{dist}}_s(\alpha) \le \widehat U^{\mathrm{dist}}_s(\alpha).
\]
Therefore,
\[
    V_t^{\mathrm{dist}}
    =
    \sigma^2 A_{t_0}^2 \eta_{t_0}^2 Z_{t_0}
    +
    \sum_{s=t_0+1}^t \sigma^2 A_s^2 \eta_s^2 Z_s
    \le
    \sigma^2 A_{t_0}^2 \eta_{t_0}^2 R_0
    +
    \sum_{s = t_0 + 1}^t \sigma^2 A_s^2 \eta_s^2 \widehat U^{\mathrm{dist}}_s(\alpha)
    =
    \widehat V_t^{\mathrm{dist}}.
\]
By monotonicity again,
\[
    \mathfrak H_\alpha\left(V_t^{\mathrm{dist}}\right)
    \le
    \mathfrak H_\alpha\left(\widehat V_t^{\mathrm{dist}}\right).
\]
Using also $Z_{t_0} \le R_0$, we conclude that
\begin{align*}
    U^{\mathrm{dist}}_{t+1}(\alpha)
    =
    \frac{1}{A_t}
    \left[
        Z_{t_0}
        +
        \sum_{s = t_0}^t A_s \eta_s^2 \norm{g_s}^2
        +
        4\, \mathfrak H_\alpha\left(V_t^{\mathrm{dist}}\right)
    \right]
    &\le
    \frac{1}{A_t}
    \left[
        R_0
        +
        \sum_{s = t_0}^t A_s \eta_s^2 \norm{g_s}^2
        +
        4\, \mathfrak H_\alpha\left(\widehat V_t^{\mathrm{dist}}\right)
    \right]
    \\&=
    \widehat U^{\mathrm{dist}}_{t+1}(\alpha).
\end{align*}
This closes the induction. Hence, on $E_\alpha$,
\[
    Z_{t+1} \le \widehat U^{\mathrm{dist}}_{t+1}(\alpha)
    \qquad\text{for all } t \ge t_0.
\]
Therefore,
\[
    \Prob\left(
        \forall t \ge t_0:\ 
        Z_{t+1} \le \widehat U^{\mathrm{dist}}_{t+1}(\alpha)
    \right)
    \ge
    \Prob(E_\alpha)
    \ge
    1 - \alpha.
\]
This completes the proof.


\subsubsection{Proof of Proposition~\ref{thm:rate-distance-CS}}

Fix $\alpha \in (0, 1)$ and write $L_t \coloneqq L_t(\alpha)$. Since $B \ge 1$, the sequence $L_t$ is nondecreasing and satisfies $L_t \ge 1$ for all $t \ge 1$. For convenience, throughout the proof define
\[
    \widehat V_{t,\mathrm{eff}}^{\mathrm{dist}}
    \coloneqq
    \max\left\{\widehat V_t^{\mathrm{dist}},1\right\},
    \qquad
    \widehat m_t^{\mathrm{dist}}
    \coloneqq
    m\left(\widehat V_t^{\mathrm{dist}}\right)
    =
    \left\lceil
        \log_2 \widehat V_{t,\mathrm{eff}}^{\mathrm{dist}}
    \right\rceil.
\]
For $\eta_t = 1/(\mu t)$ and $t_0 = 3$, we have
\[
    1 - 2 \mu \eta_t
    =
    1 - \frac{2}{t}
    =
    \frac{t - 2}{t},
    \qquad t \ge 3.
\]
Therefore
\[
    A_t
    =
    \prod_{s = 3}^t \frac{s}{s - 2}
    =
    \frac{t(t - 1)}{2},
    \qquad t \ge 3.
\]
In particular,
\[
    \frac{1}{A_t}
    =
    \frac{2}{t(t - 1)}.
\]

We first bound the deterministic part of the recursive envelope. Define
\[
    D_t
    \coloneqq
    \frac{1}{A_t}
    \left[
        R_0 + \sum_{s = 3}^t A_s \eta_s^2 \norm{g_s}^2
    \right],
    \qquad t \ge 3.
\]
Since
\[
    A_s \eta_s^2
    =
    \frac{s(s - 1)}{2} \cdot \frac{1}{\mu^2 s^2}
    =
    \frac{s - 1}{2 \mu^2 s}
    \le
    \frac{1}{2 \mu^2},
\]
The almost-sure bound $\norm{g_t}\le G$ gives
\[
    \sum_{s=3}^t A_s \eta_s^2 \norm{g_s}^2
    \le
    \frac{G^2}{2 \mu^2}(t-2).
\]
Using $R_0 \le G^2/\mu^2$, we obtain
\[
    D_t
    \le
    \frac{2}{t(t-1)}
    \left[
        \frac{G^2}{\mu^2}
        +
        \frac{G^2}{2 \mu^2}(t-2)
    \right]
    =
    \frac{G^2}{\mu^2} \frac{1}{t-1}.
\]
Thus, for every $t \ge 3$,
\[
    D_t
    \le
    B \frac{1}{t-1}
    \le
    2 B \frac{L_{t+1}}{t+1},
\]
where we used $B \ge G^2/\mu^2$, $L_{t+1} \ge 1$, and $1/(t-1) \le 2/(t+1)$. Hence,
\begin{equation}
\label{eq:Dt-rate}
    D_t
    \le
    2 B \frac{L_{t+1}}{t+1},
    \qquad t \ge 3.
\end{equation}

Now set $K \coloneqq 4096$. We prove by induction that
\[
    \widehat U^{\mathrm{dist}}_s(\alpha)
    \le
    K B \frac{L_s}{s},
    \qquad s \ge 4.
\]
For the base case $s = 4$, note that $A_3 = 3$ and $\eta_3 = 1/(3 \mu)$, so $A_3^2 \eta_3^2 = 1/\mu^2$. Therefore,
\[
    \widehat V_3^{\mathrm{dist}}
    =
    \sigma^2 A_3^2 \eta_3^2 R_0
    \le
    \frac{\sigma^2}{\mu^2} \frac{G^2}{\mu^2}
    \le
    B^2.
\]
Since $B \ge 1$, this implies $\widehat V_{3, \mathrm{eff}}^{\mathrm{dist}} = \max\{\widehat V_3^{\mathrm{dist}}, 1\} \le B^2$. Consequently,
\[
    \widehat m_3^{\mathrm{dist}}
    =
    \left\lceil \log_2 \widehat V_{3, \mathrm{eff}}^{\mathrm{dist}}\right\rceil
    \le
    1 + 2\log_2 B.
\]
Using $\log_2 B \le 2 \log B$ for $B \ge 1$, we get $\widehat m_3^{\mathrm{dist}} + 2 \le 3 + 2\log_2 B \le 4(1 + \log B)$. Hence,
\[
    \log\bigl((\widehat m_3^{\mathrm{dist}} + 1)(\widehat m_3^{\mathrm{dist}} + 2)\bigr)
    \le
    2 \log(\widehat m_3^{\mathrm{dist}} + 2)
    \le
    2 \log 4 + 2\log(1 + \log B).
\]
Since $\log(1 + \log B) = \log\log(e B) \le L_4$ and $L_4 \ge 1$, the preceding display gives
\[
    \log\bigl((\widehat m_3^{\mathrm{dist}} + 1)(\widehat m_3^{\mathrm{dist}} + 2)\bigr)
    \le
    5 L_4.
\]
Therefore, by~\eqref{eq:subgaussian-stitched-boundary},
\[
    4\,\mathfrak H_\alpha\left(\widehat V_3^{\mathrm{dist}}\right)
    \le
    4 \sqrt{
        B^2
        \left(
            \log\frac{1}\alpha + 5 L_4
        \right)
    }
    \le
    10 B L_4,
\]
where we used $\log(1/\alpha) \le L_4$ and $L_4 \ge 1$. Now,
\[
    \widehat U^{\mathrm{dist}}_4(\alpha)
    =
    \frac{1}{A_3}
    \left[
        R_0 + A_3 \eta_3^2 \norm{g_3}^2 + 4\,\mathfrak H_\alpha\left(\widehat V_3^{\mathrm{dist}}\right)
    \right].
\]
Using $A_3 = 3$, $\eta_3 = 1/(3\mu)$, $R_0 \le G^2/\mu^2$, and $\norm{g_3} \le G$, we obtain
\[
    \widehat U^{\mathrm{dist}}_4(\alpha)
    \le
    \frac{1}{3}
    \left[
        \frac{G^2}{\mu^2}
        +
        \frac{G^2}{3 \mu^2}
        +
        10 B L_4
    \right].
\]
Because $B \ge G^2/\mu^2$ and $L_4 \ge 1$, this yields
\[
    \widehat U^{\mathrm{dist}}_4(\alpha)
    \le
    4 B L_4
    \le
    K B \frac{L_4}{4},
\]
since $K = 4096\ge 16$. This proves the base case.

Now fix $t \ge 4$ and assume that
\[
    \widehat U^{\mathrm{dist}}_s(\alpha)
    \le
    K B \frac{L_s}{s},
    \qquad s = 4, \dots, t.
\]
We prove the same bound for $s = t+1$. By the recursive definition of the observable intrinsic time,
\[
    \widehat V_t^{\mathrm{dist}}
    =
    \sigma^2 A_3^2 \eta_3^2 R_0
    +
    \sum_{s=4}^t
        \sigma^2 A_s^2 \eta_s^2
        \widehat U^{\mathrm{dist}}_s(\alpha).
\]
As shown above, the initial term satisfies $\sigma^2 A_3^2 \eta_3^2 R_0 \le B^2$. For the summation term, we have
\[
    A_s^2 \eta_s^2
    =
    \left(\frac{s(s - 1)}{2}\right)^2
    \frac{1}{\mu^2 s^2}
    =
    \frac{(s-1)^2}{4 \mu^2}
    \le
    \frac{s^2}{4 \mu^2}.
\]
Using the induction hypothesis,
\[
    \sum_{s = 4}^t
        \sigma^2 A_s^2 \eta_s^2
        \widehat U^{\mathrm{dist}}_s(\alpha)
    \le
    \frac{K \sigma^2 B}{4 \mu^2}
    \sum_{s = 4}^t s L_s 
    \le
    \frac{K \sigma^2 B}{4 \mu^2}
    t^2 L_t                      
    \le
    \frac{K}{4} B^2 t^2 L_t,
\]
where we used the monotonicity of $L_s$ and the fact that $\sigma^2/\mu^2 \le B$. Since $t \ge 4$, $B \ge 1$, $L_t \ge 1$, and $K \ge 1$, we also have
\[
    B^2 \le K B^2 t^2 L_t
    \qquad\text{and}\qquad
    1 \le K B^2 t^2 L_t.
\]
Therefore the initial term and the effective truncation by $1$ are both absorbed into the same scale, and
\begin{equation}
\label{eq:Vhat-rate}
    \widehat V_{t,\mathrm{eff}}^{\mathrm{dist}}
    \le
    K B^2 t^2 L_t.
\end{equation}

We next control the stitched logarithmic factor. From \eqref{eq:Vhat-rate} and $K = 4096 = 2^{12}$,
\[
    \widehat m_t^{\mathrm{dist}}
    =
    \left\lceil \log_2 \widehat V_{t,\mathrm{eff}}^{\mathrm{dist}}\right\rceil
    \le
    13 + 2 \log_2 B + 2 \log_2 t + \log_2 L_t.
\]
Therefore, $\widehat m_t^{\mathrm{dist}} + 2 \le 16 + 2\log_2 B + 2\log_2 t + \log_2 L_t$. Since $\log_2 y \le 2\log y$ for $y\ge 1$,
\[
    \widehat m_t^{\mathrm{dist}} + 2
    \le
    16 + 4 \log B + 4\log t + 2\log L_t.
\]
The right-hand side is bounded by $16 (1 + \log B)(1 + \log t)(1 + \log L_t)$. Thus,
\[
\begin{aligned}
    \log\bigl((\widehat m_t^{\mathrm{dist}} + 1)(\widehat m_t^{\mathrm{dist}} + 2)\bigr)
    &\le
    2 \log(\widehat m_t^{\mathrm{dist}} + 2) \\
    &\le
    2 \log 16
    + 2 \log(1 + \log B)
    + 2 \log(1 + \log t)
    + 2 \log(1 + \log L_t).
\end{aligned}
\]
Each term on the right is controlled by $L_t$. Indeed,
\[
    \log(1 + \log B)=\log\log(e B)\le L_t,
    \qquad
    \log(1 + \log t)=\log\log(e t)\le L_t,
\]
and, since $L_t \ge 1$, we have $\log(1 + \log L_t)\le L_t$. Also $2 \log 16 \le 6 L_t$. Hence
\[
    \log\bigl((\widehat m_t^{\mathrm{dist}} + 1)(\widehat m_t^{\mathrm{dist}} + 2)\bigr)
    \le
    12 L_t.
\]
Since $\log(1/\alpha) \le L_t$, we have
\[
    \log\frac{1}{\alpha}
    +
    \log\bigl((\widehat m_t^{\mathrm{dist}} + 1)(\widehat m_t^{\mathrm{dist}} + 2)\bigr)
    \le
    13 L_t.
\]

Combining this with \eqref{eq:Vhat-rate}, we obtain
\[
\begin{aligned}
    4\,\mathfrak H_\alpha\left(\widehat V_t^{\mathrm{dist}}\right)
    =
    4 \sqrt{
        \widehat V_{t,\mathrm{eff}}^{\mathrm{dist}}
        \left(
            \log\frac{1}{\alpha}
            +
            \log\bigl((\widehat m_t^{\mathrm{dist}} + 1)(\widehat m_t^{\mathrm{dist}} + 2)\bigr)
        \right)}                                      
    \le
    4 \sqrt{
        K B^2 t^2 L_t \cdot 13 L_t}                                      
    &=
    4 \sqrt{13 K}\, B\, t L_t                  
    \\&\le
    15 \sqrt K \, B \, t L_t .
\end{aligned}
\]
Dividing by $A_t = t(t-1)/2$ gives
\[
    \frac{4\,\mathfrak H_\alpha\left(\widehat V_t^{\mathrm{dist}}\right)}{A_t}
    \le
    30 \sqrt K \, B \, \frac{L_t}{t-1}.
\]
For $t \ge 4$, we have $1/(t-1) \le 5/(3(t+1))$, and therefore
\begin{equation}
\label{eq:b-over-A-rate}
    \frac{4\,\mathfrak H_\alpha\left(\widehat V_t^{\mathrm{dist}}\right)}{A_t}
    \le
    50 \sqrt K \,B \,\frac{L_{t+1}}{t+1},
\end{equation}
where we also used $L_t \le L_{t+1}$.

Finally, by definition,
\[
    \widehat U^{\mathrm{dist}}_{t+1}(\alpha)
    =
    D_t + \frac{4\,\mathfrak H_\alpha\left(\widehat V_t^{\mathrm{dist}}\right)}{A_t}.
\]
Using \eqref{eq:Dt-rate} and \eqref{eq:b-over-A-rate}, we get
\[
    \widehat U^{\mathrm{dist}}_{t+1}(\alpha)
    \le
    \left(2 + 50 \sqrt K \right)
    B \frac{L_{t+1}}{t+1}.
\]
Since $K = 4096$, $\sqrt K = 64$, and $2 + 50 \sqrt K = 3202 \le 4096 = K$, we conclude that
\[
    \widehat U^{\mathrm{dist}}_{t+1}(\alpha)
    \le
    K B \frac{L_{t+1}}{t+1}.
\]
This closes the induction. Substituting $K = 4096$ gives the claimed bound.


\subsubsection{Proof of Proposition~\ref{thm:adaptive-suboptimality-CS}}

\begin{proof}
By Lemma~\ref{lem:weighted-suboptimality-decomposition}, for every $t \ge t_0$,
\[
    f(\bar x_t) - f(x^\star)
    \le
    \frac{1}{\mu S_t}
    \sum_{s = t_0}^t
    \norm{g_s}^2
    +
    \frac{\mu t_0 (t_0-1)}{4S_t} Z_{t_0}
    +
    \frac{1}{S_t} \bar E_t.
\]
It remains to control $\bar E_t$ uniformly over time. Recall that
\[
    E_t
    =
    t \inner{x^\star-x_t}{\xi_t}.
\]
Since $x_t$ is $\cF_{t-1}$-measurable, applying Assumption~\ref{assump:stochastic-gradients}(ii) with the predictable direction $t (x^\star - x_t)$ gives
\[
    \E\left[
        e^{\lambda E_t}
        \,\middle|\,
        \cF_{t-1}
    \right]
    \le
    \exp\left(
        \frac{\lambda^2 \sigma^2 t^2 Z_t}{2}
    \right)
\]
for every $\lambda \ge 0$. Thus, the increments $E_t = \bar E_t - \bar E_{t-1}$ satisfy the hypotheses of Lemma~\ref{prop:dyadic-stitching} with
\[
    v_t^{\mathrm{sub}}
    =
    \sigma^2 t^2 Z_t,
    \qquad
    \gamma = -1.
\]
Applying Lemma~\ref{prop:dyadic-stitching} with $\bar E_t$ in place of $\bar X_t$ and $V_t^{\mathrm{sub}}$ in place of $V_t$ yields
\[
    \Prob\left(
        \forall t\ge t_0:\
        \bar E_t
        \le
        2\, \mathfrak H_\alpha\left(V_t^{\mathrm{sub}}\right)
    \right)
    \ge
    1-\alpha.
\]
Substituting this bound into the preceding decomposition and using the definition of $U_t^{\mathrm{sub}}(\alpha)$ in~\eqref{eq:Usub} proves the result.
\end{proof}


\subsubsection{Proof of Theorem~\ref{thm:observable-suboptimality-CS}}

We first note that, for each $t \ge t_0$, the quantity $\widehat U_t^{\mathrm{sub}}(\alpha)$ depends only on $R_0$, the observed trajectory up to time $t$, and the recursive observable bounds
\[
    \widehat U_{t_0+1}^{\mathrm{dist}}(\alpha/2), \dots, \widehat U_t^{\mathrm{dist}}(\alpha/2).
\]
Hence $\widehat U_t^{\mathrm{sub}}(\alpha)$ is fully observable. Now let $U_t^{\mathrm{sub}}(\alpha/2)$ be the adaptive confidence sequence from Proposition~\ref{thm:adaptive-suboptimality-CS}. Set
\[
    E_1
    \coloneqq
    \left\{
        \forall t \ge t_0:\ 
        f(\bar x_t) - f(x^\star) \le U_t^{\mathrm{sub}}(\alpha/2)
    \right\}.
\]
By Proposition~\ref{thm:adaptive-suboptimality-CS}, $\Prob(E_1) \ge 1 - \alpha/2$. Also set
\[
    E_2
    \coloneqq
    \left\{
        \forall t \ge t_0:\ 
        Z_{t+1} \le \widehat U_{t+1}^{\mathrm{dist}}(\alpha/2)
    \right\}.
\]
By Theorem~\ref{thm:observable-distance-cs}, $\Prob(E_2) \ge 1-\alpha/2$. We now show that on $E_2$, $U_t^{\mathrm{sub}}(\alpha/2) \le \widehat U_t^{\mathrm{sub}}(\alpha)$ for all $t \ge t_0$. Indeed, on $E_2$, we have $Z_s \le \widehat U_s^{\mathrm{dist}}(\alpha/2)$ for every $s \ge t_0 + 1$, and also $Z_{t_0} \le R_0$. Therefore,
\[
    V_t^{\mathrm{sub}}
    =
    \sigma^2 t_0^2 Z_{t_0}
    +
    \sum_{s = t_0 + 1}^t \sigma^2 s^2 Z_s
    \le
    \sigma^2 t_0^2 R_0
    +
    \sum_{s = t_0 + 1}^t \sigma^2 s^2 \widehat U_s^{\mathrm{dist}}(\alpha/2)
    =
    \widehat V_t^{\mathrm{sub}}(\alpha/2).
\]
Since $\mathfrak H_{\alpha/2}$ is nondecreasing, it follows that
\[
    2\,\mathfrak H_{\alpha/2}\left(V_t^{\mathrm{sub}}\right)
    \le
    2\,\mathfrak H_{\alpha/2}\left(
        \widehat V_t^{\mathrm{sub}}(\alpha/2)
    \right).
\]
Using also $Z_{t_0} \le R_0$, we obtain
\[
    U_t^{\mathrm{sub}}(\alpha/2)
    \le
    \frac{1}{\mu S_t} \sum_{s = t_0}^t \norm{g_s}^2
    +
    \frac{\mu t_0 (t_0-1)}{4 S_t} R_0
    +
    \frac{2}{S_t}\,
    \mathfrak H_{\alpha/2}\left(
        \widehat V_t^{\mathrm{sub}}(\alpha/2)
    \right)
    =
    \widehat U_t^{\mathrm{sub}}(\alpha).
\]
Therefore, on $E_1 \cap E_2$,
\[
    f(\bar x_t) - f(x^\star)
    \le
    U_t^{\mathrm{sub}}(\alpha/2)
    \le
    \widehat U_t^{\mathrm{sub}}(\alpha)
    \qquad
    \text{for all } t \ge t_0.
\]
Hence
\[
    \Prob\left(
        \forall t \ge t_0:\ 
        f(\bar x_t) - f(x^\star) \le \widehat U_t^{\mathrm{sub}}(\alpha)
    \right)
    \ge
    \Prob(E_1 \cap E_2)
    \ge
    1 - \alpha,
\]
where the last step follows from the union bound. This completes the proof.


\subsubsection{Proof of Proposition~\ref{thm:rate-suboptimality-CS}}

Throughout the proof, $C < \infty$ denotes a universal constant whose value may change from line to line. Write $L_t^{\mathrm{sub}} \coloneqq L_t^{\mathrm{sub}}(\alpha)$. The sequence $L_t^{\mathrm{sub}}$ is nondecreasing and satisfies $L_t^{\mathrm{sub}} \ge 1$. For convenience, throughout the proof define
\[
    \widehat V_{t,\mathrm{eff}}^{\mathrm{sub}}(\alpha/2)
    \coloneqq
    \max\left\{
        \widehat V_t^{\mathrm{sub}}(\alpha/2),1
    \right\},
    \qquad
    \widehat m_t^{\mathrm{sub}}(\alpha/2)
    \coloneqq
    m\left(
        \widehat V_t^{\mathrm{sub}}(\alpha/2)
    \right)
    =
    \left\lceil
        \log_2
        \widehat V_{t,\mathrm{eff}}^{\mathrm{sub}}(\alpha/2)
    \right\rceil.
\]
For every $t \ge 4$,
\[
    S_t
    = 
    \sum_{s=4}^t s
    =
    \frac{t(t+1)}{2} - 6
    \ge
    \frac{t^2}{4}.
\]
By definition of the observable weighted-average suboptimality bound,
\[
    \widehat U_t^{\mathrm{sub}}(\alpha)
    =
    \frac{1}{S_t}
    \left(
        \frac{1}{\mu} \sum_{s=4}^t \norm{g_s}^2
        +
        3 \mu R_0
        +
        2\,\mathfrak H_{\alpha/2}\left(
            \widehat V_t^{\mathrm{sub}}(\alpha/2)
        \right)
    \right),
    \qquad t \ge 4,
\]
where we used $\mu t_0 (t_0 - 1)/4 = 3\mu$. 

We first bound the deterministic terms. Since $\norm{g_s} \le G$ and, by the standing initialization, $R_0 \le G^2/\mu^2$, we have
\[
    \frac{1}{S_t} \frac{1}{\mu} \sum_{s=4}^t \norm{g_s}^2
    \le
    \frac{G^2}{\mu} \frac{t}{S_t}
    \le
    4 \frac{G^2}{\mu} \frac{1}{t},
\]
and
\[
    \frac{3 \mu R_0}{S_t}
    \le
    3 \frac{G^2}{\mu} \frac{1}{S_t}
    \le
    12 \frac{G^2}{\mu} \frac{1}{t^2}
    \le
    12 \frac{G^2}{\mu}\frac{1}{t}.
\]
Since $B_{\mathrm{sub}} \ge G^2/\mu$ and $L_t^{\mathrm{sub}} \ge 1$, it follows that
\begin{equation}
\label{eq:subopt-deterministic-rate}
    \frac{1}{S_t}
    \left(
        \frac{1}{\mu} \sum_{s=4}^t \norm{g_s}^2
        +
        3 \mu R_0
    \right)
    \le
    C B_{\mathrm{sub}} \frac{L_t^{\mathrm{sub}}}{t}.
\end{equation}

It remains to control the observable boundary term. By Remark~\ref{rem:alternative-polynomial-stepsize}, applied with confidence level $\alpha/2$, there exists a universal constant $C < \infty$ such that, for every $s \ge 5$,
\[
    \widehat U_s^{\mathrm{dist}}(\alpha/2)
    \le
    C \, B \, \frac{L_s(\alpha/2)}{s},
\]
where $L_s(\cdot)$ is the logarithmic factor from Proposition~\ref{thm:rate-distance-CS}. Since
\[
    L_s(\alpha/2)
    =
    L_s(\alpha) + \log 2
    \le
    C L_s^{\mathrm{sub}}(\alpha),
\]
we obtain
\[
    \widehat U_s^{\mathrm{dist}}(\alpha/2)
    \le
    C \, B \, \frac{L_s^{\mathrm{sub}}}{s},
    \qquad s \ge 5.
\]

Using the definition of the observable intrinsic time, we get
\[
\begin{aligned}
    \widehat V_t^{\mathrm{sub}}(\alpha/2)
    =
    16 \sigma^2 R_0
    +
    \sum_{s=5}^t \sigma^2 s^2 \widehat U_s^{\mathrm{dist}}(\alpha/2)
    &\le
    16 \sigma^2 \frac{G^2}{\mu^2}
    +
    C \sigma^2 B
    \sum_{s=5}^t s L_s^{\mathrm{sub}}\\
    &\le
    16 \sigma^2 \frac{G^2}{\mu^2}
    +
    C \sigma^2 B \, t^2 L_t^{\mathrm{sub}}.
\end{aligned}
\]
The first term is bounded by $C \mu^2 B^2$, because
\[
    \sigma^2 \frac{G^2}{\mu^2}
    =
    \mu^2
    \left(\frac{\sigma^2}{\mu^2}\right)
    \left(\frac{G^2}{\mu^2}\right)
    \le
    \mu^2 B^2.
\]
The second term is bounded by $C \mu^2 B^2 t^2 L_t^{\mathrm{sub}}$, since
\[
    \sigma^2 B
    =
    \mu^2 \left(\frac{\sigma^2}{\mu^2}\right) B
    \le
    \mu^2 B^2.
\]
Therefore, using $t \ge 4$ and $L_t^{\mathrm{sub}} \ge 1$,
\[
    \widehat V_t^{\mathrm{sub}}(\alpha/2)
    \le
    C \mu^2 B^2 t^2 L_t^{\mathrm{sub}}.
\]
Since $B_{\mathrm{sub}} = 1 + \mu B$, we have $\mu B \le B_{\mathrm{sub}}$ and $1 \le B_{\mathrm{sub}}$. Hence
\[
    \mu^2 B^2 t^2 L_t^{\mathrm{sub}}
    \le
    B_{\mathrm{sub}}^2 t^2 L_t^{\mathrm{sub}},
    \qquad
    1 \le B_{\mathrm{sub}}^2 t^2 L_t^{\mathrm{sub}}.
\]
Thus the effective intrinsic time satisfies
\begin{equation}
\label{eq:subopt-variance-rate}
    \widehat V_{t,\mathrm{eff}}^{\mathrm{sub}}(\alpha/2)
    \le
    C B_{\mathrm{sub}}^2 t^2 L_t^{\mathrm{sub}}.
\end{equation}
Therefore,
\[
    \widehat m_t^{\mathrm{sub}}(\alpha/2)
    \le
    C
    +
    2 \log_2 B_{\mathrm{sub}}
    +
    2 \log_2 t
    +
    \log_2 L_t^{\mathrm{sub}}.
\]
The same elementary logarithmic bound used in the proof of Proposition~\ref{thm:rate-distance-CS} then gives
\[
    \log\left(
        \bigl(\widehat m_t^{\mathrm{sub}}(\alpha/2) + 1\bigr)
        \bigl(\widehat m_t^{\mathrm{sub}}(\alpha/2) + 2\bigr)
    \right)
    \le
    C L_t^{\mathrm{sub}}.
\]
The logarithmic factor on the left is bounded in this way because the preceding display introduces only iterated logarithms of $t$, $B_{\mathrm{sub}}$, and $L_t^{\mathrm{sub}}$, all of which are absorbed by the definition of $L_t^{\mathrm{sub}}$; the same definition also gives $\log(2/\alpha) \le C L_t^{\mathrm{sub}}$. Using the definition of $\mathfrak H_{\alpha/2}$ and~\eqref{eq:subopt-variance-rate}, we obtain
\[
    2\,\mathfrak H_{\alpha/2}\left(
        \widehat V_t^{\mathrm{sub}}(\alpha/2)
    \right)
    \le
    2 \sqrt{
        C B_{\mathrm{sub}}^2 t^2 L_t^{\mathrm{sub}}
        \cdot
        C L_t^{\mathrm{sub}}
    }
    \le
    C B_{\mathrm{sub}}\,t\,L_t^{\mathrm{sub}}.
\]
Dividing by $S_t \ge t^2/4$ gives
\begin{equation}
\label{eq:subopt-boundary-rate}
    \frac{2\,\mathfrak H_{\alpha/2}\left(
            \widehat V_t^{\mathrm{sub}}(\alpha/2)
        \right)}{S_t}
    \le
    C B_{\mathrm{sub}}\frac{L_t^{\mathrm{sub}}}{t}.
\end{equation}

Combining \eqref{eq:subopt-deterministic-rate} and \eqref{eq:subopt-boundary-rate}, we conclude that
\[
    \widehat U_t^{\mathrm{sub}}(\alpha)
    \le
    C B_{\mathrm{sub}}\frac{L_t^{\mathrm{sub}}}{t},
    \qquad t \ge 4,
\]
almost surely. This proves the proposition.


\subsection{Proofs of Section~\ref{sec:refined-confidence-sequences}}

\subsubsection{Proof of Lemma~\ref{lem:eb-exponential}}

The claim is immediate for $\theta = 0$, so fix $\theta \in (0,1)$. For $x \ge -1$, define
\[
    h_\theta(x)
    \coloneqq
    \log(1+\theta x) - \theta x
    +\psi_{\mathrm E}(\theta) x^2.
\]
The logarithm is well defined because $1 + \theta x \ge 1 - \theta > 0$, and \eqref{eq:eb-pointwise} is equivalent to $h_\theta(x) \ge 0$. Differentiating gives
\begin{equation}
\label{eq:eb-h-derivative}
    h_\theta'(x)
    =
    x \left(
        2 \psi_{\mathrm E}(\theta)
        - \frac{\theta^2}{1 + \theta x}
    \right).
\end{equation}
The power-series representation
\[
    \psi_{\mathrm E}(\theta)
    =
    \sum_{m=2}^{\infty} \frac{\theta^m}{m}
\]
implies
\begin{equation}
\label{eq:eb-psi-two-sided}
    \frac{\theta^2}{2}
    \le
    \psi_{\mathrm E}(\theta)
    \le
    \frac{\theta^2}{2(1-\theta)}.
\end{equation}
The upper bound is \eqref{eq:eb-psi-upper}.

For $x \ge 0$, the term in parentheses in \eqref{eq:eb-h-derivative} is at least $2 \psi_{\mathrm E}(\theta) - \theta^2 \ge 0$. Thus $h_\theta$ is nondecreasing on $[0, \infty)$, and $h_\theta(x) \ge h_\theta(0) = 0$ there. For $x \in [-1,0]$, set
\[
    q_\theta(x)
    \coloneqq
    2 \psi_{\mathrm E}(\theta)
    - \frac{\theta^2}{1 + \theta x}.
\]
The function $q_\theta$ is nondecreasing, since
\[
    q_\theta'(x)
    =
    \frac{\theta^3}{(1 + \theta x)^2} > 0.
\]
By \eqref{eq:eb-psi-two-sided}, $q_\theta(-1) \le 0$ and $q_\theta(0) \ge 0$. Hence $q_\theta$ changes sign at most once, from nonpositive to nonnegative. Since $x \le 0$, equation \eqref{eq:eb-h-derivative} shows that $h_\theta$ first increases and then decreases on $[-1, 0]$. Its minimum is therefore attained at an endpoint. Finally,
\[
    h_\theta(-1)
    =
    \log(1 - \theta) + \theta + \psi_{\mathrm E}(\theta)
    = 0,
    \qquad
    h_\theta(0) = 0.
\]
This proves \eqref{eq:eb-pointwise}.

For the conditional statement, let $m \coloneqq \E[X \mid \mathcal G]$. Apply \eqref{eq:eb-pointwise} with $\theta = \lambda b$ and $x = X/b$ to obtain
\[
    \exp\left(
        \lambda X
        -\frac{\psi_{\mathrm E}(\lambda b)}{b^2} X^2
    \right)
    \le
    1 + \lambda X.
\]
Taking conditional expectations and multiplying by the $\mathcal G$-measurable factor $e^{- \lambda m}$ gives
\[
    \E\left[
        \exp\left(
            \lambda(X - m)
            - \frac{\psi_{\mathrm E}(\lambda b)}{b^2} X^2
        \right)
        \middle| \mathcal G
    \right]
    \le
    e^{-\lambda m}(1 + \lambda m)
    \le 1.
\]
Here $\lambda m \ge - \lambda b > -1$, and the last inequality follows from $(1+y) e^{-y} \le 1$ for $y > -1$. By \eqref{eq:eb-psi-upper},
\[
    \frac{\psi_{\mathrm E}(\lambda b)}{b^2}
    \le
    \frac{\lambda^2}{2 (1-\lambda b)}.
\]
Replacing the quadratic coefficient by the larger quantity on the right decreases the exponential integrand and proves \eqref{eq:eb-conditional}.

\subsubsection{Proof of Theorem~\ref{thm:variable-range-eb}}

For $k, \ell \ge 0$, define
\[
    \overline v_k \coloneqq 2^k,
    \qquad
    \overline b_\ell \coloneqq 2^\ell.
\]
We stitch over the dyadic intervals
\[
    I_0^{\mathrm v} \coloneqq [0, 1],
    \qquad
    I_k^{\mathrm v}
    \coloneqq
    (2^{k-1}, 2^k],
    \qquad
    k \ge 1,
\]
and
\[
    I_0^{\mathrm b} \coloneqq [0,1],
    \qquad
    I_\ell^{\mathrm b}
    \coloneqq
    (2^{\ell-1}, 2^\ell],
    \qquad
    \ell \ge 1.
\]
For every $k, \ell \ge 0$, set
\[
    \alpha_{k, \ell}
    \coloneqq
    \frac{\alpha}
    {(k+1)(k+2)(\ell+1)(\ell+2)}
\]
and
\[
    d_{k, \ell}
    \coloneqq
    \log \frac{1}{\alpha_{k, \ell}}
    =
    \log \frac{1}{\alpha}
    +
    \log \bigl( (k+1)(k+2) \bigr)
    +
    \log \bigl((\ell+1)(\ell+2)\bigr).
\]
Since
\[
    \sum_{j = 0}^{\infty}
    \frac{1}{(j+1)(j+2)}
    =
    1,
\]
we have
\[
    \sum_{k, \ell \ge 0}\alpha_{k, \ell}
    =
    \alpha.
\]
For each pair $(k, \ell)$, define
\begin{equation}
\label{eq:eb-lambda-kl}
    q_{k, \ell}
    \coloneqq
    \sqrt{\frac{2 d_{k, \ell}}{\overline v_k}},
    \qquad
    \lambda_{k, \ell}
    \coloneqq
    \frac{q_{k, \ell}}
    {1 + \overline b_\ell q_{k, \ell}}.
\end{equation}
Then
\[
    0
    <
    \lambda_{k, \ell} \overline b_\ell
    <
    1.
\]

Fix $k, \ell \ge 0$. Since $c_s$ is $\cF_{s-1}$-measurable for every $s$, the running maximum $C_s = \max_{t_0 \le r \le s} c_r$ is also $\cF_{s-1}$-measurable. Consequently,
\[
    J_s^{(\ell)}
    \coloneqq
    \mathds{1}\{C_s \le \overline b_\ell\}
\]
is predictable. Define
\[
    \overline H_{t_0 - 1}^{(\ell)}
    \coloneqq
    0,
    \qquad
    \overline H_t^{(\ell)}
    \coloneqq
    \sum_{s = t_0}^t
    J_s^{(\ell)}
    \left(
        H_s - \E[H_s \mid \cF_{s-1}]
    \right),
\]
and
\[
    W_{t_0 - 1}^{(\ell)}
    \coloneqq
    0,
    \qquad
    W_t^{(\ell)}
    \coloneqq
    \sum_{s = t_0}^t
    J_s^{(\ell)} H_s^2.
\]
For $t \ge t_0 - 1$, let
\[
    M_t^{k, \ell}
    \coloneqq
    \exp\left(
        \lambda_{k, \ell} \overline H_t^{(\ell)}
        -
        \frac{\lambda_{k, \ell}^2}
        {2 (1-\lambda_{k, \ell} \overline b_\ell)}
        W_t^{(\ell)}
    \right).
\]
If $J_s^{(\ell)} = 1$, then
\[
    |H_s|
    \le
    c_s
    \le
    C_s
    \le
    \overline b_\ell,
\]
whereas $J_s^{(\ell)} H_s = 0$ if $J_s^{(\ell)} = 0$. Hence
\[
    |J_s^{(\ell)} H_s|
    \le
    \overline b_\ell.
\]
Moreover, predictability gives
\[
    \E[J_s^{(\ell)} H_s \mid \cF_{s-1}]
    =
    J_s^{(\ell)} \E[H_s \mid \cF_{s-1}].
\]
Applying Lemma~\ref{lem:eb-exponential} conditionally with $X = J_s^{(\ell)} H_s$, $b = \overline b_\ell$, and $\lambda = \lambda_{k,\ell}$ yields
\[
    \E[M_s^{k, \ell} \mid \cF_{s-1}]
    \le
    M_{s-1}^{k, \ell}.
\]
Thus $\{M_t^{k, \ell}\}_{t \ge t_0-1}$ is a nonnegative supermartingale with $M_{t_0 - 1}^{k, \ell} = 1$. By Ville's inequality,
\begin{equation}
\label{eq:eb-pairwise-bound}
\begin{aligned}
    \Prob\Bigg(
        \exists\,  t \ge t_0:\
        \overline H_t^{(\ell)}
        \ge
        \frac{d_{k, \ell}}{\lambda_{k, \ell}}
        +
        \frac{\lambda_{k,\ell}}
        {2 (1-\lambda_{k, \ell}\overline b_\ell)}
        W_t^{(\ell)}
    \Bigg)
    \le
    \alpha_{k, \ell}.
\end{aligned}
\end{equation}

We now stitch these fixed-$\lambda$ bounds over the dyadic ranges of $W_t$ and $C_t$. Because $C_t$ is nondecreasing, $C_t \le \overline b_\ell$ implies
\[
    J_s^{(\ell)} = 1
    \qquad
    \text{for every }t_0 \le s \le t.
\]
Therefore,
\[
    \overline H_t^{(\ell)}
    =
    \overline H_t,
    \qquad
    W_t^{(\ell)}
    =
    W_t
\]
at every time $t$ satisfying $C_t \le \overline b_\ell$. Taking a union bound in \eqref{eq:eb-pairwise-bound} over all $k, \ell \ge 0$, with probability at least
\[
    1-
    \sum_{k, \ell \ge 0} \alpha_{k, \ell}
    =
    1 - \alpha,
\]
we have
\begin{equation}
\label{eq:eb-all-pairs-linear}
    \overline H_t
    <
    \frac{d_{k, \ell}}{\lambda_{k,\ell}}
    +
    \frac{\lambda_{k, \ell}}
    {2 (1-\lambda_{k, \ell} \overline b_\ell)}
    W_t
\end{equation}
simultaneously for every $k, \ell \ge 0$ and every $t \ge t_0$ such that $C_t \le \overline b_\ell$.

Work on this event and fix $t \ge t_0$. Let
\[
    W_{t,\mathrm{eff}}
    \coloneqq
    \max\{W_t, 1\},
    \qquad
    C_{t,\mathrm{eff}}
    \coloneqq
    \max\{C_t, 1\},
\]
and let
\[
    k_t \coloneqq k(W_t),
    \qquad
    \ell_t \coloneqq \ell(C_t).
\]
Then
\[
    W_t \in I_{k_t}^{\mathrm v},
    \qquad
    C_t \in I_{\ell_t}^{\mathrm b},
\]
and therefore
\[
    W_t
    \le
    \overline v_{k_t}
    \le
    2 W_{t, \mathrm{eff}},
    \qquad
    C_t
    \le
    \overline b_{\ell_t}
    \le
    2 C_{t, \mathrm{eff}}.
\]
Moreover,
\[
    d_{k_t, \ell_t}
    =
    \Gamma_\alpha(W_t, C_t).
\]
Applying \eqref{eq:eb-all-pairs-linear} with $(k, \ell) = (k_t, \ell_t)$ gives
\begin{align*}
    \overline H_t
    &<
    \frac{d_{k_t, \ell_t}}{\lambda_{k_t, \ell_t}}
    +
    \frac{\lambda_{k_t, \ell_t}}
    {2 (1 - \lambda_{k_t, \ell_t} \overline b_{\ell_t})}
    W_t
    \le
    \frac{d_{k_t, \ell_t}}{\lambda_{k_t, \ell_t}}
    +
    \frac{\lambda_{k_t, \ell_t} \overline v_{k_t}}
    {2 (1 - \lambda_{k_t,\ell_t} \overline b_{\ell_t})}.
\end{align*}
By the choice of $\lambda_{k_t,\ell_t}$ in \eqref{eq:eb-lambda-kl},
\[
    \frac{d_{k_t, \ell_t}}{\lambda_{k_t, \ell_t}}
    +
    \frac{\lambda_{k_t, \ell_t} \overline v_{k_t}}
    {2 (1-\lambda_{k_t, \ell_t} \overline b_{\ell_t})}
    =
    \sqrt{2 d_{k_t, \ell_t} \overline v_{k_t}}
    +
    d_{k_t, \ell_t} \overline b_{\ell_t}.
\]
Consequently,
\begin{align*}
    \overline H_t
    <
    \sqrt{2 d_{k_t,\ell_t} \overline v_{k_t}}
    +
    d_{k_t,\ell_t} \overline b_{\ell_t}
    \le
    2 \sqrt{W_{t,\mathrm{eff}} d_{k_t,\ell_t}}
    +
    2 C_{t,\mathrm{eff}} d_{k_t,\ell_t}
    =
    \mathfrak B_\alpha(W_t, C_t).
\end{align*}
Since $t \ge t_0$ was arbitrary, this proves \eqref{eq:variable-range-eb-conclusion}. Monotonicity follows because all terms in \eqref{eq:eb-effective-scales}--\eqref{eq:eb-boundary} are nondecreasing in the corresponding argument.


\subsubsection{Proof of Theorem~\ref{thm:observable-distance-eb}}

The recursion is well defined: at time $t$, all $R_s^{\mathrm{EB}}$ with $s \le t$ have already been computed, and therefore \eqref{eq:eb-observable-variance}--\eqref{eq:eb-observable-distance} define nonnegative $\cF_t$-measurable quantities. Hence the resulting process is fully observable. Let
\[
    \mathcal E_\alpha^{\mathrm{EB}}
    \coloneqq
    \left\{
        \forall t \ge t_0:\
        Z_{t+1}
        \le
        U_{t+1}^{\mathrm{dist,EB}}(\alpha)
    \right\}.
\]
By Proposition~\ref{thm:adaptive-distance-eb}, $\Prob(\mathcal E_\alpha^{\mathrm{EB}}) \ge 1-\alpha$. We prove by induction that, on this event, for every $t \ge t_0$,
\begin{equation}
\label{eq:eb-observable-induction}
    V_t^{\mathrm{dist,EB}}
    \le
    \widehat V_t^{\mathrm{dist,EB}},
    \qquad
    B_t^{\mathrm{dist,EB}}
    \le
    \widehat B_t^{\mathrm{dist,EB}},
    \qquad
    U_{t+1}^{\mathrm{dist,EB}}(\alpha)
    \le
    \widehat U_{t+1}^{\mathrm{dist,EB}}(\alpha).
\end{equation}

For $t = t_0$, Cauchy--Schwarz and $Z_{t_0} \le R_0 = R_{t_0}^{\mathrm{EB}}$ give
\[
    V_{t_0}^{\mathrm{dist,EB}}
    \le
    4 A_{t_0}^2 \eta_{t_0}^2
    \norm{g_{t_0}}^2 R_{t_0}^{\mathrm{EB}}
    =
    \widehat V_{t_0}^{\mathrm{dist,EB}},
\]
and $B_{t_0}^{\mathrm{dist,EB}} \le \widehat B_{t_0}^{\mathrm{dist,EB}}$. Monotonicity of $\mathfrak B_\alpha$ then implies the final inequality in \eqref{eq:eb-observable-induction}.

Now fix $t \ge t_0+1$ and suppose the claim holds through time $t - 1$. On $\mathcal E_\alpha^{\mathrm{EB}}$, for every $s = t_0 + 1, \dots, t$,
\[
    Z_s
    \le
    U_s^{\mathrm{dist,EB}}(\alpha)
    \le
    \widehat U_s^{\mathrm{dist,EB}}(\alpha)
    =
    R_s^{\mathrm{EB}}.
\]
Together with $Z_{t_0} \le R_{t_0}^{\mathrm{EB}}$, this yields
\begin{align*}
    V_t^{\mathrm{dist,EB}}
    &\le
    4 \sum_{s = t_0}^t
    A_s^2 \eta_s^2 \norm{g_s}^2 R_s^{\mathrm{EB}}
    =
    \widehat V_t^{\mathrm{dist,EB}},
    \\
    B_t^{\mathrm{dist,EB}}
    &\le
    2 G \max_{t_0 \le s \le t}
    A_s \eta_s \sqrt{R_s^{\mathrm{EB}}}
    =
    \widehat B_t^{\mathrm{dist,EB}}.
\end{align*}
Monotonicity of $\mathfrak B_\alpha$ and $Z_{t_0} \le R_0$ then give
\[
    U_{t+1}^{\mathrm{dist,EB}}(\alpha)
    \le
    \widehat U_{t+1}^{\mathrm{dist,EB}}(\alpha).
\]
This closes the induction. Hence, on $\mathcal E_\alpha^{\mathrm{EB}}$,
\[
    Z_{t+1}
    \le
    \widehat U_{t+1}^{\mathrm{dist,EB}}(\alpha)
    \qquad
    \text{for every }t \ge t_0.
\]
The theorem follows because $\Prob(\mathcal E_\alpha^{\mathrm{EB}}) \ge 1 - \alpha$.


\subsubsection{Proof of Proposition~\ref{thm:observable-distance-eb-rate}}

Throughout the proof, $C < \infty$ denotes a universal constant whose value may change from line to line. For convenience, throughout the proof define
\[
    \widehat V_{t,\mathrm{eff}}^{\mathrm{dist,EB}}
    \coloneqq
    \max\left\{
        \widehat V_t^{\mathrm{dist,EB}},1
    \right\},
    \qquad
    \widehat B_{t,\mathrm{eff}}^{\mathrm{dist,EB}}
    \coloneqq
    \max\left\{
        \widehat B_t^{\mathrm{dist,EB}},1
    \right\}.
\]
For the specified stepsize,
\[
    A_t
    =
    \prod_{s = 3}^t \frac{s}{s-2}
    =
    \frac{t(t-1)}{2},
    \qquad
    t \ge 3.
\]
Consequently,
\begin{equation}
\label{eq:eb-polynomial-rescaling}
    A_t
    \ge
    \frac{t^2}{4},
    \quad\
    A_t \eta_t^2
    =
    \frac{t - 1}{2 \mu^2 t}
    \le
    \frac{1}{2\mu^2},
    \quad\
    A_t^2 \eta_t^2
    =
    \frac{(t-1)^2}{4 \mu^2}
    \le
    \frac{t^2}{4 \mu^2},
    \quad\
    A_t \eta_t
    =
    \frac{t-1}{2 \mu}
    \le
    \frac{t}{2 \mu}.
\end{equation}
In particular, \eqref{eq:eb-integrability-condition} is automatic.

Write
\[
    L_t \coloneqq L_t^{\mathrm{EB}}(\alpha),
    \qquad
    q_t \coloneqq \frac{L_t}{t},
    \qquad
    \Phi_t \coloneqq q_t + q_t^2.
\]
As in the proof of Proposition~\ref{thm:rate-distance-CS}, we use an induction argument. We claim that, for a sufficiently large universal constant $K$,
\begin{equation}
\label{eq:eb-rate-induction}
    \widehat U_s^{\mathrm{dist,EB}}(\alpha)
    \le
    K B_{\mathrm{EB}} \Phi_s,
    \qquad
    s \ge 4.
\end{equation}

We first verify the claim at $s = 4$. Since $A_3 \eta_3 = 1/\mu$, the definitions \eqref{eq:eb-observable-variance}--\eqref{eq:eb-observable-range} give
\[
    \widehat V_{3,\mathrm{eff}}^{\mathrm{dist,EB}}
    \le
    4 B_{\mathrm{EB}}^2,
    \qquad
    \widehat B_{3,\mathrm{eff}}^{\mathrm{dist,EB}}
    \le
    2 B_{\mathrm{EB}}.
\]
Consequently,
\[
    k\left(\widehat V_3^{\mathrm{dist,EB}}\right) + 2
    \le
    C \left(1 + \log B_{\mathrm{EB}}\right),
    \qquad
    \ell\left(\widehat B_3^{\mathrm{dist,EB}}\right) + 2
    \le
    C \left(1 + \log B_{\mathrm{EB}}\right),
\]
and hence
\[
    \Gamma_\alpha\left(
        \widehat V_3^{\mathrm{dist,EB}},
        \widehat B_3^{\mathrm{dist,EB}}
    \right)
    \le
    C L_4.
\]
Using these bounds in \eqref{eq:eb-observable-distance}, together with $R_0 \le B_{\mathrm{EB}}$, yields
\[
    \widehat U_4^{\mathrm{dist,EB}}(\alpha)
    \le
    C B_{\mathrm{EB}}
    \left(1 + \sqrt{L_4} + L_4\right)
    \le
    C B_{\mathrm{EB}}L_4.
\]
Since $L_4 \ge 1$ and $\Phi_4 \ge L_4/4$, this proves \eqref{eq:eb-rate-induction} at $s = 4$ for some constant $K$.

Now fix $t \ge 4$ and suppose \eqref{eq:eb-rate-induction} holds through time $t$. The deterministic part of \eqref{eq:eb-observable-distance} satisfies
\begin{align}
\label{eq:eb-rate-deterministic}
    \frac{1}{A_t}
    \left[
        R_0
        + \sum_{s = 3}^t A_s\eta_s^2 \norm{g_s}^2
    \right]
    &\le
    \frac{2}{t(t-1)}
    \left[
        \frac{G^2}{\mu^2}
        +\frac{t-2}{2} \frac{G^2}{\mu^2}
    \right]
    \notag\\
    &=
    \frac{G^2/\mu^2}{t-1}
    \le
    2 B_{\mathrm{EB}} q_{t+1}.
\end{align}

We next control the two arguments of the empirical Bernstein boundary. We claim that
\begin{align}
\label{eq:eb-rate-variance-proxy}
    \widehat V_{t,\mathrm{eff}}^{\mathrm{dist,EB}}
    &\le
    C K B_{\mathrm{EB}}^2 t^2 L_t(1 + q_t),
    \\
\label{eq:eb-rate-range-proxy}
    \widehat B_{t,\mathrm{eff}}^{\mathrm{dist,EB}}
    &\le
    C \sqrt K B_{\mathrm{EB}}
    t \sqrt{q_t(1+q_t)}.
\end{align}
For \eqref{eq:eb-rate-variance-proxy}, the contribution of $s = 3$ is at most $4 B_{\mathrm{EB}}^2$. For $s \ge 4$, the induction hypothesis, \eqref{eq:eb-polynomial-rescaling}, and $G^2/\mu^2 \le B_{\mathrm{EB}}$ give
\[
    4 A_s^2 \eta_s^2 \norm{g_s}^2 \widehat U_s^{\mathrm{dist,EB}}(\alpha)
    \le
    K B_{\mathrm{EB}}^2 s^2 \Phi_s.
\]
Since $L_s$ is nondecreasing,
\[
    \sum_{s = 4}^t s^2 \Phi_s
    =
    \sum_{s = 4}^t \left( s L_s + L_s^2 \right)
    \le
    C \left(t^2 L_t + t L_t^2\right)
    =
    C t^2 L_t(1 + q_t).
\]
The initial contribution and the unit floor are absorbed by the right-hand side of \eqref{eq:eb-rate-variance-proxy}.

For \eqref{eq:eb-rate-range-proxy}, the contribution of $s = 3$ is at most $2 B_{\mathrm{EB}}$. For $s \ge 4$, the induction hypothesis and \eqref{eq:eb-polynomial-rescaling} give
\begin{align*}
    2 G A_s \eta_s \sqrt{\widehat U_s^{\mathrm{dist,EB}}(\alpha)}
    &\le
    C \sqrt K B_{\mathrm{EB}}\,
    s \sqrt{\Phi_s}
    \le
    C \sqrt K B_{\mathrm{EB}}
    \left(\sqrt{s L_s} + L_s\right),
\end{align*}
where we used $(G/\mu) \sqrt{B_{\mathrm{EB}}} \le B_{\mathrm{EB}}$. By the monotonicity of $L_s$ and the inequality
\[
    \sqrt q + q
    \le
    2 \sqrt{q(1+q)},
    \qquad
    q \ge 0,
\]
we obtain
\[
    \max_{4 \le s \le t}
    \left(\sqrt{s L_s} + L_s\right)
    \le
    2 t \sqrt{q_t(1+q_t)}.
\]
The initial contribution and the unit floor are again absorbed, proving \eqref{eq:eb-rate-range-proxy}.

We now bound the logarithmic factor. By \eqref{eq:eb-rate-variance-proxy}--\eqref{eq:eb-rate-range-proxy} and the definitions in \eqref{eq:eb-scale-indices},
\begin{align*}
    k\left(\widehat V_t^{\mathrm{dist,EB}}\right) + 2
    &\le
    C \left[
        1 + \log K + \log B_{\mathrm{EB}} + \log t
        + \log L_t + \log(1 + q_t)
    \right],
    \\
    \ell\left(\widehat B_t^{\mathrm{dist,EB}}\right) + 2
    &\le
    C \left[
        1 + \log K + \log B_{\mathrm{EB}} + \log t
        + \log L_t + \log(1 + q_t)
    \right].
\end{align*}
Since $\log((j+1)(j+2)) \le 2\log(j+2)$ for $j \ge 0$, these bounds yield
\begin{align*}
    \Gamma_\alpha\left(
        \widehat V_t^{\mathrm{dist,EB}},
        \widehat B_t^{\mathrm{dist,EB}}
    \right)
    \le C \Bigg[&1 + \log \frac{1}{\alpha}
        + \log(1 + \log K)
        + \log \log(eB_{\mathrm{EB}})
        \\
        &+\log\log(et)
        + \log(1+\log L_t)
        + \log(1+\log(1+q_t))\Bigg].
\end{align*}
The definition of $L_t$ and the elementary bounds
\[
    \log(1+\log L_t) \le L_t,
    \qquad
    \log(1+\log(1+q_t)) \le L_t
\]
therefore imply
\begin{equation}
\label{eq:eb-rate-log-factor}
    \Gamma_\alpha\left(
        \widehat V_t^{\mathrm{dist,EB}},
        \widehat B_t^{\mathrm{dist,EB}}
    \right)
    \le
    C \rho_K L_t,
    \qquad
    \rho_K \coloneqq 1 + \log(1+\log K).
\end{equation}

Combining \eqref{eq:eb-polynomial-rescaling}, \eqref{eq:eb-rate-variance-proxy}, and \eqref{eq:eb-rate-log-factor}, the square-root contribution satisfies
\begin{align}
\label{eq:eb-rate-square-root}
    \frac{2}{A_t}
    \sqrt{
        \widehat V_{t,\mathrm{eff}}^{\mathrm{dist,EB}}
        \Gamma_\alpha\left(
            \widehat V_t^{\mathrm{dist,EB}},
            \widehat B_t^{\mathrm{dist,EB}}
        \right)
    }
    &\le
    C \sqrt{K \rho_K} B_{\mathrm{EB}}
    q_t \sqrt{1 + q_t}
    \notag\\
    &\le
    C \sqrt{K \rho_K} B_{\mathrm{EB}}
    \left(q_t + q_t^2\right).
\end{align}
Similarly, the linear range contribution satisfies
\begin{equation}
\label{eq:eb-rate-range}
    \mathcal R_t^{\mathrm{EB}}(\alpha)
    \le
    C \sqrt K \rho_K B_{\mathrm{EB}}
    q_t \sqrt{q_t(1+q_t)}.
\end{equation}
For every $q_t \ge 0$, the right-hand side of \eqref{eq:eb-rate-range} is bounded by a constant multiple of $B_{\mathrm{EB}}(q_t + q_t^2)$; whenever $q_t \le 1$, it is bounded by a constant multiple of $B_{\mathrm{EB}} q_t^{3/2}$.

Combining \eqref{eq:eb-rate-deterministic}, \eqref{eq:eb-rate-square-root}, and \eqref{eq:eb-rate-range}, and using
\[
    q_t
    \le
    \frac{t+1}{t} q_{t+1}
    \le
    \frac{4}{3} q_{t+1},
    \qquad
    t \ge 3,
\]
which implies $\Phi_t \le (16/9) \Phi_{t+1}$, gives
\[
    \widehat U_{t+1}^{\mathrm{dist,EB}}(\alpha)
    \le
    \left[
        C + C \sqrt{K \rho_K} + C \sqrt K \rho_K
    \right]
    B_{\mathrm{EB}} \Phi_{t+1}.
\]
Since $\rho_K = o(\sqrt K)$, a sufficiently large universal $K$ makes the bracketed coefficient at most $K$. This closes the induction, and proves \eqref{eq:observable-distance-eb-rate}. Moreover, whenever $L_{t+1}^{\mathrm{EB}}(\alpha) \le t+1$, we have $q_{t+1} \le 1$ and therefore $\Phi_{t+1} \le 2 q_{t+1}$.

Finally, whenever $t \ge 4$ and $q_t \le 1$, the sharper form of \eqref{eq:eb-rate-range} gives \eqref{eq:eb-range-lower-order} after fixing the universal induction constant $K$. For fixed $\alpha$ and fixed problem parameters,
\[
    L_t^{\mathrm{EB}}(\alpha)
    =
    O(1 + \log\log t),
\]
so $q_t \to 0$ and $q_t^{3/2} = o(q_t)$. This also confirms the lower-order interpretation stated after the proposition.


\subsubsection{Proof of Theorem~\ref{thm:observable-suboptimality-eb}}

Fix $\alpha \in (0,1)$ and set $\beta \coloneqq \alpha/2$. For the prescribed deterministic stepsize, condition~\eqref{eq:eb-integrability-condition} is automatic, so Theorem~\ref{thm:observable-distance-eb} applies at confidence level $\beta$. Define
\[
    \mathcal E_{\mathrm{sub}}^{\mathrm{EB}}
    \coloneqq
    \left\{
        \forall t \ge t_0:\
        f(\bar x_t) - f(x^\star)
        \le
        U_t^{\mathrm{sub,EB}}(\beta)
    \right\}
\]
and
\[
    \mathcal E_{\mathrm{dist}}^{\mathrm{EB}}
    \coloneqq
    \left\{
        \forall s \ge t_0 + 1:\
        Z_s
        \le
        \widehat U_s^{\mathrm{dist,EB}}(\beta)
    \right\}.
\]
By Proposition~\ref{thm:adaptive-suboptimality-eb} and Theorem~\ref{thm:observable-distance-eb}, each event has probability at least $1 - \beta$.

On $\mathcal E_{\mathrm{dist}}^{\mathrm{EB}}$, equation \eqref{eq:eb-sub-radius-proxy} and the initialization $Z_{t_0} \le R_0$ imply
\[
    Z_s
    \le
    R_s^{\mathrm{dist,EB}}(\beta)
    \qquad
    \text{for every }s \ge t_0.
\]
Consequently, Cauchy--Schwarz gives, for every $t \ge t_0$,
\begin{align*}
    V_t^{\mathrm{sub,EB}}
    &\le
    \sum_{s = t_0}^t
    s^2 \norm{g_s}^2 Z_s
    \le
    \widehat V_t^{\mathrm{sub,EB}}(\beta),
    \\
    B_t^{\mathrm{sub,EB}}
    &\le
    G \max_{t_0 \le s \le t}
    s \sqrt{R_s^{\mathrm{dist,EB}}(\beta)}
    =
    \widehat B_t^{\mathrm{sub,EB}}(\beta).
\end{align*}
By the monotonicity of $\mathfrak B_\beta$ in both arguments and $Z_{t_0} \le R_0$, it follows that
\[
    U_t^{\mathrm{sub,EB}}(\beta)
    \le
    \widehat U_t^{\mathrm{sub,EB}}(\alpha)
    \qquad
    \text{for every }t \ge t_0.
\]
Therefore, on $\mathcal E_{\mathrm{sub}}^{\mathrm{EB}} \cap \mathcal E_{\mathrm{dist}}^{\mathrm{EB}}$,
\[
    f(\bar x_t) - f(x^\star)
    \le
    \widehat U_t^{\mathrm{sub,EB}}(\alpha)
    \qquad
    \text{for every }t \ge t_0.
\]
The union bound gives
\[
    \Prob\left(
        \mathcal E_{\mathrm{sub}}^{\mathrm{EB}}
        \cap
        \mathcal E_{\mathrm{dist}}^{\mathrm{EB}}
    \right)
    \ge
    1 - 2\beta
    =
    1 - \alpha,
\]
which proves the theorem.


\subsubsection{Proof of Proposition~\ref{thm:observable-suboptimality-eb-rate}}

Throughout the proof, $C < \infty$ denotes a universal constant whose value may change from line to line. For convenience, throughout the proof define
\[
    \widehat V_{t,\mathrm{eff}}^{\mathrm{sub,EB}}(\beta)
    \coloneqq
    \max\left\{
        \widehat V_t^{\mathrm{sub,EB}}(\beta),1
    \right\},
    \qquad
    \widehat B_{t,\mathrm{eff}}^{\mathrm{sub,EB}}(\beta)
    \coloneqq
    \max\left\{
        \widehat B_t^{\mathrm{sub,EB}}(\beta),1
    \right\}.
\]
Set
\[
    \beta \coloneqq \frac{\alpha}{2},
    \qquad
    L_t \coloneqq L_t^{\mathrm{sub,EB}}(\alpha),
    \qquad
    q_t \coloneqq \frac{L_t}{t}.
\]
The sequence $L_t$ is nondecreasing and satisfies $L_t \ge 1$.

Following Remark~\ref{rem:eb-alternative-polynomial-stepsize}, we have:
\begin{equation}
\label{eq:eb-sub-distance-all-time-rate}
    \widehat U_s^{\mathrm{dist,EB}}(\beta)
    \le
    C B_{\mathrm{EB}}
    \left[
        \frac{L_s^{\mathrm{EB}}(\beta)}{s}
        +
        \left(
            \frac{L_s^{\mathrm{EB}}(\beta)}{s}
        \right)^2
    \right],
    \qquad
    s \ge 5.
\end{equation}
By \eqref{eq:eb-sub-rate-log}, $L_s^{\mathrm{EB}}(\beta) \le L_s$, and hence
\begin{equation}
\label{eq:eb-sub-distance-radius-rate}
    R_s^{\mathrm{dist,EB}}(\beta)
    \le
    C B_{\mathrm{EB}}
    \left(q_s + q_s^2\right),
    \qquad
    s \ge 5.
\end{equation}
At the initial index,
\[
    R_4^{\mathrm{dist,EB}}(\beta)
    =
    R_0
    =
    \frac{G^2}{\mu^2}.
\]

We next control the two arguments of the empirical Bernstein boundary. We claim that, for every $t \ge 4$,
\begin{align}
\label{eq:eb-sub-rate-variance-proxy}
    \widehat V_{t,\mathrm{eff}}^{\mathrm{sub,EB}}(\beta)
    &\le
    C \left( B_{\mathrm{sub}}^{\mathrm{EB}} \right)^2
    t^2 L_t(1 + q_t),
    \\
\label{eq:eb-sub-rate-range-proxy}
    \widehat B_{t,\mathrm{eff}}^{\mathrm{sub,EB}}(\beta)
    &\le
    C B_{\mathrm{sub}}^{\mathrm{EB}}
    t \sqrt{q_t(1+q_t)}.
\end{align}
For \eqref{eq:eb-sub-rate-variance-proxy}, the contribution of $s = 4$ is bounded by
\[
    16 G^2 R_0
    =
    \frac{16 G^4}{\mu^2}
    \le
    16 \left(B_{\mathrm{sub}}^{\mathrm{EB}}\right)^2.
\]
For $s \ge 5$, Assumption~\ref{assump:bounded-stochastic-gradients} and \eqref{eq:eb-sub-distance-radius-rate} give
\[
    s^2 \norm{g_s}^2 R_s^{\mathrm{dist,EB}}(\beta)
    \le
    C G^2 B_{\mathrm{EB}}
    \left(s L_s + L_s^2\right).
\]
Moreover,
\[
    G^2 B_{\mathrm{EB}}
    \le
    \mu^2 B_{\mathrm{EB}}^2
    \le
    \left(B_{\mathrm{sub}}^{\mathrm{EB}}\right)^2,
\]
because $G^2 \le \mu^2 B_{\mathrm{EB}}$. Since $L_s$ is nondecreasing,
\[
    \sum_{s = 5}^t \left(s L_s + L_s^2\right)
    \le
    C \left(t^2 L_t + t L_t^2\right)
    =
    C t^2 L_t(1 + q_t).
\]
The initial contribution and the unit floor are therefore absorbed by the right-hand side of \eqref{eq:eb-sub-rate-variance-proxy}.

For \eqref{eq:eb-sub-rate-range-proxy}, the contribution of $s = 4$ is at most
\[
    4 G \sqrt{R_0}
    =
    \frac{4 G^2}{\mu}
    \le
    4 B_{\mathrm{sub}}^{\mathrm{EB}}.
\]
For $s \ge 5$, equation \eqref{eq:eb-sub-distance-radius-rate} yields
\[
    G s \sqrt{R_s^{\mathrm{dist,EB}}(\beta)}
    \le
    C G \sqrt{B_{\mathrm{EB}}}\,
    s \sqrt{q_s(1+q_s)}.
\]
Since
\[
    G \sqrt{B_{\mathrm{EB}}}
    \le
    \mu B_{\mathrm{EB}}
    \le
    B_{\mathrm{sub}}^{\mathrm{EB}},
\]
and
\[
    s^2 q_s(1+q_s)
    =
    s L_s + L_s^2
    \le
    t L_t + L_t^2
    =
    t^2 q_t(1+q_t),
\]
we obtain \eqref{eq:eb-sub-rate-range-proxy}, after absorbing the initial contribution and the unit floor.

We now control the logarithmic factor. The definitions in \eqref{eq:eb-scale-indices}, together with \eqref{eq:eb-sub-rate-variance-proxy}--\eqref{eq:eb-sub-rate-range-proxy}, imply
\begin{align*}
    k\left(\widehat V_t^{\mathrm{sub,EB}}(\beta)\right) + 2
    &\le
    C \left[
        1 + \log B_{\mathrm{sub}}^{\mathrm{EB}}
        + \log t
        + \log L_t
        + \log(1 + q_t)
    \right],
    \\
    \ell\left(\widehat B_t^{\mathrm{sub,EB}}(\beta)\right) + 2
    &\le
    C\left[
        1 + \log B_{\mathrm{sub}}^{\mathrm{EB}}
        + \log t
        + \log L_t
        + \log(1 + q_t)
    \right].
\end{align*}
Using $\log((j+1)(j+2))\le2\log(j+2)$ for $j \ge 0$, we conclude that
\begin{align*}
    \Gamma_\beta\left(
        \widehat V_t^{\mathrm{sub,EB}}(\beta),
        \widehat B_t^{\mathrm{sub,EB}}(\beta)
    \right)
    &\le
    C \Bigg[
        1 + \log\frac{2}{\alpha}
        + \log\log \left(eB_{\mathrm{sub}}^{\mathrm{EB}}\right)
        + \log\log(et)
        \\
        &\hspace{20mm}
        + \log(1 + \log L_t)
        + \log(1 + \log(1 + q_t))
    \Bigg].
\end{align*}
The definition of $L_t$ controls the first four terms. Furthermore, since $L_t\ge1$,
\[
    \log(1 + \log L_t)
    \le
    L_t,
\]
and, because $\log(1+x) \le x$ for every $x \ge 0$,
\[
    \log(1 + \log(1+q_t))
    \le
    \log(1+q_t)
    \le
    q_t
    =
    \frac{L_t}{t}
    \le
    L_t.
\]
Hence
\begin{equation}
\label{eq:eb-sub-rate-log-factor}
    \Gamma_\beta\left(
        \widehat V_t^{\mathrm{sub,EB}}(\beta),
        \widehat B_t^{\mathrm{sub,EB}}(\beta)
    \right)
    \le
    C L_t.
\end{equation}

Since
\[
    S_t
    =
    \sum_{s=4}^t s
    =
    \frac{t(t+1)}{2} - 6
    \ge
    \frac{t^2}{4},
    \qquad
    t \ge 4,
\]
combining \eqref{eq:eb-sub-rate-variance-proxy} and \eqref{eq:eb-sub-rate-log-factor} gives
\begin{equation}
\label{eq:eb-sub-rate-square-root}
    \frac{2}{S_t}
    \sqrt{
        \widehat V_{t,\mathrm{eff}}^{\mathrm{sub,EB}}(\beta)
        \Gamma_\beta\left(
            \widehat V_t^{\mathrm{sub,EB}}(\beta),
            \widehat B_t^{\mathrm{sub,EB}}(\beta)
        \right)
    }
    \le
    C B_{\mathrm{sub}}^{\mathrm{EB}}
    q_t \sqrt{1 + q_t}
    \le
    C B_{\mathrm{sub}}^{\mathrm{EB}}
    \left(q_t + q_t^2\right).
\end{equation}
Similarly, \eqref{eq:eb-sub-rate-range-proxy} and \eqref{eq:eb-sub-rate-log-factor} yield
\begin{equation}
\label{eq:eb-sub-rate-range}
    \mathcal R_t^{\mathrm{sub,EB}}(\alpha)
    \le
    C B_{\mathrm{sub}}^{\mathrm{EB}}
    q_t \sqrt{q_t(1+q_t)}.
\end{equation}
For every $q_t \ge 0$, the right-hand side is bounded by a constant multiple of $B_{\mathrm{sub}}^{\mathrm{EB}}(q_t + q_t^2)$; whenever $q_t \le 1$, it is bounded by a constant multiple of $B_{\mathrm{sub}}^{\mathrm{EB}}q_t^{3/2}$.

It remains to control the deterministic terms in \eqref{eq:eb-sub-observable-cs}. Assumption~\ref{assump:bounded-stochastic-gradients}, $R_0 = G^2/\mu^2$, and $S_t \ge t^2/4$ give
\begin{align*}
    \frac{1}{S_t}
    \left[
        \frac{1}{\mu}\sum_{s = 4}^t \norm{g_s}^2
        + 3\mu R_0
    \right]
    \le
    \frac{C}{t}\frac{G^2}{\mu}
    \le
    C B_{\mathrm{sub}}^{\mathrm{EB}}q_t,
\end{align*}
where we used $L_t \ge 1$. Combining this bound with \eqref{eq:eb-sub-rate-square-root} and \eqref{eq:eb-sub-rate-range} gives the stronger estimate
\[
    \widehat U_t^{\mathrm{sub,EB}}(\alpha)
    \le
    C B_{\mathrm{sub}}^{\mathrm{EB}}
    \left(q_t + q_t^2\right),
    \qquad
    t \ge 4.
\]
This proves \eqref{eq:observable-suboptimality-eb-rate}. Moreover, whenever $L_t \le t$, we have $q_t \le 1$, and therefore $q_t + q_t^2 \le 2 q_t$. Under the same condition, the sharper form of \eqref{eq:eb-sub-rate-range} gives \eqref{eq:eb-sub-range-lower-order}.

Finally, for fixed $\alpha$ and fixed problem parameters,
\[
    L_t^{\mathrm{sub,EB}}(\alpha)
    =
    O(1 + \log\log t),
\]
so $q_t \to 0$ and $q_t^{3/2} = o(q_t)$. This also confirms the lower-order interpretation stated after the proposition.


\subsection{Proofs of Section~\ref{sec:minibatching}}

\subsubsection{Proof of Proposition~\ref{cor:minibatch-variable-range-eb}}

For brevity, write
\[
    H_t \coloneqq \frac{1}{b}\sum_{i=1}^b H_t^{(i)}.
\]
We follow the predictable-truncation and stitching proof of Theorem~\ref{thm:variable-range-eb}. Fix a dyadic range level $\ell \ge 0$, let $\overline b_\ell \coloneqq 2^\ell$, and define the predictable truncation
\[
    J_t^{(\ell)}
    \coloneqq
    \mathds{1}\left\{ C_t \le \overline b_\ell \right\}.
\]
For every $\lambda \in [0, 1/\overline b_\ell)$, Lemma~\ref{lem:eb-exponential} applied conditionally to $J_t^{(\ell)} H_t^{(i)}$ gives
\[
    \E\Bigg[
        \exp \Bigg(
            \lambda J_t^{(\ell)}
            \bigl( H_t^{(i)} - m_t \bigr)
            -
            \frac{\lambda^2 J_t^{(\ell)}}
            {2 \bigl(1 - \lambda\overline b_\ell\bigr)}
            \bigl(H_t^{(i)}\bigr)^2
        \Bigg)
        \mathrel{\Big|}
        \cF_{t-1}
    \Bigg]
    \le1.
\]
Because $J_t^{(\ell)}$ is $\cF_{t-1}$-measurable and $H_t^{(1)},\ldots,H_t^{(b)}$ are conditionally independent given $\cF_{t-1}$, the corresponding exponential factors are conditionally independent. Multiplying their conditional expectations therefore gives
\[
    \E\Bigg[
        \exp\Bigg(
            \lambda b J_t^{(\ell)}
            \bigl( H_t - m_t \bigr)
            -
            \frac{\lambda^2 J_t^{(\ell)}}
            {2 \bigl( 1 - \lambda \overline b_\ell \bigr)}
            \sum_{i=1}^b \bigl(H_t^{(i)}\bigr)^2
        \Bigg)
        \mathrel{\Big|}
        \cF_{t-1}
    \Bigg]
    \le 1.
\]
Thus, for every fixed intrinsic-time and range scale, the exponential process used in the proof of Theorem~\ref{thm:variable-range-eb} remains a nonnegative supermartingale after replacing its centered increment by $b J_t^{(\ell)} (H_t - m_t)$ and its observed quadratic increment by $J_t^{(\ell)}\sum_{i=1}^b (H_t^{(i)})^2$. Repeating the same dyadic stitching over the intrinsic-time and range scales gives, with probability at least $1 - \alpha$, simultaneously for every $t \ge t_0$,
\[
    b \sum_{s=t_0}^t ( H_s - m_s )
    \le
    \mathfrak B_\alpha\left(
        W_t^{\mathrm{MB}},
        C_t
    \right).
\]
Since $m_s = \E[H_s \mid \cF_{s-1}]$, division by $b$ proves \eqref{eq:minibatch-eb-bound}.


\section{Additional Results}
\label{sec:additional-results}

\begin{proposition}[Refined confidence sequence for $\norm{x_{t+1} - x^\star}^2$]
\label{thm:adaptive-distance-eb}
Suppose Assumptions~\ref{assump:strong-convexity}--\ref{assump:stepsize-bound} and \ref{assump:bounded-stochastic-gradients} hold, and suppose~\eqref{eq:eb-integrability-condition} is satisfied. Let $V_t^{\mathrm{dist,EB}}$ and $B_t^{\mathrm{dist,EB}}$ be defined as in~\eqref{eq:eb-latent-variance}--\eqref{eq:eb-latent-range}. Then, for every $\alpha \in (0,1)$, the process $U_{t+1}^{\mathrm{dist,EB}}(\alpha)$ defined in~\eqref{eq:adaptive-distance-eb} satisfies
\[
    \Prob\left(
        \forall t \ge t_0:\
        Z_{t+1}
        \le
        U_{t+1}^{\mathrm{dist,EB}}(\alpha)
    \right)
    \ge
    1-\alpha.
\]
\end{proposition}
\begin{proof}
By Lemma~\ref{lem:distance-decomposition}, for every $t \ge t_0$,
\[
    A_t Z_{t+1}
    \le
    Z_{t_0}
    +
    \sum_{s = t_0}^t
    A_s \eta_s^2 \norm{g_s}^2
    +
    \sum_{s = t_0}^t X_s.
\]
By~\eqref{eq:eb-distance-H} and conditional unbiasedness,
\[
    H_s^{\mathrm{dist}}
    -
    \E\left[
        H_s^{\mathrm{dist}}
        \,\middle|\,
        \cF_{s-1}
    \right]
    =
    X_s.
\]
Moreover, Assumption~\ref{assump:bounded-stochastic-gradients} gives
\[
    \left|H_s^{\mathrm{dist}}\right|
    \le
    2 G A_s \eta_s \sqrt{Z_s}.
\]
By construction, $\{H_s^{\mathrm{dist}}\}_{s \ge t_0}$ is adapted. Since $A_s$, $\eta_s$, and $Z_s$ are $\cF_{s-1}$-measurable, the process
\[
    c_s^{\mathrm{dist}}
    \coloneqq
    2 G A_s \eta_s \sqrt{Z_s}
\]
is predictable and almost surely finite. Moreover, the quadratic process and running predictable range associated with $\{H_s^{\mathrm{dist}}\}_{s \ge t_0}$ in Theorem~\ref{thm:variable-range-eb} are precisely $V_t^{\mathrm{dist,EB}}$ and $B_t^{\mathrm{dist,EB}}$, respectively. Finally, since $Z_s \le G^2/\mu^2$ almost surely, condition~\eqref{eq:eb-integrability-condition} ensures that $H_s^{\mathrm{dist}}$ is integrable for every deterministic $s \ge t_0$. Applying Theorem~\ref{thm:variable-range-eb} therefore gives, with probability at least $1-\alpha$, simultaneously for every $t \ge t_0$,
\[
    \sum_{s = t_0}^t X_s
    \le
    \mathfrak B_\alpha\left(
        V_t^{\mathrm{dist,EB}},
        B_t^{\mathrm{dist,EB}}
    \right).
\]
Substituting this inequality into the distance decomposition and dividing by $A_t$ yields
\[
    Z_{t+1}
    \le
    U_{t+1}^{\mathrm{dist,EB}}(\alpha)
\]
simultaneously for every $t \ge t_0$, which proves the result.
\end{proof}


\begin{proposition}[Refined confidence sequence for $f(\bar x_t) - f(x^\star)$]
\label{thm:adaptive-suboptimality-eb}
Suppose Assumptions~\ref{assump:strong-convexity}--\ref{assump:stepsize-bound} and \ref{assump:bounded-stochastic-gradients} hold. Let $t_0 \ge 4$ and $\eta_t = 2/(\mu(t+1))$ for every $t \ge t_0$. Let $V_t^{\mathrm{sub,EB}}$ and $B_t^{\mathrm{sub,EB}}$ be defined as in~\eqref{eq:eb-sub-latent-variance}--\eqref{eq:eb-sub-latent-range}. Then, for every $\alpha \in (0,1)$, the process $U_t^{\mathrm{sub,EB}}(\alpha)$ defined in~\eqref{eq:adaptive-suboptimality-eb} satisfies
\[
    \Prob\left(
        \forall t \ge t_0:\
        f(\bar x_t) - f(x^\star)
        \le
        U_t^{\mathrm{sub,EB}}(\alpha)
    \right)
    \ge
    1-\alpha.
\]
\end{proposition}
\begin{proof}
Lemma~\ref{lem:weighted-suboptimality-decomposition} gives, for every $t \ge t_0$,
\begin{equation}
\label{eq:eb-sub-weighted-recursion}
    f(\bar x_t) - f(x^\star)
    \le
    \frac{1}{\mu S_t} \sum_{s = t_0}^t \norm{g_s}^2
    +
    \frac{\mu t_0 (t_0 - 1)}{4 S_t} Z_{t_0}
    +
    \frac{1}{S_t} \sum_{s = t_0}^t E_s.
\end{equation}
By~\eqref{eq:eb-sub-H} and conditional unbiasedness,
\[
    H_s^{\mathrm{sub}}
    -
    \E\left[
        H_s^{\mathrm{sub}}
        \,\middle|\,
        \cF_{s-1}
    \right]
    =
    E_s.
\]
Moreover, the predictable range associated with $H_s^{\mathrm{sub}}$ has running maximum $B_t^{\mathrm{sub,EB}}$, while the corresponding quadratic process is
\[
    \sum_{s = t_0}^t
    \left(H_s^{\mathrm{sub}}\right)^2
    =
    V_t^{\mathrm{sub,EB}}.
\]
Theorem~\ref{thm:variable-range-eb} therefore implies that, with probability at least $1-\alpha$, simultaneously for every $t \ge t_0$,
\[
    \sum_{s = t_0}^t E_s
    \le
    \mathfrak B_\alpha\left(
        V_t^{\mathrm{sub,EB}},
        B_t^{\mathrm{sub,EB}}
    \right).
\]
Substituting this inequality into~\eqref{eq:eb-sub-weighted-recursion} proves the proposition.
\end{proof}


\section{Supporting Lemmas}
\label{sec:additional-proofs}

\begin{lemma}[One-step contractive recursion]
\label{lem:contraction-distance}
Under Assumptions~\ref{assump:strong-convexity} and~\ref{assump:stepsize-bound}, the iterates satisfy
\begin{equation}
\label{eq:contraction-distance}
    Z_{t+1}
    \le
    (1 - 2 \mu \eta_t) Z_t
    +
    \eta_t^2 \norm{g_t}^2
    +
    2 \eta_t \inner{x^\star - x_t}{\xi_t},
    \qquad \forall t \ge t_0.
\end{equation}
\end{lemma}
\begin{proof}
The proof follows the standard one-step analysis of projected stochastic approximation; see, e.g., \citet[Section~2.1]{nemirovski2009robust}. We retain the directional gradient noise term explicitly for the anytime-valid analysis. Since $x^\star \in \cX$, the nonexpansiveness of the projection gives
\[
    Z_{t+1}
    =
    \norm{x_{t+1} - x^\star}^2
    =
    \norm{\proj_{\cX}(x_t - \eta_t g_t) - x^\star}^2
    \le
    \norm{x_t - \eta_t g_t - x^\star}^2.
\]
Expanding the square,
\[
    Z_{t+1}
    \le
    Z_t
     - 2 \eta_t \inner{x_t - x^\star}{g_t}
    + \eta_t^2 \norm{g_t}^2.
\]
Now, writing $g_t = \nabla f(x_t) + \xi_t$,
\[
    Z_{t+1}
    \le
    Z_t
    - 2 \eta_t \inner{x_t - x^\star}{\nabla f(x_t)}
    + 2 \eta_t \inner{x^\star - x_t}{\xi_t}
    + \eta_t^2 \norm{g_t}^2.
\]
It remains to lower-bound the deterministic gradient term. By the gradient strong monotonicity property,
\[
    \inner{\nabla f(x_t) - \nabla f(x^\star)}{x_t - x^\star}
    \ge
    \mu \norm{x_t - x^\star}^2
    = \mu Z_t.
\]
Hence
\[
    \inner{\nabla f(x_t)}{x_t - x^\star}
    =
    \inner{\nabla f(x_t) - \nabla f(x^\star)}{x_t - x^\star}
    +
    \inner{\nabla f(x^\star)}{x_t - x^\star}.
\]
By the first-order optimality condition at $x^\star$, we have $\inner{\nabla f(x^\star)}{x_t - x^\star}\ge 0$. Therefore,
\[
    \inner{\nabla f(x_t)}{x_t - x^\star}\ge \mu Z_t.
\]
Substituting back proves \eqref{eq:contraction-distance}.
\end{proof}

\begin{lemma}[Last-iterate distance to the optimizer decomposition]
\label{lem:distance-decomposition}
Suppose Assumptions~\ref{assump:strong-convexity} and~\ref{assump:stepsize-bound} hold, and let $\{A_t\}_{t \ge t_0-1}$ be defined as in~\eqref{eq:def-at-At}. For $t \ge t_0$, define
\[
    X_t
    \coloneqq
    2 A_t \eta_t
    \inner{x^\star - x_t}{\xi_t},
    \qquad
    \bar X_t
    \coloneqq
    \sum_{s = t_0}^t X_s.
\]
Then, for every $t \ge t_0$,
\begin{equation}
\label{eq:distance-decomposition-lemma}
    Z_{t+1}
    \le
    \frac{1}{A_t}
    \left[
        Z_{t_0}
        +
        \sum_{s = t_0}^t
        A_s \eta_s^2 \norm{g_s}^2
        +
        \bar X_t
    \right].
\end{equation}
\end{lemma}
\begin{proof}
A closely related unrolling of the contractive SGD recursion for the classical stepsize $\eta_t = 1/(\mu t)$ appears in \citet[Appendix~B.7, Lemma~6]{rakhlin2012making}; the decomposition below records the same telescoping mechanism for general predictable stepsizes. By Lemma~\ref{lem:contraction-distance}, for every $s \ge t_0$,
\[
    Z_{s+1}
    \le
    (1 - 2 \mu \eta_s) Z_s
    +
    \eta_s^2 \norm{g_s}^2
    +
    2 \eta_s \inner{x^\star - x_s}{\xi_s}.
\]
Multiplying by $A_s$ and using $A_s (1 - 2 \mu \eta_s) = A_{s-1}$, we obtain
\[
    A_s Z_{s+1}
    \le
    A_{s-1} Z_s
    +
    A_s \eta_s^2 \norm{g_s}^2
    +
    X_s.
\]
Summing over $s = t_0, \ldots, t$ gives
\[
    A_t Z_{t+1}
    -
    A_{t_0 - 1} Z_{t_0}
    \le
    \sum_{s = t_0}^t
    A_s \eta_s^2 \norm{g_s}^2
    +
    \bar X_t.
\]
Since $A_{t_0 - 1} = 1$, we obtain
\[
    A_t Z_{t+1}
    \le
    Z_{t_0}
    +
    \sum_{s = t_0}^t
    A_s \eta_s^2 \norm{g_s}^2
    +
    \bar X_t.
\]
Finally, Assumption~\ref{assump:stepsize-bound} implies $A_t>0$, so dividing by $A_t$ proves~\eqref{eq:distance-decomposition-lemma}.
\end{proof}

\begin{lemma}[Weighted-average suboptimality decomposition]
\label{lem:weighted-suboptimality-decomposition}
Suppose Assumption~\ref{assump:strong-convexity} holds. Let $t_0 \ge 4$ and $\eta_t = 2/(\mu(t+1))$ for every $t \ge t_0$. Define
\[
    E_t
    \coloneqq
    t\inner{x^\star - x_t}{\xi_t},
    \qquad
    \bar E_t
    \coloneqq
    \sum_{s = t_0}^t E_s.
\]
Then, for every $t\ge t_0$,
\[
    f(\bar x_t) - f(x^\star)
    \le
    \frac{1}{\mu S_t}
    \sum_{s = t_0}^t
    \norm{g_s}^2
    +
    \frac{\mu t_0(t_0 - 1)}{4S_t}Z_{t_0}
    +
    \frac{1}{S_t}\bar E_t.
\]
\end{lemma}
\begin{proof}
The proof builds on the weighted telescoping argument of~\citet[Section~3.2]{lacoste2012simpler}, adapted here to retain the directional gradient noise term pathwise rather than taking expectations. Let
\[
    E_{t_0-1} \coloneqq 0,
    \qquad
    E_s \coloneqq s\,\inner{x^\star - x_s}{\xi_s}
    \qquad
    \bar E_t \coloneqq \sum_{s=t_0}^t E_s,
    \qquad t \ge t_0.
\]
Since $x^\star \in \cX$, the nonexpansiveness of the projection gives
\[
    Z_{t+1}
    =
    \norm{x_{t+1} - x^\star}^2
    =
    \norm{\proj_{\cX}(x_t - \eta_t g_t) - x^\star}^2
    \le
    \norm{x_t - \eta_t g_t - x^\star}^2.
\]
Expanding the square,
\[
    Z_{t+1}
    \le
    Z_t
    -2 \eta_t \inner{x_t - x^\star}{g_t}
    + \eta_t^2 \norm{g_t}^2.
\]
Writing $g_t = \nabla f(x_t) + \xi_t$ and using strong convexity,
\[
    \inner{\nabla f(x_t)}{x_t - x^\star}
    \ge
    f(x_t) - f(x^\star)+\frac{\mu}{2} Z_t,
\]
we obtain
\[
    2 \eta_t \bigl(f(x_t) - f(x^\star)\bigr)
    \le
    \eta_t^2 \norm{g_t}^2
    + (1 - \mu\eta_t) Z_t
    -Z_{t+1}
    +2 \eta_t \inner{x^\star - x_t}{\xi_t}.
\]
Now substitute $\eta_t = 2/(\mu(t+1))$. Since
\[
    \frac{\eta_t}{2} = \frac{1}{\mu(t+1)},
    \qquad
    \frac{1 - \mu \eta_t}{2 \eta_t} = \frac{\mu(t-1)}{4},
    \qquad
    \frac{1}{2 \eta_t}=\frac{\mu(t+1)}{4},
\]
it follows that
\[
    f(x_t) - f(x^\star)
    \le
    \frac{\norm{g_t}^2}{\mu(t+1)}
    +
    \frac{\mu}{4}\bigl((t-1)Z_t - (t+1)Z_{t+1}\bigr)
    +
    \inner{x^\star - x_t}{\xi_t}.
\]
Multiplying by $t$ and using $t/(t+1)\le 1$, we obtain
\[
    t \bigl(f(x_t) - f(x^\star)\bigr)
    \le
    \frac{1}{\mu} \norm{g_t}^2
    +
    \frac{\mu}{4} \bigl(t(t-1) Z_t - t(t+1) Z_{t+1}\bigr)
    +
    t\, \inner{x^\star - x_t}{\xi_t}.
\]
Summing from $s = t_0$ to $t$, and using
\[
    \sum_{s=t_0}^t \bigl(s(s-1) Z_s - s(s+1) Z_{s+1}\bigr)
    =
    t_0 (t_0 - 1) Z_{t_0} - t(t+1) Z_{t+1},
\]
we obtain
\[
    \sum_{s = t_0}^t s \bigl(f(x_s) - f(x^\star)\bigr)
    \le
    \frac{1}{\mu} \sum_{s = t_0}^t \norm{g_s}^2
    +
    \frac{\mu}{4} \Bigl(t_0(t_0-1) Z_{t_0} - t(t+1) Z_{t+1}\Bigr)
    +
    \bar E_t.
\]
Dropping the negative terminal term yields
\[
    \sum_{s=t_0}^t s \bigl(f(x_s) - f(x^\star)\bigr)
    \le
    \frac{1}{\mu} \sum_{s=t_0}^t \norm{g_s}^2
    +
    \frac{\mu t_0 (t_0-1)}{4} Z_{t_0}
    +
    \bar E_t.
\]
Dividing by $S_t$ and using the convexity of $f$ we obtain
\begin{equation}
\label{eq:subopt-S_t}
    f(\bar x_t) - f(x^\star)
    \le
    \frac{1}{\mu S_t}\sum_{s = t_0}^t \norm{g_s}^2
    +
    \frac{\mu t_0 (t_0-1)}{4S_t} Z_{t_0}
    +
    \frac{1}{S_t} \bar E_t.
\end{equation}
\end{proof}

\end{document}